\documentclass[11pt]{article}
\usepackage{authblk}

\usepackage{natbib}
\usepackage{fullpage,amsthm,amsmath,amssymb}
\usepackage{thmtools}
\usepackage{thm-restate}
\usepackage{algorithm,algorithmic}
\usepackage{mathtools}
\usepackage{bbm,bm}
\usepackage{enumerate}
\usepackage{cancel}
\usepackage{xspace}
\usepackage[textsize=tiny]{todonotes}

\usepackage{titletoc}

\usepackage[colorlinks,citecolor=blue,urlcolor=blue,linkcolor=black]{hyperref}

\usepackage[capitalize]{cleveref}
\crefname{ALC@unique}{line}{lines}
\Crefname{ALC@unique}{Line}{Lines}

\usepackage{booktabs}
\usepackage{tabularx}

\theoremstyle{definition}

\newtheorem{definition}{Definition}

\newtheorem*{assumption*}{Assumption}

\theoremstyle{remark}
\newtheorem{remark}{Remark}
\newtheorem*{remark*}{Remark}

\theoremstyle{plain}
\newtheorem{theorem}{Theorem}[section]
\newtheorem{lemma}{Lemma}[section]
\newtheorem{corollary}{Corollary}[section]
\newtheorem{proposition}{Proposition}[section]

\newcommand{\ip}[1]{\langle #1 \rangle}
\DeclareMathOperator{\argmax}{argmax}
\DeclareMathOperator{\Tr}{Tr}

\newcommand{\set}[1]{\left\{#1\right\}}
\newcommand{\EE}[1]{\mathbb{E}\left[#1\right]}

\newcommand{\linucb}{\text{LinUCB}\xspace}

\newcommand{\oope}{\text{OOPE}\xspace}

\usepackage{theorems}          
\usepackage{mathmacros}        
\usepackage{projectmacros}     

\usepackage{array}
\usepackage{makecell}

\usepackage{nomencl}
\makenomenclature

\title{Gaussian-Width Bounds for Offline-to-Online Linear Bandits}

\author[]{Kushagra Chandak\protect\footnotemark[1]}
\author[]{Volodymyr Tkachuk\protect\footnotemark[1]}
\author[]{Xiaoqi Tan}
\author[]{Csaba Szepesv\'{a}ri}

\affil[]{Department of Computing Science, University of Alberta, Canada}

\affil[]{\small \texttt{\{kchandak, vtkachuk, xiaoqi.tan, szepesva\}@ualberta.ca}}

\date{\vspace{-5ex}} 

\begin{document}

\renewcommand{\thefootnote}{\fnsymbol{footnote}}
\maketitle
\footnotetext[1]{equal contribution}
\renewcommand{\thefootnote}{\arabic{footnote}}
\setcounter{footnote}{0}

\begin{abstract}
We study how offline data can help online learning in stochastic linear bandits. We focus on partially covered offline data, where only $k$ of the $d$ feature directions are poorly covered. Existing analyses of algorithms based on phased elimination and optimism give regret bounds of order $\widetilde O(\sqrt{dkT})$ in this regime. We show that this dependence arises from loose control of the reward prediction error, and obtain tighter bounds using concentration of Lipschitz functions of Gaussian vectors. Based on this analysis, we develop a phased elimination algorithm with regret $\widetilde O(k\sqrt T)$. For the unit Euclidean ball, we prove a matching regret lower bound under partial coverage, as well as matching upper and lower bounds under uniform offline coverage. We also study pure exploration with offline data, providing algorithms and matching upper and lower bounds on simple regret for fixed-design algorithms and also for adaptive algorithms under partial and uniform coverage of offline data.
\end{abstract}

\section{Introduction}

We consider online learning in the stochastic linear bandit framework, where a learner interacts sequentially with an environment. In every round $t$, the learner chooses an action $A_t \in \cA \subset \R^d$, and receives a reward $R_t = A_t^\top \theta_* + \eta_t$, where $\theta_* \in \R^d$ is an unknown parameter and $\eta_t$ is zero-mean noise. The goal of the learner is to choose actions to maximize the total reward it collects over all the interactions. Equivalently, it aims to minimize its \emph{regret}, that is, how much it suffers for not choosing an optimal action.

In this paper, we study a variant of regret minimization in linear bandits, where the learner also has access to an offline dataset before the start of online interactions. 
Such a dataset can provide valuable hints to the learner so we expect it to be useful, but how? 
In this \emph{offline-to-online} learning problem, we focus on how we can utilize the offline dataset and how much improvement it can provide over pure online learning.

The benefit of offline data depends critically on its coverage of the feature space. Sufficient coverage can largely resolve uncertainty before online interaction begins, bringing the problem closer to offline learning. At the opposite extreme, uniformly poor coverage provides little useful information, bringing the problem closer to pure online learning. These extremes offer less insight into the distinctive challenges of offline-to-online learning. Partial coverage is therefore particularly interesting as the offline data may accurately identify some components of the unknown parameter while leaving others uncertain, requiring online exploration to complement the available information. 

For partially covered offline data, our key observation is that the existing regret guarantees for algorithms like LinUCB \citep[Chapter 19]{abbasi2011improved, lattimore2020bandit} and phased elimination \citep[Chapter 22]{lattimore2020bandit}, which are optimal for several action sets in the online case, are no longer optimal when augmented with offline data \citep{vijayan2026regret,chandak2026offlinetoonline}.
By comparison, in the unstructured multi-armed bandit (MAB) setting, a UCB variant that computes the means and confidence radii by simply combining the offline and online data is optimal \citep{shivaswamy2012multi, cheung2024leveraging}. This contrast motivates a closer examination of how linear structure affects the use of partially covered offline data.

To illustrate the gap, consider the case of partially covered offline data and the unit $d$-dimensional Euclidean ball as the action set.
Suppose $k$ out of $d$ directions have very little offline data and the remaining $d-k$ directions have a large amount (see \Cref{def:spectral-coverage-regimes}). Previous offline-to-online algorithms based on phased elimination \citep{vijayan2026regret} and LinUCB \citep{chandak2026offlinetoonline} achieve a regret bound of $\tilde O(\sqrt{dkT})$, where $T$ is the number of online rounds. More details are in \Cref{sec:oope-linucb}. However, since there are only $k$ uncovered directions in the offline data, the problem is effectively $k$-dimensional, and the lower bound \citep{vijayan2026regret} suggests that the regret should be $\Omega(k\sqrt{T})$. Therefore, there is a multiplicative $\sqrt{d/k}$ gap in the upper and lower bounds, which could be significant in high-dimensional problems. \Cref{tab:results} provides a summary.

Our main insight is that looseness in LinUCB and phased elimination with offline data is due to conservative bounds on the reward prediction error. These algorithms typically estimate $\theta_*$ using a least-squares estimator $\hat \theta$, and control the prediction error $a^\top(\theta_* - \hat \theta)$ where $a \in \cA$. With offline data that has good coverage in some direction $i$, we expect $\hat \theta$ to be a good estimate of $\theta_*$ in that direction, that is, $(\theta_* - \hat \theta)_i \approx 0$. Therefore, components of actions in direction $i$ contribute $\approx 0$ to the prediction error. That is, the prediction error is insensitive to action components in directions that are well covered by the offline data.
However, the previous analyses \citep{vijayan2026regret, chandak2026offlinetoonline} do not fully exploit this low sensitivity. \Cref{sec:oope-linucb} explains how this leads to loose regret guarantees.

Based on the works of \citet{pmlr-v336-maiti26a} and \citet{wagenmaker2021experimental}, we provide finer control of the reward prediction error uniformly over the action set using concentration of Lipschitz functions, and present a phased elimination algorithm with an upper bound on its regret.
To complement the upper bound, we provide a matching lower bound for various cases of offline coverage, when the action set is the unit $d$-dimensional Euclidean ball. 
We also analyze the pure exploration setting, where the learner's performance is evaluated only on the final action it outputs, also known as simple regret. We provide an algorithm with a bound on its simple regret, and also provide lower bounds for adaptive and non-adaptive algorithms. To summarize, our key contributions are the following. 
\begin{enumerate}
    \item We identify limitations of existing regret guarantees for standard bandit algorithms like phased elimination and LinUCB when augmented with offline data due to conservative bounds on the reward prediction error.
    \item Using a tighter bound on the prediction error, we design a phased elimination algorithm incorporating offline data and show an upper bound on its regret. We show a regret lower bound when the action set is the unit Euclidean ball in $\R^d$, and show that it matches the upper bound of the phased elimination algorithm up to log factors for various cases of offline dataset coverage.
    \item For offline-to-online pure exploration, we provide a non-adaptive algorithm with matching upper and lower bounds on simple regret. We also prove a lower bound for adaptive algorithms when the action set is the unit $d$-dimensional Euclidean ball, and show that it matches the upper bound of the non-adaptive algorithm for many offline data coverage regimes as in the cumulative regret case.
\end{enumerate}

The paper is organized as follows. In \Cref{sec:setting}, we describe our problem setting. Next, in \Cref{sec:oope-linucb}, we discuss the issues faced by phased elimination and LinUCB algorithms with partial coverage of offline data, and suggest a fix for the phased elimination algorithm. As a warmup to the full regret setting, we analyze the simpler pure exploration problem in \Cref{sec:pureexp}. Finally, we present our phased elimination algorithm and its analysis with regret lower bounds in \Cref{sec:gw-oope}. The related work and proofs for all the results are in the appendix. A summary of the notation used is provided in \Cref{tab:notation}.

\begin{table}[t]
    \centering
    \small
    \renewcommand{\arraystretch}{1.25}
    \begin{tabularx}{\columnwidth}{>{\raggedright\arraybackslash}X>{\centering\arraybackslash}p{0.30\columnwidth}}
        \toprule
        \textbf{Algorithm} & \textbf{$k$-uncovered} \\
        \midrule
        LinUCB [\Cref{app:linucb}]
        & \textcolor{red!80!black}{$\sqrt{k d T}$} \\
        \addlinespace
        \oope \citep{vijayan2026regret}
        & \textcolor{red!80!black}{$\sqrt{k d T}$} \\
        \addlinespace
        \gwoope \textbf{[Our Work]}
        & \textcolor{green!50!black}{$k\sqrt{T}$} \\
        \midrule
        Lower Bound \textbf{[Our Work]}
        & \textcolor{green!50!black}{$k\sqrt{T}$} \\
        \bottomrule
    \end{tabularx}
    \caption{
        Leading expected regret terms in the $k$-uncovered regime of \Cref{def:spectral-coverage-regimes}, where $\cA=B_2^d$.
        Logarithmic factors and online budget requirements are suppressed.
        The OOPE guarantee additionally requires an offline count condition (see \Cref{eq:oope-offline-count-condition}).
    }
    \label{tab:results}
\end{table}

\section{Problem Setting: Linear Bandits with Offline Data} \label{sec:setting}
Let $\cA \subset \R^d$ be a nonempty compact%
\footnote{
    Compactness is assumed to ensure that the relevant quantities remain bounded and that the maximizers used throughout the paper exist. 
}
set of actions available to the learner. 
Before the start of the online interactions, the learner is also given an offline dataset $\Doff =  \paren*{\left(\aoff_s, \roff_s\right)}_{s=1}^{\Toff}$, where $\aoff_s \in \cA$, $\roff_s \in \R$ is the corresponding reward, and $\Toff \in \bN$ is the number of offline samples. 
Similar to previous works \citep{shivaswamy2012multi,banerjee2022artificial,cheung2024leveraging,sentenac2025balancing,vijayan2026regret,chandak2026offlinetoonline}, we consider the case when $\Doff$ is collected by a fixed design, i.e., the number of offline samples $\Toff(a)$ for every $a \in \cA$ is fixed. 
We show in Appendix \ref{app:adapt-offline-data} why an adaptive design can be challenging to deal with. 
The rewards in the offline data are generated as $\roff_s = \ip{\aoff_s, \truetheta} + \noiseoff_s$ where $\truetheta\in\bR^d$ is the unknown parameter, and $\paren*{\noiseoff_s}_{s=1}^{\Toff}$ are iid standard Gaussian noise variables.%
\footnote{
    Gaussianity permits the use of the confidence bounds in \cref{sec:oope-linucb} and the exact independent equal-information offline data splitting in \cref{sec:gw-oope}.
    Ordinary offline data splitting can also be done without Gaussian noise but requires a different method.
}

Given $\Doff$, the learner interacts sequentially with the environment over $T$ rounds. In each round $t \in [T]$, the learner chooses an action $A_t \in \cA$ and observes reward $R_t = \ip{A_t, \truetheta} + \eta_t$ where $\eta_t$ is standard Gaussian conditional on the $\sigma$-algebra generated from the history before $R_t$ is observed, i.e., $\sigma(\Doff, A_1, R_1, \dots, A_{t-1}, R_{t-1}, A_t)$. 
Denote by $\opta \in\argmax_{a\in\cA}\ip{a,\truetheta}$ an optimal action for $\theta_* \in \R^d$.

The learner's goal is to maximize its online total reward, or equivalently, minimize its \emph{pseudo-regret} (henceforth referred to as regret):
\begin{equation*}
    \regret_T(\cA) \defeq \sum_{t=1}^T \ip{\opta-A_t, \truetheta} \,.
\end{equation*}
For regret minimization, we assume that every action has gap at most $2$, which is satisfied under the usual normalization $\cA\subseteq B_2^d$ and $\norm{\truetheta}_2\le1$.
This normalization is not imposed for pure exploration.

For a set $\cX \subseteq \bR^d$ and budget $n\in\bN_0$, let $\alloc_n(\cX)$ denote the set of finitely supported \emph{exact designs} $N \colon \cX\to\bN_0$ satisfying $\sum_{a\in\cX}N(a)=n$.
Write $\alloc(\cX) \defeq \cup_{n\ge0}\alloc_n(\cX)$ and $|N| \defeq \sum_aN(a)$.
Define the offline and the online design matrices as
\begin{align*}
    \Voff 
    &\defeq 
    \textstyle
    \sum_{s=1}^{\Toff}\aoff_s(\aoff_s)^\top, \\[2pt] 
    V(N) 
    &\defeq 
    \textstyle
    \sum_{a\in\supp(N)}N(a)aa^\top.
\end{align*}
The combined design matrix is then given by 
\begin{align*}
    \textstyle
    \, \tilde V(N) \defeq \Voff+V(N).
\end{align*}
When $n=0$, $\alloc_0(\cX)$ consists only of the zero design, and $V(N)=0$.
To keep the presentation simple, we assume throughout that $\Voff$ is invertible.

Although our upper bounds hold for general compact action sets, throughout the paper we use the following coverage regimes to give interpretable bounds when $\cA=B_2^d$.
The rotational symmetry of the Euclidean ball allows us to work in an eigenbasis of $\Voff$ without changing the action set.

\begin{definition}[Spectral coverage regimes for $B_2^d$]
    \label{def:spectral-coverage-regimes}
    Let $\cA=B_2^d$, and let $0<\nu_1\le\cdots\le\nu_d$ be the eigenvalues of $\Voff$.
    We say that the offline data are \emph{equally-covered} if, for a fixed constant $c\in(0,1]$,
    \begin{align*}
        \nu_i\ge c\Toff/d,\qquad i\in[d].
    \end{align*}
    For $1\le k\le d$ and $T\ge1$, we say that the offline data are \emph{$k$-uncovered} if %
    \begin{align*}
        \textstyle
        \nu_i\le1 \text{ for } i\le k,
        \quad \text{and} \quad
        \nu_i\ge(d-k)T/k \text{ for } i>k.
    \end{align*}
    When $k=d$, the second condition is vacuous, so $d$-uncovered means only the first condition is satisfied.
    The constants may be replaced by any other constants without affecting our resulting rates.
\end{definition}

For $\cA=B_2^d$, every offline action has norm at most one, so $\tr(\Voff)\le\Toff$.
Consequently, $\Toff/d$ is the largest amount of coverage that can be available uniformly across all $d$ directions.
The equally-covered condition therefore says that every direction receives at least a constant fraction of the largest coverage that could be provided uniformly across all directions.

In the $k$-uncovered regime, the first condition says that $k$ directions have only constant offline coverage.
The second condition ensures that the remaining $d-k$ directions are sufficiently well covered, since our analysis will show that their contribution to the reward prediction error is no larger than that of the $k$ uncovered directions with $T$ online samples.
Thus, the remaining online problem behaves as though it were $k$-dimensional.
When $k=d$, every direction has only constant offline coverage, effectively recovering the pure online setting.

\begin{remark}[Relation to \citet{vijayan2026regret}] \label{rem:vijayan2026regret}
    The equally-covered condition agrees with their well-explored condition on $B_2^d$, up to constants.
    Our $k$-uncovered condition is analogous to their moderately-explored regime when $k=o(d)$ and their under-explored regime when $k=\Omega(d)$ (see their Table~1).
    The difference is that we additionally specify how much coverage is needed in the remaining $d-k$ directions so that the leading rate depends only on $k$.
\end{remark}

\section{LinUCB and Phased Elimination} \label{sec:oope-linucb} \label{SEC:OOPE-LINUCB}

\Citet{vijayan2026regret} introduced the Offline-Online Phased Elimination (\oope) algorithm which combines the phased elimination (PE) algorithm \citep[Chapter 22]{lattimore2020bandit} with offline data.
They pointed out a multiplicative $\sqrt{d/k}$ gap between their regret upper bound and their hypercube-based lower bound under a coverage condition similar to our $k$-uncovered condition (see \Cref{rem:vijayan2026regret}).
In this section we discuss why the regret bound for \oope is loose when the offline data are $k$-uncovered.
This observation will motivate an improvement to the \oope algorithm, which we present in \Cref{sec:gw-oope}.
We also discuss how combining LinUCB with offline data \citep{chandak2026offlinetoonline} suffers an analogous $\sqrt{d/k}$ gap. In \Cref{app:vijayan-lower-bound}, we show that the lower bound in \citet{vijayan2026regret} can be loose, which motivates our regret lower bound in \Cref{sec:gw-oope}.

\subsection{LinUCB Prediction Error}
\label{sec:linucb-prediction-error}
The standard LinUCB analysis \citep[Chapter 19]{abbasi2011improved,lattimore2020bandit} proceeds by bounding the instantaneous regret $\ip{\opta - A_t, \theta_*}$ in every round $t$. On the event that $\theta_*$ lies in the confidence set, using the optimistic estimate $\theta_t^{\textrm{op}}$ and Cauchy-Schwarz, we have
\begin{equation*}
    \ip{\opta - A_t, \theta_*} \le \ip{A_t, \theta_t^{\textrm{op}} - \theta_*} \le \norm{A_t}_{\tilde V_{t-1}^{-1}} \norm{\theta_t^{\textrm{op}} - \theta_*}_{\tilde V_{t-1}} \,,
\end{equation*}
where $\tilde V_t = \Voff + \sum_{s=1}^t A_s A_s^\top$ so that $\tilde V_0 = \Voff$.
The term $\norm{\theta_t^{\textrm{op}} - \theta_*}_{\tilde V_{t-1}}$ can be upper bounded by $2\sqrt{\beta_t}$ where $\sqrt{\beta_t}$ is the radius of the elliptical confidence set $\cC_t \coloneqq \{ \theta \in \R^d \,:\, \norm{\theta - \hat \theta_{t-1}}_{\tilde V_{t-1}} \le \sqrt{\beta_t} \}$ and $\hat \theta_{t-1}$ is the least squares estimate used in the construction of $\cC_t$.
The radius $\sqrt{\beta_t}$ is a high-probability bound on $\norm{\hat \theta_{t-1} - \theta_*}_{\tilde V_{t-1}}$ uniformly over all rounds $t$, and $(\beta_t)_t$ is non-decreasing. 
In our offline-to-online case, note that $\hat \theta_0$ is computed purely based on offline data and using the definition of reward generation, one can show that $\norm{\hat \theta_0 - \theta_*}_{\Voff} \overset{d}{=} \norm{g}_2$ and with high probability $\norm{g}_2 \asymp \sqrt{d}$ \citep[Theorem 3.1.1]{vershynin2026high} where $g \sim \cN(0, I_d)$.
Thus, at any fixed sufficiently high confidence level, a valid initial radius must satisfy \(\sqrt{\beta_1}\gtrsim\sqrt d\). Since $(\beta_t)_t$ is non-decreasing, the standard regret analysis of LinUCB retains a $\sqrt{d}$ factor. 
Hence, we cannot expect this analysis to give a regret bound that scales with only $k$ for the $k$-uncovered case. 
More precisely, in \Cref{app:linucb} we show that LinUCB suffers a $\sqrt{d/k}$ gap in its regret bound.

\subsection{Phased Elimination Prediction Error}
The phase-based \oope algorithm \citep{vijayan2026regret} assumes that the action set $\cA$ is finite.
For infinite action sets, \citet[Remark 4.6]{vijayan2026regret} propose discretizing $\cA$ and running \oope on an $\eps$-net $\cA_\eps$.
For $\cA=B_2^d$ and $\norm{\truetheta}_2\le1$, restricting the actions to $\cA_\eps$ introduces a discretization error of at most $T\eps$ to the regret.
Thus, their guarantee 
(see \Cref{thm:oope}) 
becomes
\begin{align*}
    \regret_T(\cA)
    \le\regret_T(\cA_\eps)+T\eps = \bigOt \paren{\sqrt{d_{\mathrm{eff}}T\log|\cA_\eps|} +d^2+T\eps} = \bigOt \paren{\sqrt{d_{\mathrm{eff}}dT}+d^2},
\end{align*}
where $d_{\mathrm{eff}} \le d$ is an effective dimension, defined in \Cref{app:oope-regret-bound}.
The last line uses a sufficiently fine net with $T\eps\le1$ and $\log|\cA_\eps| = \bigOt(d)$. 
Under their additional offline count condition (see \Cref{eq:oope-offline-count-condition}),
specializing this to $k$-uncovered offline data we get that $d_{\mathrm{eff}} \asymp k$ when $T \ge k$, which gives a regret bound of $\bigOt(\sqrt{kdT}+d^2)$.

Before presenting our algorithm in \Cref{sec:gw-oope}, which improves the $\sqrt d$ factor to $\sqrt k$, we focus on understanding the source of the $\sqrt{\log|\cA_\eps|} = \bigOt(\sqrt d)$ factor in the \oope regret bound. 
We start by looking at the phased elimination algorithm \citep[Chapter 22]{lattimore2020bandit}.

At a high level, PE eliminates suboptimal actions in phases and outputs any of the remaining actions once the online budget is exhausted.
More precisely, in phase $\ell$ the algorithm (1) selects a design $\hat N_\ell$ based on the set of uneliminated actions $\cA_\ell$, (2) chooses actions based on $\hat N_\ell$ and observes the corresponding rewards, (3) computes a least-squares estimate $\hat\theta_\ell$, and 
(4) eliminates actions whose estimated gap exceeds a constant multiple of $\eps_\ell$.
To ensure that optimal actions are not eliminated, the design is selected so that the maximum reward prediction error
\begin{equation}\label{eq:estimation-error}
    \textstyle
    \max_{a\in\cA_\ell}|\ip{a,\hat\theta_\ell-\truetheta}|
\end{equation}
is at most $\eps_\ell$ with high probability. Since $\truetheta$ is unknown, the design is chosen using an upper bound on \Cref{eq:estimation-error}.
If this bound is loose, a phase can require more online samples than are needed for the desired accuracy.

To bound \Cref{eq:estimation-error}, \oope uses a union bound over the actions, giving a factor of $\sqrt{\log|\cA|}$.
When $\log|\cA|$ is of order $d$, this introduces a $\sqrt d$ term.
Our proposal is to replace the union bound by a tighter bound using concentration of Lipschitz functions of Gaussian random vectors, following \citet{pmlr-v336-maiti26a,wagenmaker2021experimental}.
To motivate this choice, we step away from PE and study the maximum prediction error under a fixed design, as in a single phase of PE.

\subsection{Improvements via Tighter Concentration}

We use the terminology \emph{fixed} design for any deterministic design $N\in\alloc_T(\cA)$ chosen without inspecting any rewards.
Fix $T\in\bN_0$ and a design $N\in\alloc_T(\cA)$.
Let $a_1,\ldots,a_T$ be the actions chosen according to $N$, and let $R_1,\ldots,R_T$ be their rewards.
The least-squares estimate using the offline and online data is
\begin{align}
    \textstyle
    \hat\theta_N=\tilde V(N)^{-1}\left(
        \sum_{s=1}^{\Toff}\roff_s\aoff_s+
        \sum_{t=1}^TR_ta_t\right).
    \label{eq:hat_theta_t}
\end{align}
Substituting the reward definitions into the estimator shows that the error is the Gaussian vector 
\begin{equation}
    \hat\theta_N-\truetheta
    \stackrel{d}{=}\tilde V(N)^{-1/2}g,
    \qquad g\sim\cN(0,I_d).
    \label{eq:fixed-design-gaussian-error}
\end{equation}
Taking the inner product with any action $a \in \cA$ gives
\begin{align}
    \textstyle \ip{a, \hat \theta_N - \theta_*} 
    \overset{d}{=} 
    \ip{a, \tilde V(N)^{-1/2}g}.
    \label{eq:prediction-error}
\end{align}
If the action set is finite,
then one way to bound the maximum of the above prediction error is to use Gaussian concentration \citep[Proposition~2.1.2]{vershynin2026high} for each action $a \in \cA$ and a union bound. 
\begin{restatable}[Union Bound Concentration]{lemma}{lemUnionBoundConcentration}
    \label{lem:union-bound-concentration}
    For any finite $\cA$, and any fixed design $N\in\alloc_T(\cA)$, with probability at least $1-\delta$,
    \begin{align*}
        \max_{a\in\cA}|\ip{a,\hat\theta_N-\truetheta}|
        \le 
        \max_{a\in\cA}\norm{a}_{\tilde V(N)^{-1}}
        \sqrt{2\log(2|\cA|/\delta)}
        \, .
    \end{align*}
\end{restatable}

An algorithm can then choose its design $N$ to minimize the right-hand side.
In the pure online setting ($\Toff=0$), although this bound can be loose for some combinatorial action sets \citep{katz2020empirical}, it does get the correct dimension dependence for other commonly studied finite action sets such as the corners of the hypercube $B_\infty^d$ or a $1/2$-net of $B_2^d$ with $\log|\cA| \asymp d$, for sufficiently large $T$ (see \Cref{app:tightness-of-union-bounds}).
To see what happens when the action set is a $1/2$-net of $B_2^d$ containing the standard basis vectors with $\log|\cA| \asymp d$, and the offline data satisfy the $k$-uncovered eigenvalue conditions, consider the following example.

Suppose for $1\le k<d$,
\begin{equation*}
    \Voff=\sum_{i=1}^ke_ie_i^\top
    +\frac{\Toff-k}{d-k}\sum_{j=k+1}^de_je_j^\top.
\end{equation*}
In the extreme case where $\Toff\to\infty$, with $d,k,T$ fixed, and $N(e_i)=T/k$ for $i\le k$, 
then%
\footnote{
    Assume that $k$ divides $T$ for simplicity. 
}
\begin{align*}
    \tilde V(N)^{-\frac12}g
    =\frac{\sum_{i=1}^kg_ie_i}{\sqrt{1+T/k}}
    +\sqrt{\frac{d-k}{\Toff-k}}\sum_{j=k+1}^dg_je_j \longrightarrow
    \frac{1}{\sqrt{1+T/k}} (g_1,\ldots,g_k,0,\ldots,0)^\top 
\end{align*}
The contribution of the well-covered $e_{k+1}, \dots, e_d$ directions vanishes and the error is only $k$-dimensional.
Combining this observation with \Cref{eq:prediction-error} suggests that the maximum prediction error should only scale with $k$.
However, the union bound still pays for all $|\cA|$ actions and therefore suffers a $\sqrt{\log|\cA|} \asymp \sqrt d$ term.

\begin{restatable}{lemma}{lemUBKuncovered}
    \label{lem:ub-k-uncovered}
    Let $\cA\subset B_2^d$ be a $1/2$-net with $\log|\cA| \asymp d$, containing an orthonormal eigenbasis of $\Voff$.
    Under the $k$-uncovered eigenvalue conditions\footnote{We use the eigenvalue condition from the $k$-uncovered definition, but change the action set to a $1/2$-net so that it is finite.} in \Cref{def:spectral-coverage-regimes} and $T\ge k$, 
    \begin{align*}
        &\min_{N\in\alloc_T(\cA)}
        \max_{a\in\cA}\norm{a}_{\tilde V(N)^{-1}}
        \sqrt{2\log(2|\cA|/\delta)} \asymp 
        \sqrt{\frac{kd}{T}}
        +\sqrt{\frac{k\log(2/\delta)}{T}}.
    \end{align*}
\end{restatable}

To get a tighter bound on reward prediction error, we apply concentration of Lipschitz functions of a Gaussian vector (\Cref{lem:lipschitz-concentration}).
To state the result, define the Gaussian width%
\footnote{
    By Talagrand's majorizing measure theorem \citep[Theorem~8.5.5]{vershynin2026high}, $\gw(\cX)$ is equivalent, up to constants, to the  $\gamma_2(\cX,\norm{\cdot}_2)$ functional, where the metric is Euclidean.
    We keep our bounds in terms of Gaussian width, since it can often be bounded directly for the sets of interest without evaluating $\gamma_2$.
}
of a compact set $\cX\subset\bR^d$ as
\begin{align*}
 \gw(\cX)\defeq\bE\max_{x\in\cX}\ip{x,g},
    \quad \text{where } g \sim \cN(0,I_d).
\end{align*}
For a positive definite matrix $M$ and $\beta\ge0$, define
\begin{align*}
    \rad_\beta(\cX,M)\defeq
    \gw(M^{-1/2}\cX)+\beta\max_{x\in\cX}\norm{x}_{M^{-1}}.
\end{align*}
By noting that $g\mapsto\max_{a\in\cA}\ip{a,\tilde V(N)^{-1/2}g}$
is Lipschitz with constant $\max_{a\in\cA}\norm{a}_{\tilde V(N)^{-1}}$, we can
apply Gaussian concentration to both signs to get the following bound.
\begin{restatable}[Gaussian Width Concentration]{lemma}{lemGWConcentration}
    \label{lem:gw-concentration}
    For any compact $\cA\subset\bR^d$ and any fixed design $N\in\alloc_T(\cA)$, with probability at least $1-\delta$,
    \begin{align}
        \max_{a\in\cA}|\ip{a,\hat\theta_N-\truetheta}|
        \le\rad_{\sqrt{2\log(2/\delta)}}(\cA,\tilde V(N)).
        \label{eq:shared-fixed-design-two-sided-confidence}
    \end{align}
\end{restatable}
Choosing the design based on the above bound gives
\begin{restatable}{corollary}{corGWKuncovered}
    \label{cor:gw-k-uncovered}
    Let $\cA=B_2^d$ or let $\cA$ be the net in \Cref{lem:ub-k-uncovered}.
    Under the $k$-uncovered eigenvalue conditions and $T\ge k$, there is a fixed design $N\in\alloc_T(\cA)$ such that, with probability at least $1-\delta$,
    \begin{align*}
        \max_{a\in\cA}|\ip{a,\hat\theta_N-\truetheta}|
        \lea\frac{k}{\sqrt T}+\sqrt{\frac{k\log(2/\delta)}{T}}.
    \end{align*}
\end{restatable}
Thus the $\sqrt{kd}$ factor from the union bound can be replaced by $k$.
What this demonstrates is that uneven offline coverage can make the union bound loose, even when the union bound was tight in the pure online setting for the same action set%
\footnote{
    By union bound here we refer to the bound obtained by \Cref{lem:union-bound-concentration} when the design is chosen to minimize the right-hand side.
}.
Although the action geometry has not changed, the new issue is that the offline data force the reward prediction error to be much smaller in some directions than in others.

In the following sections we will show improved rates for pure exploration and regret minimization by using the Gaussian width concentration bound (\Cref{lem:gw-concentration}) to choose the design.

\section{Pure Exploration with Offline Data} \label{sec:pureexp} \label{SEC:PUREEXP}

As a warm-up, we present results for an algorithm that uses a fixed design in pure exploration, which should be thought of as resembling a single phase of the phased elimination algorithm.
In pure exploration with a fixed budget $T \in \bN_0$, the learner's goal is to output an action $\hat A \in \cA$ that maximizes the expected reward $\ip{\hat A, \truetheta}$, 
or equivalently, to minimize the \emph{simple regret}
\begin{align*}
    \sregret_T(\cA)
    \defeq
    \ip{\opta - \hat A, \truetheta}.
\end{align*}
We present the Gaussian Width Offline-to-Online Fixed Design (\gwoofd) algorithm (\Cref{alg:gw-oofd-fixed-budget}), 
which is an offline-to-online variant of the fixed-design algorithm from \citep[Section 2.1]{pmlr-v336-maiti26a}. 
The algorithm selects an optimal design based on the bound in \Cref{lem:gw-concentration}, takes actions based on the selected design, and returns the action that maximizes the reward estimate.

Note that the optimization in \Cref{alg:gw-oofd-fixed-budget} requires an exact design oracle and it need not be computationally efficient, but it gives a statistical benchmark.
In Appendix \ref{app:gw-fractional-designs}, we show that an approximate design can be rounded to an exact design with only a constant-factor loss in the confidence radius when the online budget is sufficiently large. The simple regret of \gwoofd is bounded as follows.
\begin{restatable}[\gwoofd upper bound]{theorem}{thmGWOOFDUpperBound}
    \label{thm:gw-oofd-upper-bound}
    For any compact $\cA \subset \bR^d$,
    with probability at least $1-\delta$, the simple regret of \gwoofd satisfies
    \begin{align*}
        \sregret_T(\cA)
        \le
        2
        \min_{N \in \alloc_T(\cA)}
        \rad_{\sqrt{2\log(2/\delta)}}
        \paren{\cA, \tilde V(N)}.
    \end{align*}
\end{restatable}

\begin{algorithm}[t]
\caption{\gwoofd}
\label{alg:gw-oofd-fixed-budget}
\begin{algorithmic}[1]
\STATE \textbf{Input:} $\cA$, $\Doff$, $T \in \bN_0$, $\delta\in(0,1)$
\STATE
$\textstyle
    \hat N
    \in
    \argmin\nolimits_{N\in\alloc_T(\cA)}
    \rad_{\sqrt{2\log(2/\delta)}} \paren{\cA,\tilde V(N)}
$
\STATE Take actions using $\hat N$ and compute $\hat\theta_{\hat N}$ (\Cref{eq:hat_theta_t})
\STATE \textbf{return} $\hat A\in\argmax_{a\in\cA}\ip{a, \hat \theta_{\hat N}}$
\end{algorithmic}
\end{algorithm}

\subsection{Lower Bounds}
Next, we state lower bounds.
Define a fixed design algorithm to be any algorithm that:
selects a fixed design $\hat N \in \alloc_T(\cA)$;
takes actions based on $\hat N$ and observes online rewards; 
outputs an action $\hat A \in \cA$, which may depend on all the offline and online actions and rewards.
Notice that \gwoofd is a fixed design algorithm.
In the pure online setting (i.e., $\Toff=0$), \citet{pmlr-v336-maiti26a} showed that the simple regret of \gwoofd is the best possible for any fixed-design algorithm.
The following theorem shows this remains true in the offline-to-online setting.%

\begin{restatable}[Fixed-design lower bound]{theorem}{thmFixedDesignLowerBound}
    \label{thm:fixed-design-lower-bound}
    Let $\cA \subset \bR^d$ be a compact set,
    $d\ge2$ and\,%
    \footnote{
        For simplicity of presentation, we define the confidence radius using individual actions rather than pairwise action differences.
        This requires $d\ge2$ for the lower bound to hold for general action sets.
        Defining the confidence radius $\zeta$ using differences $a-b$, with $a,b\in\cA$, would give upper and lower bounds for $d\ge1$.
    } 
    $\delta\in(0,1/10]$.
    Suppose a fixed-design algorithm has simple regret at most $\eps > 0$ with probability at least $1-\delta$, for every $\theta\in\bR^d$.
    Then
    \begin{align*}
        \eps
        \gea \min_{N\in\alloc_T(\cA)} \rad_{\sqrt{\log(1/\delta)}} \paren{\cA, \tilde V(N)}.
    \end{align*}
\end{restatable}
We now state the simple regret lower bound when $\cA = B_2^d$ for any (possibly adaptive) algorithm, which can choose its online actions based on previous online actions and rewards, and the offline dataset.

\begin{restatable}[Adaptive-design lower bound]{theorem}{HpSrLb}\label{thm:hp-sr-lb}
    Let $d \ge 2$ and $\cA = B_2^d$ and $\Delta \in (0,1)$. For $\delta \in (0, \Delta^2/(128\sqrt{e}))$ and $\epsilon > 0$, fix an algorithm with simple regret at most $\epsilon$, with probability at least $1-\delta$, for every $\theta \in \R^d$.
    Suppose $\mu = \cN(0,\Sigma)$ and $\Sigma = \diag(\sigma_1^2,\dots,\sigma_d^2)$ in the eigenbasis of $\Voff$ with $\sigma_1^2 \ge \dots \ge\sigma_d^2$ and $\sigma_1^2 > 0$. Let
    \begin{equation*}
      c_\Delta = \frac{(1-\Delta)^2}{2\Delta - \Delta^2},
      \quad
      d_\Delta = \frac{\sum_{i=2}^d \log\big(1 + c_\Delta\tfrac{\sigma_i^2}{\sigma_1^2}\big)}
                         {\log\big(1 + c_\Delta\big)} .
    \end{equation*}
    Suppose $\Sigma$ and $\Delta$ satisfy
    \begin{equation}\label{eq:exp-fano}
      \sum_{i=1}^d \log\big(1 + \sigma_i^2(\nu_i + n_i)\big) + \log 2 \le \frac{d_\Delta}{2}\log\big(1 + c_\Delta\big),
    \end{equation}
    where $\nu_i$ is the eigenvalue of $\Voff$ corresponding to the $i$-th eigenvector, and $n_i = \bE_\mu[\sum_{t=1}^T A_{t,i}^2]$, with coordinates in the specified eigenbasis. Then
    \begin{equation*}
      \textstyle
      \epsilon \gtrsim \Delta^2 \sqrt{\Tr \Sigma} + \sup_{a \in \cA} \norm{a}_{\tilde V_0^{-1}}\sqrt{\log(1/\delta)}\,,
    \end{equation*}
    where $\tilde V_0 = \Voff + \bE_0[\sum_{t=1}^T A_t A_t^\top]$ and $\bE_0$ is the expectation under $\theta = 0$.
\end{restatable}

Next, we show that for the spectral coverage regimes in \Cref{def:spectral-coverage-regimes}, the simple regret of \gwoofd matches the above lower bound.
\begin{restatable}[Simple regret for spectral coverage]{corollary}{corPureExpSpectralCoverage}\label{cor:hp-sr-lb}
    Let $\cA = B_2^d$. 
    For any $\delta \in (0,1)$, under the coverage conditions in \Cref{def:spectral-coverage-regimes}, with probability at least $1-\delta$, the simple regret of \gwoofd is bounded, up to constant factors, by
    \begin{align*}
        &
        \text{equally-covered $(\forall T\ge0)$}:
        \!\!\!\!&& 
        \frac{d}{\sqrt{T+\Toff}} + \sqrt{\frac{d\log(1/\delta)}{T+\Toff}}\,,
        \\[0pt]
        &\textstyle
        \text{$k$-uncovered $(\forall T\ge k)$}:
        &&\frac{k}{\sqrt T} +\sqrt{\frac{k\log(1/\delta)}{T}}
        \,,
    \end{align*}
    Let \(0<\delta<(12800\sqrt e)^{-1}\). Suppose an algorithm, possibly adaptive, has simple regret at most $\eps > 0$ with probability at least \(1-\delta\), for every \(\theta\in\mathbb R^d\). Then \(\epsilon\) is lower bounded, up to universal constants, by the corresponding expressions above, for \(d\ge2\) in the equally-covered regime and \(k\ge2\) in the \(k\)-uncovered regime.
\end{restatable}
The different budget requirements in this corollary reflect the number of directions that still need online coverage.
In the equally-covered regime the offline data already span every direction, so the bound remains meaningful even at $T=0$.
In the $k$-uncovered regime, the condition $T \ge k$ is sufficient to give each poorly covered direction at least one online action. 
Note that $k=d$ in the $k$-uncovered case recovers the online simple regret rate from \citet[Table 1]{pmlr-v336-maiti26a}.

\section{Regret Minimization with Offline Data} \label{sec:gw-oope} \label{SEC:GW-OOPE}

In the pure online setting, \citet{wagenmaker2021experimental} introduced the Gaussian Width Phased Elimination (GW-PE)%
\footnote{
    They actually named it Gaussian Width \emph{Action} Elimination (GW-AE), since action and phased elimination are used interchangeably in the literature.
}
algorithm, which uses the Gaussian width concentration bound (\Cref{lem:gw-concentration}) to select an optimal design in each phase.
We now present the Gaussian Width Offline-to-Online Phased Elimination (\gwoope) algorithm (\Cref{alg:gw-oope}), which combines GW-PE with offline data.
The main difference from \oope of \citet{vijayan2026regret} is the use of the Gaussian width concentration bound (\Cref{lem:gw-concentration}) in place of a union bound (\Cref{lem:union-bound-concentration}) to select the design in each phase.

Incorporating offline data in a phased elimination algorithm presents some challenges.
Although in practice one might use all of the offline data in each phase, for analysis purposes this causes issues.
In particular, the phased elimination algorithm requires that rewards in each phase are independent of each other when conditioned on the previous phases, so that fixed design confidence bounds (e.g., \Cref{lem:union-bound-concentration,lem:gw-concentration}) can be used.

To split the offline data in \gwoope while maintaining independence across phases, the algorithm first limits the number of phases to $m\defeq\max\{1,\lceil\log_2T\rceil\}$ (\Cref{line:total-phases}). 
The choice of $m$ is such that after phase $m$ the remaining actions contribute only a constant regret.
The algorithm then uses the Gaussian noise structure to mimic an even split of the offline data across the $m$ phases (\Cref{line:offline-split}). In particular, for each phase $\ell \in [m]$ we can construct independent Gaussian vectors (see \Cref{lem:gw-oope-gaussian-split})
\begin{align*}
b_{\mathrm{off}}^\ell
\sim \cN(\Voff[\ell]\truetheta,\Voff[\ell]),
\qquad
\Voff[\ell]\defeq \Voff/m \,.
\end{align*}
To see the intuition, suppose that the offline samples could be partitioned into $m$ groups so that the design matrix of the actions in each group were exactly $\Voff/m$. 
Then, for the $\ell$-th group, the corresponding offline term in the least-squares estimate would be $\sum_{s\in I_\ell} \aoff_s \roff_s$, where $I_\ell$ indexes the samples in that group.
This vector would have distribution $\cN\!\left((\Voff/m)\truetheta,\Voff/m\right)$.
Thus, the constructed vector $b_{\mathrm{off}}^\ell$ has exactly the same distribution as the offline least-squares term that would arise from such an ideal split, even when the original offline actions cannot actually be partitioned in this way.

After splitting the offline data, \Cref{alg:gw-oope} proceeds by choosing a design $\hat N_\ell$ in each phase $\ell$ by optimizing the Gaussian width objective (\Cref{lem:gw-concentration}) for target accuracy $\epsilon_\ell = 2^{-\ell}$ (\Cref{line:design}). 
Then, similar to standard phased elimination, it chooses actions based on $\hat N_\ell$, observes the corresponding rewards, and computes the least-squares estimate $\hat\theta_\ell$ (\Cref{line:least-sq}). 
The estimator $\hat \theta_\ell$ uses the design matrix $\tilde V_\ell(N) \defeq \Voff[\ell]+V(N)$ along with the offline data $b_{\mathrm{off}}^\ell$ and the online data collected by the algorithm in phase $\ell$. Finally, using $\hat \theta_\ell$, the algorithm eliminates actions with estimated gaps greater than $4\eps_\ell$ (\Cref{line:elim-actions}).

\begin{algorithm}[t]
\caption{\gwoope}
\label{alg:gw-oope}
\begin{algorithmic}[1]
\setcounter{ALC@unique}{0}
\STATE \textbf{Input:} $\cA$, $\Doff$, $T \in \bN$, $\delta\in(0,1)$
\STATE $m \gets \max\{1,\lceil\log_2T\rceil\}$, $\cA_1\gets\cA$, $\ell\gets1$ \label{line:total-phases}
\STATE Construct independent $b_{\mathrm{off}}^1, \dots, b_{\mathrm{off}}^m$ according to \Cref{lem:gw-oope-gaussian-split} \label{line:offline-split}
\WHILE{$\ell\le m$ and $\cA_\ell$ is not a singleton} 
    \vspace{0.3em}
    \STATE $\eps_\ell \gets 2^{-\ell}$, $\beta_\ell \gets \sqrt{2\log(4\ell^2/\delta)}$
    \STATE Choose a minimum-size design $\hat N_\ell\in\alloc(\cA_\ell)$ satisfying
    $ 
        \rad_{\beta_\ell}(\cA_\ell,\tilde V_\ell(\hat N_\ell))\le\eps_\ell
    $ \label{line:design}
    \vspace{0.3em}
    \STATE $n_\ell\gets|\hat N_\ell|$
    \STATE If $n_\ell$ exceeds the remaining online budget then \textbf{break}
    \STATE Take actions using $\hat N_\ell$, observe $(A_{\ell,t},R_{\ell,t})_{t=1}^{n_\ell}$
    \vspace{0.2em}
    \STATE
    $ 
        \hat\theta_\ell\gets\tilde V_\ell(\hat N_\ell)^{-1}
        \paren{b_{\mathrm{off}}^\ell+\sum_{t=1}^{n_\ell}R_{\ell,t}A_{\ell,t}}
    $ \label{line:least-sq}
    \STATE
    $
        \cA_{\ell+1}\gets\braces{a\in\cA_\ell: \max_{a'\in\cA_\ell}\ip{a'-a,\hat\theta_\ell}\le4\eps_\ell}
    $ \label{line:elim-actions}
    \vspace{0.3em}
    \STATE $\ell\gets\ell+1$
\ENDWHILE
\IF{$\ell=1$}
    \STATE Pull any action in $\cA_1$ for all remaining rounds
\ELSE
    \STATE Choose $\hat a\in\argmax_{a\in\cA_\ell}\ip{a,\hat\theta_{\ell-1}}$ and pull it for all remaining rounds
\ENDIF
\end{algorithmic}
\end{algorithm}

To state the regret bound, define the set of $\eta$-optimal actions by
\begin{align*}
    \textstyle \cA(\eta)\defeq\{a\in\cA:\ip{\opta-a,\truetheta}\le\eta\}.
\end{align*}
For each phase $\ell \in [m]$, let

\begin{align}
    n_\ell^\star\defeq\min\bigl\{|N|:\;&N\in\alloc(\cA(4\eps_\ell))\,,\rad_{\beta_\ell}(\cA(12\eps_\ell),\tilde V_\ell(N))\le\eps_\ell\bigr\}
    \label{eq:gw-oope-comparison-budget}
\end{align}
be a high-probability upper bound on the number of online samples $n_\ell$ needed in phase $\ell$.
When no feasible design exists, we set $n_\ell^\star = +\infty$.

\begin{restatable}[\gwoope upper bound]{theorem}{thmGWOOPEUpperBound}
    \label{thm:gw-oope-upper-bound}
    For any compact $\cA \subset \bR^d$,
    with probability at least $1-\delta$, the regret of \gwoope satisfies

    \begin{align}\label{eq:gw-oope-oracle-bound}
        \regret_T(\cA) \le 12\sum_{\ell=1}^m\eps_\ell\min\{n_\ell^\star,T\}
        + 2T\min\braces*{\eps_m,\rad_{\beta_m}(\cA(12\eps_m),\Voff[m])}.
    \end{align}
\end{restatable}
The first term in \Cref{eq:gw-oope-oracle-bound} accounts for the regret incurred until (and including) phase $m$, and the second term accounts for the regret incurred by actions after phase $m$. 
If there is a large amount of equally-covered offline data then the error based on just the offline data can be smaller than the accuracy in the final phase $\eps_m$, which is accounted for by the pure offline error term $\rad_{\beta_m}(\cA(12\eps_m),\Voff[m])$.

\subsection{Regret Lower Bounds}
We next provide a minimax regret lower bound when $\cA = B_2^d$ and show that \gwoope matches it for the spectral coverage regimes in \Cref{def:spectral-coverage-regimes}. 
Let $\pioff(a) \coloneqq \Toff(a)/\Toff$ be the empirical offline design distribution.
Rotating coordinates, we may assume
\[
  \bar V_{\textrm{off}} \;\coloneqq\; \bE_{A \sim \pioff}\big[AA^\top\big]
  \;=\; \diag(\bar\nu_1,\dots,\bar\nu_d)\,
\]
with $\sum_{i=1}^d \bar\nu_i = \bE_{\pioff}\norm{A}_2^2 \le 1$. Define
\begin{align*}
    f^\star = \min_{w \in \Delta_d} f(w), \quad w^\star \in \argmin_{w \in \Delta_d} f(w), \quad \alpha_i = r\sqrt{w_i^\star}\,,
\end{align*}
and $r = \min\set{1, T/f^\star}$ where
\begin{equation}\label{eq:f-variational}
  \textstyle f(w) \coloneqq \sum_{i=1}^d \sqrt{w_i}\,\big(4Tw_i + 2\big)
  \sqrt{\bar\nu_i \Toff + T w_i + 1}\,,
\end{equation}
and $\Delta_d$ is the probability simplex.
Consider the hypercube parameter family $\Theta \coloneqq \{ \theta \in \R^d \,:\, \theta_i = \pm \alpha_i \}$. 

\begin{restatable}[Regret lower bound]{theorem}{ThmRegLB}\label{thm:reg-lb}
Let $\cA = B_2^d$. For any algorithm, there exists $\theta \in \Theta$ such that
$\bE_\theta[\regret_T(\cA)] \ge \frac14 \min\set{T,\; T^2/f^\star}$, 
where $f^\star$ is the minimum value of $f$ given by \Cref{eq:f-variational}.
\end{restatable}

\begin{restatable}[Regret for spectral coverage]{corollary}{corRegLB}
\label{cor:reg-lb}
Under the coverage conditions in \Cref{def:spectral-coverage-regimes}, run \gwoope with
$\delta = 1/(2(T+\Toff))$ in the equally-covered regime and $\delta = 1/(2T)$ in the
$k$-uncovered regime. Then its expected regret is bounded, up to logarithmic and constant
factors, by
\begin{align*}
  &\text{equally-covered}: && \frac{dT}{\sqrt{T+\Toff}} \,, \\[0pt]
  &\text{$k$-uncovered}:   && k\sqrt{T} \,.
\end{align*}
Furthermore, for any algorithm, there exists $\theta \in \Theta$ such that for the equally-covered case with $d \le 2T$ and $d \le 8\sqrt{3}\sqrt{T+\Toff}$, the expected regret is lower bounded by the same expression as above up to constants.  
The same holds for the $k$-uncovered case with $k \le T$ and $k \le 8\sqrt{3}\sqrt{T}$.
\end{restatable}
Importantly, the $k$-uncovered bound improves the leading $\sqrt{kdT}$ term from the \oope \citep{vijayan2026regret} guarantee to $k\sqrt T$.
Again, note that $k=d$ in the bound above recovers the pure online regret rate \citep[Chapters 19, 24]{lattimore2020bandit}.

\section{Conclusion} \label{sec:conclusion}
Offline data with function approximation can pose challenges in designing optimal online algorithms if the sensitivity of the reward prediction error is not taken into account properly. 
We showed this phenomenon in the stochastic linear bandit problem, and offered a solution based on tighter analysis of the prediction error. 
Using that analysis, we provided algorithms for pure exploration and regret minimization with matching lower bounds for various offline data coverage regimes, improving upon the existing results. 

For future work, it would be interesting to see if LinUCB and its analysis can be modified to avoid the $\sqrt{d/k}$ gap in the regret bound for the $k$-uncovered case (see \Cref{sec:linucb-prediction-error}). 
It would also be interesting to determine whether the analysis of Thompson sampling also suffers a multiplicative $\sqrt{d/k}$ gap.
Our lower bounds match our upper bounds up to logarithmic factors in the stated coverage regimes (see \Cref{def:spectral-coverage-regimes}), but we do not obtain matching bounds for arbitrary offline coverage. 
An interesting direction is to obtain matching upper and lower bounds for any offline coverage. 
Finally, designing optimal offline-to-online algorithms for adaptive offline data and the full reinforcement learning setting remains open.

\subsection*{AI use statement}

In this work, we used generative AI tools to: 
help develop theoretical models or conceptual frameworks;
formulate mathematical claims;
provide critical ingredients for proving mathematical claims; 
assist in the writing of proofs; 
propose or refine hypotheses; 
design or provide feedback on research methodology; 
interpret results%
.
We have not used generative AI tools to: 
assist with translation%
. 
The use of generative AI tools for the following tasks is not applicable to this work:
generate synthetic data sets;
implement methods; 
clean and reformat datasets; 
support qualitative and thematic data analysis%
.
Additionally, we used generative AI tools to: 
draft parts of the appendix;
summarize or analyze existing literature;
discover research topics or identify gaps;
brainstorming; 
sourcing/searching for information; 
identify relevant literature%
.
We have reviewed all AI-assisted work.
All LLM-assisted proofs were checked line by line by the authors.
We take responsibility for the final content of this work, including text, claims, or artifacts produced with the aid of generative AI.

\bibliographystyle{plainnat}
\bibliography{references}

\clearpage

\appendix
\crefalias{section}{appendix}
\crefalias{subsection}{appendix}
\section*{Appendix}
\startcontents[app]
\printcontents[app]{}{1}{} 
\clearpage

\section{Notation}\label{app:notation}
\begin{table}[H]
\centering
\label{tab:notation}
\renewcommand{\arraystretch}{1.5}
\begin{tabular}{@{}ll@{}}
\toprule
\textbf{Symbol} & \textbf{Description} \\
\midrule
$\N_0$ & Set of nonnegative integers \\
$\cM_1(S)$ & Set of probability distributions over a set $S$\\
$\ip{x,y}$ & Inner product $x^\top y$ \\
$\norm{x}_A$ & $\sqrt{x^\top A x}$ for a positive semidefinite matrix $A$ \\
$\Tr(A)$ & Trace of a matrix $A$ \\
$\det(A)$ & Determinant of a matrix $A$ \\
$M\cX$ & $\set{Mx \,:\, x \in \cX \subseteq \R^d}$ for a matrix $M$ \\
$B_p^d$ & $\braces*{x\in\bR^d \,:\, \norm{x}_p\le1}$\\
$\overset{d}{=}$ & Equality in distribution \\
$\delta$ & Confidence parameter in $(0,1)$ \\
$\asymp, \gea, \lea$ & Equalities/inequalities which hide only constant factors \\
$\tilde O, \tilde \Omega, \tilde \Theta$ & Bachmann–Landau notation with hidden log factors \\
\bottomrule
\end{tabular}
\caption{Summary of notation used in the paper.}
\end{table}

\section{Related Work}\label{sec:related-works}
We review work that combines historical data with subsequent online interaction, and work on the Gaussian-width tools used in our analysis.

\paragraph{Minimizing Online Regret with Offline Data.} 
In the multi-armed bandit setting, \citet{shivaswamy2012multi} proposed $\epsilon$-greedy and upper-confidence-bound algorithms that use historical observations, and \citet{cheung2024leveraging} gave upper and lower bounds allowing distribution shift between the offline and online rewards. \citet{he2024learning} considered heterogeneous offline and online feedback. In linear bandits with finitely many actions, \citet{hao2023leveraging} analyzed the Bayesian regret of Thompson sampling initialized with expert demonstrations.

The closest works to ours are \citet{vijayan2026regret} and \citet{chandak2026offlinetoonline}. \citet{vijayan2026regret} proposed Offline-Online Phased Elimination, based on an extended D-optimal design, and obtained frequentist upper and lower bounds that depend on the eigenspectrum of the offline design. Our analysis identifies a regime in which their cardinality-based upper bound has an additional $\sqrt{d/k}$ factor and replaces it by a Gaussian-width bound. \citet{chandak2026offlinetoonline} combined optimism and pessimism to obtain guarantees relative to both an optimal online action and an offline reference. \citet{sentenac2025balancing} studied the same problem in the multi-armed bandit setting. Our objective is different: under a shared linear reward parameter, we characterize how directional offline coverage affects regret relative to the optimal action.

\citet{han2026direction} studied linear contextual bandits when the offline and online reward parameters may differ. They introduced a direction-dependent certificate for this bias and obtained improved regret when well-covered directions also have small bias. We assume no offline-to-online parameter shift and instead study how anisotropic coverage affects the attainable regret even under a common parameter.

\paragraph{Pure Exploration with Offline Data.}
\citet{agrawal2023optimal} studied fixed-confidence best-arm identification in multi-armed bandits with unbiased offline observations, giving sample-complexity lower bounds and asymptotically matching algorithms. \citet{yang2025best} allowed the offline observations to be biased and derived an instance-dependent lower bound that is matched in certain regimes. These works concern unstructured multi-armed bandits and fixed-confidence identification. We consider fixed-budget simple regret in linear bandits, where observations of one action provide information about other actions through their features.

\paragraph{Experimental Design and Gaussian Width.}
Our upper bounds build on the experimental-design approach of \citet{wagenmaker2021experimental}, who used Gaussian-width confidence bounds to avoid cardinality-based union bounds in online linear bandits. \citet{pmlr-v336-maiti26a} used related Gaussian-width quantities to characterize fixed-design $\epsilon$-best-arm identification and the possible benefit of adaptivity. We incorporate a fixed offline information matrix into these objectives and show that uneven offline coverage can make the standard union-bound analysis loose even for action sets where it has the correct dimension dependence in the purely online problem.

The analogous combination of offline data and online interaction has also been studied in reinforcement learning. For linear-structure MDPs, \citet{wagenmaker2023leveraging} characterized the online sample complexity achievable from a given offline dataset. Our setting concerns cumulative and simple regret in stochastic linear bandits rather than policy identification in an MDP.

\paragraph{Online Learning with Hints and Queries.}
Our problem setting is also related to online learning with hints or prior information \citep{dekel2017online} where the offline dataset can be thought of as a structured hint. \citet{cutkosky2022leveraging} investigated the linear bandit problem when a hint to the optimal action is given to the algorithm. They proposed an algorithm that improves upon the standard regret guarantees when the hint is accurate, along with lower bounds. \citet{wei2020taking} studied the contextual bandit problem in both stochastic and adversarial settings where hints are given as predictions of losses, showing that the algorithm can achieve improved regret bounds when the prediction error is small. \citet{bhaskara2023bandit} studied a variation of the adversarial linear bandit optimization setting where either hints are received before every online interaction or the learner can query to get the extra information. Along similar lines, \cite{hou2026benefits} considered the multi-armed bandit problem with a free exploration phase before the online learning starts, in which the learner can query any action without incurring any loss.

\section{Challenges with Adaptively Collected Offline Data} \label{app:adapt-offline-data}
In this section, we discuss how adaptively collected (offline) data could be challenging to deal with.

First, recall that for fixed designs, $\ip{x, \hat\theta-\theta}\sim\mathcal N(0,\norm{x}^2_{V_T^{-1}})$,
so the error in a fixed direction is $O(\norm{x}_{V_T^{-1}}\sqrt{\log(1/\delta)})$.
We show that this fails for adaptive designs, even when $V_T$ has the same law as
under a fixed design.
 
Let $y_t=\ip{a_t, \theta}+\eta_t$ with $\eta_t\overset{\text{iid}}{\sim}\mathcal N(0,1)$,
$\theta=0$, and $T=d$. For $t<d$ play $a_t=e_t$, and at $t=d$ play
$a_d=e_d+\sum_{i<d}s_ie_i$ for some $s\in\{\pm1\}^{d-1}$. Consider
\begin{itemize}
  \item[(F)] \emph{randomized non-adaptive design:} $s$ uniform on $\{\pm1\}^{d-1}$, independent of $(\eta_t)_t$;
  \item[(A)] \emph{adaptive design:} $s_i=-\sign(y_i)=-\sign(\eta_i)$ for $i<d$.
\end{itemize}
 
\begin{proposition}\label{prop:adaptive}
Let $Z=(\hat\theta_d-\theta_d)/\norm{e_d}_{V_T^{-1}}$, where $\hat\theta$ is the
least-squares estimate. Under both (F) and (A), $V_T$ has the same law and
$\norm{e_d}_{V_T^{-1}}=\sqrt d$ almost surely. Under (F), $Z\sim\mathcal N(0,1)$.
Under (A), for all $\delta\in(0,1)$, with probability at least $1-\delta$,
\[
  Z\;\ge\;\sqrt{\tfrac{2}{\pi}}\,\frac{d-1}{\sqrt d}-\sqrt{2\log(1/\delta)} .
\]
\end{proposition}
 
\begin{proof}
Let $A\in\R^{d\times d}$ have rows $a_1^\top,\dots,a_d^\top$, so $V_T=A^\top A$ and
\[
  A=\begin{pmatrix} I_{d-1} & 0\\ s^\top & 1\end{pmatrix},\qquad
  A^{-1}=\begin{pmatrix} I_{d-1} & 0\\ -s^\top & 1\end{pmatrix}.
\]
Since $A$ is invertible, $\hat\theta=A^{-1}y$, and hence
\begin{equation}\label{eq:err}
  \hat\theta_d-\theta_d=\ip{c, \eta}=\eta_d-\sum_{i<d}s_i\eta_i,
  \qquad c=A^{-\top}e_d=(-s,1).
\end{equation}
It follows that $\norm{e_d}^2_{V_T^{-1}}=\norm{A^{-\top}e_d}^2=\norm{c}^2=d$ for every $s$.
 
$V_T$ is a function of $s$ alone. By symmetry of $\eta_i$, $s$ is uniform on
$\{\pm1\}^{d-1}$ under (A) as well as under (F), so $V_T$ has the same law under both.
 
Under (F), $s\perp\eta$, so~\eqref{eq:err} gives $\hat\theta_d-\theta_d\sim\mathcal N(0,d)$,
and therefore $Z\sim\mathcal N(0,1)$.
 
Under (A), $-s_i\eta_i=|\eta_i|$, so $\hat\theta_d-\theta_d=\eta_d+\sum_{i<d}|\eta_i|$,
which has mean $(d-1)\sqrt{2/\pi}$. The map $x\mapsto|x|$ is $1$-Lipschitz, so each
$|\eta_i|-\bE|\eta_i|$ is $1$-subgaussian. The error is therefore $\sqrt d$-subgaussian about its mean, and $\Pr\big(\hat\theta_d-\theta_d\le(d-1)\sqrt{2/\pi}-t\big)\le e^{-t^2/(2d)}$.
Setting $t=\sqrt{2d\log(1/\delta)}$ and dividing by $\sqrt d$ gives the claim.

\end{proof}
 
\begin{remark}

The shape of the standard confidence sets depends on the data only through $V_T$, and $V_T$ has the same
law in both designs, so they report identical uncertainty. The difference is that
under (A), $V_T$ is not independent of $\eta$: the weight $-s_i$ on $\eta_i$
in~\eqref{eq:err} is chosen after $\eta_i$ is observed, so the error is no longer
linear in $\eta$. Geometrically, $c$ is aligned with $\eta$ on the first $d-1$
coordinates, whereas under (F) it is nearly orthogonal to $\eta$. Only $d-1$
coordinates can be aligned, because $\eta_d$ is revealed after $a_d$ is chosen. The
restriction $\theta=0$ is without loss for a lower bound, since a uniform confidence
bound must hold for every $\theta$. Choosing each action $n$ times rescales both the
error and the width by $n^{-1/2}$, so the ratio is unchanged. 
\end{remark}

\section{Missing Results and Details for Section \texorpdfstring{\ref{sec:oope-linucb}}{Section 3}} \label{app:oope-linucb}

\subsection{Discussion of Previous Regret Lower Bound} \label{app:vijayan-lower-bound}
In this section, we show that the lower bound of \citet{vijayan2026regret} (Proposition 4.10) can be loose, and might not capture the hardness of the problem.

First, recall that $\pioff(a) \coloneqq \Toff(a)/\Toff$ and $(\bar\nu_i)_{i=1}^d$ are the eigenvalues of $V_{\pioff} \coloneqq \bE_{A \sim\pioff}[AA^\top]$. For their regret lower bound, \citet{vijayan2026regret} consider linear bandit problems parametrized as $p \coloneqq (\pioff, \cA, \theta, \Toff, T)$ such that $\bar\nu_i/\max_{a \in \cA} \norm{a}_2^2$ for all $i \in [d]$ is held constant. That is, $v_i = \bar\nu_i/\max_{a \in \cA} \norm{a}_2^2$ with $\sum_{i=1}^d v_i \le 1$. The set of such problem instances is denoted as $\cP_{v, \Toff, T}^d$. When $|\cA| = O\left(d^d \right)$, they show the following minimax regret lower bound (paraphrased).
\begin{theorem}[Proposition 4.10 in \citep{vijayan2026regret}] \label{thm:eigenspectrum-lb}
For problem class $\cP_{v, \Toff, T}^d$ with $|\cA| = O\left(d^d\right)$, any adaptive algorithm which takes $\Toff$ offline samples as input satisfies
    \begin{align}\label{eq:vijayan-lb}
        \sup_{p \in \cP_{v, \Toff, T}^d} \bE[\regret_T(p)] \gtrsim T \sup_{w \in \Delta_d: w_i \ge v_i} \sum_{i=1}^d \frac{1}{\sqrt{T + \Toff\frac{v_i}{w_i}}} \,.
    \end{align}
\end{theorem}
This bound can be loose by an arbitrarily large factor. Since $\sum_i v_i = \bE_{a\sim\pioff}\norm{a}_2^2/\max_{a\in\cA}\norm{a}_2^2$, we have $\sum_i v_i = 1$ whenever $\pioff$ is supported on maximum-norm arms (e.g., when all arms have equal norm). 
The constraints $w\in\Delta_d$ and $w\ge v$ then force $w=v$, so for any such $v$ with positive entries the right-hand side of \Cref{eq:vijayan-lb} equals $dT/\sqrt{T+\Toff}$, irrespective of the spectrum. 
Now fix integers $1\le k<d$ and $j\ge 1$, let $\Toff = k+j(d-k)$, and take $\cA = \{\pm e_i\}_{i=1}^d\cup\big(\{\pm 1/\sqrt{k}\}^k\times\{0\}^{d-k}\big)$ with $\pioff(e_i) = 1/\Toff$ for $i\le k$ and $\pioff(e_i) = j/\Toff$ for $i>k$, so that each of the first $k$ directions receives exactly one offline sample. 
All arms have unit norm, $|\cA|\le 2d+2^k$, and $V_{\pioff} = \operatorname{diag}(v)$ with $v_i = \pioff(e_i)$, so this instance lies in $\cP^d_{v,\Toff,T}$ with $\sum_i v_i = 1$. For $\theta_s = \frac{1}{8}\sqrt{k/T}\,(s,0_{d-k})$ with $s\in\{\pm1\}^k$, the action set $\cA$ lies in the unit ball and contains the unit-ball maximizer $\theta_s/\norm{\theta_s}_2 = (s,0_{d-k})/\sqrt{k}$. 
Hence the unit-ball lower-bound argument of \citet[Sec.~24.2]{lattimore2020bandit} applies, with the single offline sample per direction adding only $O(k/T)$ to each KL term. 
It shows that for $T\ge 2k$ every algorithm incurs regret at least $k\sqrt{T}/32$ on some $\theta_s$, where $\norm{\theta_s}_2 = k/(8\sqrt{T})\le 1$ once $T\ge k^2/64$. 
The minimax regret is therefore at least $k\sqrt{T}/32$ for every $j$, whereas \Cref{eq:vijayan-lb} vanishes as $j\to\infty$, so there is a gap of order $(k/d)\sqrt{1+\Toff/T}$. 
The reason is that when $w=v$, the offline data in the hard instance of \citet{vijayan2026regret} covers every direction to the full extent of the action set, so no direction is poorly covered relative to the available actions.

Further, note that in \Cref{thm:eigenspectrum-lb}, fixing the normalized offline spectrum does not fix the action-set geometry. Consequently, a lower bound for the class need not characterize the difficulty of each fixed action set in the class.

\subsection{LinUCB Analysis for Offline-to-Online Learning}\label{app:linucb}
Let $0 < \nu_1 \le \dots \le \nu_d$ be the eigenvalues of $\Voff$ and
\begin{align*}
    d_e = \max\set{i \in [d] \,:\, (i-1) \nu_i \le \frac{T}{\log(1+T/\nu_1)}} \,.
\end{align*}
Following the standard LinUCB analysis \citep[Chapter 19]{abbasi2011improved,lattimore2020bandit} and the analysis of \citet{valko2014spectral}, it can be shown that LinUCB's regret for $\cA = B_2^d$ when the confidence sets are initialized by the offline data \citep{chandak2026offlinetoonline,vijayan2026regret} is%
\footnote{Here we overload the regret notation and make the dependence on the algorithm explicit.}
\begin{equation}\label{eq:linucb-regret}
    \regret_T(\linucb) \lesssim \sqrt{T\beta_T d_{e} \log\left(1+\frac{T}{\nu_1}\right)} \,,
\end{equation}
where $\sqrt\beta_T \lesssim \sqrt{d} + \sqrt{d_e \log(1+T/\nu_1)} + \sqrt{\log(2/\delta)}$. If the offline data is equally-covered, that is, $\nu_i \ge c\Toff/d$ for all $i \in [d]$ then $d_e = \Theta(1)$ as long as $\Toff \gtrsim dT$. When the offline data is fully uncovered, that is, $\nu_i \le 1$ for all $i$ then $d_e = \Theta(d)$ for large enough $T$ (i.e. $T \ge (d-1)\log(1+T/\nu_1)$). Finally, for the $k$-uncovered case, $d_e\le2k$ for every $T\ge1$, and $d_e=\Theta(k)$ whenever
$T\ge(k-1)\log(1+T/\nu_1)$. Therefore, with probability $1-\delta$, the regret of LinUCB with offline data can be upper bounded as
\begin{equation*}
\regret_T(\linucb)
=
\begin{cases}
    \widetilde O\!\left(
        \dfrac{dT}{\sqrt{T_{\mathrm{off}}}}
    \right),
    & \text{equally-covered }
      (T_{\mathrm{off}}\gtrsim dT),
    \\[6pt]
    \widetilde O\!\left(\sqrt{kdT}\right),
    & \text{$k$-uncovered }(T\gtrsim k).
\end{cases}
\end{equation*}
Thus, the standard regret guarantee for LinUCB is larger by a multiplicative factor of \(\sqrt{d/k}\) in the \(k\)-uncovered regime similar to phased elimination. A more refined analysis to fix the issue is left for future work.
\begin{remark}
    The warm started LinUCB bound in \citet[Proposition A.11]{vijayan2026regret} contains an initial-radius term controlled by \(\|\theta\|_{V_0}\), which is of order \(d\) in their Appendix A.10 example. In contrast, when the least-squares estimator incorporates the offline rewards, the offline contribution can be controlled through \(\|\hat\theta_0-\theta\|_{V_{\rm off}}\), which is of order \(\sqrt d\) up to confidence factors. This gives the regret bound in \Cref{eq:linucb-regret}.
\end{remark}

\subsection{\oope Regret Bound for Different Offline Coverage Regimes} \label{app:oope-regret-bound}
We first state the regret bound for the Offline-Online Phased Elimination Algorithm (OOPE) by \citet{vijayan2026regret} and specialize its effective dimension to our coverage regimes.
Let $L_\cA=\max_{a\in\cA}\norm{a}_2$, and let $\nu_1,\ldots,\nu_d$ be the eigenvalues of $\Voff$.
In terms of the unnormalized offline matrix, their effective dimension is
\begin{align}
    d_{\mathrm{eff}}\defeq\min\left\{
        \sum_{i=1}^d\frac{1}{1+\nu_i/(TL_\cA^2)},
        \ T\max_{a\in\cA}\norm{a}_{\Voff^{-1}}^2
    \right\}.
    \label{eq:oope-effective-dimension}
\end{align}
This is the same as their definition: if $\pi_{\mathrm{off}}$ is the empirical distribution of offline actions, then $V(\pi_{\mathrm{off}})=\Voff/\Toff$ and $g_\cA(\pi_{\mathrm{off}})=\max_{a\in\cA}\norm{a}_{V(\pi_{\mathrm{off}})^{-1}}^2$.

Their theorem has confidence $1-1/T$ and assumes enough offline observations of each action in the offline support (Assumption 3.4 in \citealp{vijayan2026regret}).
For a self-contained statement, we use the following stronger sufficient condition:
\begin{align}
    \Toff(a)\ge3\left\lceil\log_2\left(2+\sqrt{T+\Toff}\right)\right\rceil
    \qquad\text{whenever }\Toff(a)>0.
    \label{eq:oope-offline-count-condition}
\end{align}
The original assumption can be used in place of this condition.
Neither is required for our algorithms.
This count condition is compatible with weak positive eigenvalues: $q$ pulls of $e_i/\sqrt q$ provide information one in direction $e_i$, while allowing $q$ to be large.

\begin{restatable}[\oope Regret Bound, \citet{vijayan2026regret}, Theorem 4.5]{theorem}{thmOope}
    \label{thm:oope}
    Let $\cA\subseteq B_2^d$ be finite, $\norm{\truetheta}_2\le1$, and $T\ge2$.
    Suppose the offline counts satisfy \Cref{eq:oope-offline-count-condition}.
    Set
    \begin{align*}
        \ell_{\max}=\left\lfloor\log_2\sqrt{4+\frac{T}{d_{\mathrm{eff}}\log(4|\cA|T)}}\right\rfloor.
    \end{align*}
    With probability at least $1-1/T$,
    \begin{align*}
        \regret_T(\cA)\le16\sqrt{6d_{\mathrm{eff}}T\log(4\ell_{\max}^2|\cA|T)}+4d(d+1).
    \end{align*}
\end{restatable}

For the spectral coverage regimes, their bound becomes:
\begin{restatable}{corollary}{corOope}
    \label{cor:oope}
    Let $\cA\subseteq B_2^d$ be a finite $h$-net, $h\le1/2$, containing an orthonormal eigenbasis of $\Voff$.
    Under the eigenvalue conditions in \Cref{def:spectral-coverage-regimes},
    \begin{align*}
        d_{\mathrm{eff}}\asymp
        \begin{cases}
            dT/(T+\Toff),&\text{equally-covered},\\
            k,&\text{$k$-uncovered, }T\ge k,\\
            d,&\text{uncovered, }T\ge d.
        \end{cases}
    \end{align*}
    If also $\log|\cA|=O(d\log(2dmT)+\log(2+\Toff))$, the \oope regret guarantee, under the conditions of \Cref{thm:oope}, becomes
    \begin{align*}
        \regret_T(\cA)
        = 
        \begin{cases}
            \bigOt(dT/\sqrt{T+\Toff}+d^2),&\text{equally-covered},\\
            \bigOt(\sqrt{kdT}+d^2),&\text{$k$-uncovered},\\
            \bigOt(d\sqrt T+d^2),&\text{uncovered},
        \end{cases}
    \end{align*}
    with probability at least $1-1/T$, where $m=\max\{1,\lceil\log_2T\rceil\}$.
\end{restatable}
\begin{proof}
    The included eigenvectors give $L_\cA=1$ and $\max_{a\in\cA}\norm{a}_{\Voff^{-1}}^2=1/\nu_1$.
    In the equally-covered case, $c\Toff/d\le\nu_1\le\Toff/d$, where the upper bound follows from $\tr(\Voff)\le\Toff$.
    Thus the second term in \Cref{eq:oope-effective-dimension} lies between $dT/\Toff$ and $dT/(c\Toff)$.
    By convexity of $x\mapsto(1+x)^{-1}$,
    \begin{align*}
        \sum_{i=1}^d\frac{1}{1+\nu_i/T}
        \ge\frac{d}{1+\Toff/(dT)}
        \ge\frac{dT}{T+\Toff}.
    \end{align*}
    The sum is also at most $d$.
    Taking the minimum proves $d_{\mathrm{eff}} \asymp dT/(T+\Toff)$.

    In the $k$-uncovered case, the first $k$ terms in the sum lie between $k/2$ and $k$ in total, while
    \begin{align*}
        \sum_{i=k+1}^d\frac{1}{1+\nu_i/T}
        \le\sum_{i=k+1}^d\frac{T}{\nu_i}\le k.
    \end{align*}
    Hence the first term in the minimum is between $k/2$ and $2k$.
    The second term is at least $T\ge k$ because $\nu_1\le1$.
    This proves $d_{\mathrm{eff}}\asymp k$.
    In the uncovered case, the sum lies between $d/2$ and $d$, and the second term is at least $T\ge d$.
    Substituting these equivalences into \Cref{thm:oope} proves the regret bounds.
\end{proof}

As proposed in Remark 4.6 of \citet{vijayan2026regret}, one can run \oope on a sufficiently fine net of the ball.
A volumetric $h$-net has at most $(1+2/h)^d$ points; adding the $d$ eigenvectors and the offline actions increases this by at most $d+\Toff$.
The choice $h\le1/(64dmT)$ can therefore satisfy the cardinality condition above.
The difference between the best expected reward on the ball and on the net is at most $h\norm{\truetheta}_2$, so passing from regret on the net to regret on the ball costs at most $Th$ under $\norm{\truetheta}_2\le1$.

\subsection{Concentration Results for Reward Prediction Error}

\lemUnionBoundConcentration*
\begin{proof}
    For any fixed action $a\in\cA$, \Cref{eq:prediction-error} implies
    \begin{align*}
        \ip{a,\hat\theta_N-\truetheta}\sim\cN(0,\norm{a}_{\tilde V(N)^{-1}}^2).
    \end{align*}
    The Gaussian tail bound \citep[Proposition~2.1.2]{vershynin2026high}, applied to both signs, therefore gives
    \begin{align*}
        \bP\left(|\ip{a,\hat\theta_N-\truetheta}|>
        \sqrt{2\log(2|\cA|/\delta)}\norm{a}_{\tilde V(N)^{-1}}\right)
        \le\delta/|\cA|.
    \end{align*}
    For $a=0$, the bound is immediate.
    Taking a union bound over $a\in\cA$ proves the result.
\end{proof}

\lemUBKuncovered*
\begin{proof}
    Let $u_1,\ldots,u_d$ be the eigenvectors, ordered by increasing
    eigenvalue, and let $U$ be the span of the first $k$.
    For an upper bound, divide the $T$ online pulls as evenly as
    possible across $u_1,\ldots,u_k$.
    Since $T\ge k$, each weak direction receives at least
    $T/(2k)$ pulls.
    When $k<d$, every remaining direction already has offline
    eigenvalue at least $(d-k)T/k\ge T/k$; when $k=d$, there are
    no remaining directions.
    Hence, in either case,
    $\tilde V(N)\succeq TI_d/(2k)$ and
    \begin{align*}
        \max_{a\in\cA}\norm{a}_{\tilde V(N)^{-1}}
        \le\sqrt{\frac{2k}{T}}.
    \end{align*}

    Conversely, for any exact design $N$, let $P_U$ denote the
    orthogonal projector onto $U$.
    Because the first $k$ offline eigenvalues are at most one and
    every online action has norm at most one,
    \begin{align*}
        \tr(P_U\tilde V(N)P_U)\le k+T\le2T.
    \end{align*}
    Thus $\lambda_{\min}(\tilde V(N))\le2T/k$.
    Let $v$ be a unit eigenvector for this smallest eigenvalue, and
    choose a net point $a$ with $\norm{a-v}_2\le1/2$.
    Then $\ip{a,v}\ge1/2$, and consequently
    \begin{align*}
        \max_{a\in\cA}\norm{a}_{\tilde V(N)^{-1}}
        \ge\frac{1}{2\sqrt{\lambda_{\min}(\tilde V(N))}}
        \ge\sqrt{\frac{k}{8T}}.
    \end{align*}
    Multiplying by $\sqrt{2\log(2|\cA|/\delta)}$ and using that
    $\log|\cA|$ is of order $d$ proves the equivalence.
\end{proof}

\lemGWConcentration*
\begin{proof}
    Fix a nonempty compact set $\cA$ and define
    \begin{align*}
        f(z)
        =\max_{a\in\cA}\ip{a,\tilde V(N)^{-1/2}z},
        \qquad z\in\bR^d.
    \end{align*}
    We first calculate the Lipschitz constant of $f$.
    Set
    \begin{align*}
        L=\max_{a\in\cA}\norm{a}_{\tilde V(N)^{-1}}.
    \end{align*}
    For any $z,z'\in\bR^d$ and $a\in\cA$, Cauchy--Schwarz gives
    \begin{align*}
        \ip{a,\tilde V(N)^{-1/2}z}
        &=
        \ip{a,\tilde V(N)^{-1/2}z'}
        +\ip{\tilde V(N)^{-1/2}a,z-z'}\\
        &\le
        f(z')
        +\norm{a}_{\tilde V(N)^{-1}}\norm{z-z'}_2\\
        &\le
        f(z')+L\norm{z-z'}_2.
    \end{align*}
    Taking the maximum over $a$ shows that
    $f(z)-f(z')\le L\norm{z-z'}_2$.
    Exchanging $z$ and $z'$ gives
    \begin{align*}
        |f(z)-f(z')|\le L\norm{z-z'}_2.
    \end{align*}
    Therefore $f$ is $L$-Lipschitz.

    By the definition of Gaussian width,
    \begin{align*}
        \bE f(g)=\gw\bigl(\tilde V(N)^{-1/2}\cA\bigr).
    \end{align*}
    Applying \Cref{lem:lipschitz-concentration} with failure
    probability $\delta/2$ gives
    \begin{align*}
        f(g)
        \le
        \gw\bigl(\tilde V(N)^{-1/2}\cA\bigr)
        +L\sqrt{2\log(2/\delta)}
    \end{align*}
    with probability at least $1-\delta/2$.

    To control negative prediction errors as well, observe that
    \begin{align*}
        f(-g)
        =\max_{a\in\cA}-\ip{a,\tilde V(N)^{-1/2}g}.
    \end{align*}
    Since $-g$ has the same distribution as $g$, the same bound
    holds for $f(-g)$ with probability at least $1-\delta/2$.
    A union bound over these two events therefore gives, with
    probability at least $1-\delta$,
    \begin{align*}
        \max_{a\in\cA}
        \left|\ip{a,\tilde V(N)^{-1/2}g}\right|
        &=\max\{f(g),f(-g)\}\\
        &\le
        \gw\bigl(\tilde V(N)^{-1/2}\cA\bigr)
        +\sqrt{2\log(2/\delta)}
        \max_{a\in\cA}\norm{a}_{\tilde V(N)^{-1}}\\
        &=
        \rad_{\sqrt{2\log(2/\delta)}}(\cA,\tilde V(N)).
    \end{align*}
    Combining the above bound with \Cref{eq:fixed-design-gaussian-error} gives the stated bound.
\end{proof}

\begin{lemma}[Tightness of the Gaussian-width confidence radius]
    \label{lem:gw-confidence-tightness}
    Let $\cA\subset\bR^d$ be nonempty and compact,
    $M\succ0$, 
    $\delta\in(0,1/10]$,
    and $g\sim\cN(0,I_d)$.
    If $\max_{a\in\cA}|\ip{a,M^{-1/2}g}| \le \eps$, for some $\eps\ge0$, with probability at least $1 - \delta$, then
    \[
        \eps
        \gea
        \rad_{\sqrt{2\log(2/\delta)}}(\cA,M).
    \]
\end{lemma}

\begin{proof}
    Write $L=\max_{a\in\cA}\norm{a}_{M^{-1}}$.
    If $L=0$, the result is immediate.
    Otherwise, by compactness choose $a^\star\in\cA$ satisfying $\norm{a^\star}_{M^{-1}}=L$.
    Since
    \[
        \ip{a^\star,M^{-1/2}g}
        \sim\cN(0,L^2),
    \]
    we have $|\ip{a^\star,M^{-1/2}g}|\overset{d}=L|G|$ for $G\sim\cN(0,1)$.
    The assumed confidence guarantee therefore implies $\bP(L|G|\le\eps)\ge1-\delta$.
    The Gaussian tail lower bound \citep[Proposition~2.1.2]{vershynin2026high} then gives $\eps\gtrsim L\sqrt{\log(1/\delta)}$.
    
    In the next part of the proof, we will also use the simpler bound $\eps\ge L$.
    Indeed, if $\eps<L$, then $\bP(L|G|\le\eps)<\bP(|G|\le1)<9/10$, contradicting $1-\delta\ge9/10$.

    It remains to control the Gaussian-width term.
    Let $f(g)=\max_{a\in\cA}\ip{a,M^{-1/2}g}$.
    This function is $L$-Lipschitz and $\bE f(g)=\gw(M^{-1/2}\cA)$.
    Applying \Cref{lem:lipschitz-concentration} to $-f$ with failure probability $1/2$ gives
    \[
        -f(g)
        \le
        -\bE f(g)+L\sqrt{2\log2}
    \]
    with probability at least $1/2$.
    Equivalently,
    \[
        f(g)
        \ge
        \gw(M^{-1/2}\cA)-L\sqrt{2\log2}.
    \]

    The assumed confidence event has probability at least $1-\delta$, so the two events occur simultaneously with probability at least $1/2-\delta>0$.
    On their intersection, $f(g)\le\max_{a\in\cA}|\ip{a,M^{-1/2}g}|\le\eps$.
    Therefore
    \[
        \gw(M^{-1/2}\cA)
        \le
        \eps+L\sqrt{2\log2}
        \le
        (1+\sqrt{2\log2})\eps,
    \]
    where the last inequality uses $L\le\eps$.

    We have therefore shown $\eps\gtrsim \gw(M^{-1/2}\cA)$ and $\eps\gtrsim L\sqrt{\log(1/\delta)}$.
    Combining the two bounds with $\max\set{x, y} \ge (x+y)/2$ and using $\log(2/\delta)\asymp\log(1/\delta)$ for $\delta\le1/10$ gives
    \[
        \eps
        \gtrsim
        \rad_{\sqrt{2\log(2/\delta)}}(\cA,M).
    \]
\end{proof}

\corGWKuncovered*
\begin{proof}
    Use the balanced design from the upper-bound part of
    \Cref{lem:ub-k-uncovered}.
    Since $T\ge k$, each of the $k$ weak eigenvectors receives at
    least $\lfloor T/k\rfloor\ge T/(2k)$ online pulls.

    Suppose first that $k<d$.
    The remaining $d-k$ directions already have offline eigenvalues
    at least $(d-k)T/k$. Consequently,
    \begin{align*}
        \tr(\tilde V(N)^{-1})
        &\le
        k\frac{2k}{T}
        +(d-k)\frac{k}{(d-k)T}
        =\frac{2k^2+k}{T},\\
        \lambda_{\max}(\tilde V(N)^{-1})
        &\le\frac{2k}{T}.
    \end{align*}

    When $k=d$, all directions are weak, so the same balanced design
    allocates the online samples across all $d$ eigenvectors.
    Since every direction receives at least $T/(2d)$ pulls and the
    offline information is positive semidefinite,
    \begin{align*}
        \tr(\tilde V(N)^{-1})
        \le\frac{2d^2}{T}
        \le\frac{2k^2+k}{T},
        \qquad
        \lambda_{\max}(\tilde V(N)^{-1})
        \le\frac{2d}{T}
        =\frac{2k}{T}.
    \end{align*}
    Thus, in either case,
    $\tr(\tilde V(N)^{-1})\le(2k^2+k)/T$ and
    $\lambda_{\max}(\tilde V(N)^{-1})\le2k/T$.

    For both action sets, $\cA\subseteq B_2^d$.
    Thus Cauchy--Schwarz gives
    $\ip{a,\tilde V(N)^{-1/2}g}\le
    \norm{\tilde V(N)^{-1/2}g}_2$ for every $a\in\cA$.
    Taking the maximum and then the expectation, with
    $g\sim\cN(0,I_d)$, yields
    \begin{align*}
        \gw(\tilde V(N)^{-1/2}\cA)
        &\le\bE\norm{\tilde V(N)^{-1/2}g}_2\\
        &\le\sqrt{\bE\norm{\tilde V(N)^{-1/2}g}_2^2}\\
        &=\sqrt{\tr(\tilde V(N)^{-1})}.
    \end{align*}
    The second inequality is Jensen's inequality, and the last
    equality uses $\bE[gg^\top]=I_d$.
    Similarly, since every action has Euclidean norm at most one,
    \begin{align*}
        \max_{a\in\cA}\norm{a}_{\tilde V(N)^{-1}}
        \le\sqrt{\lambda_{\max}(\tilde V(N)^{-1})}.
    \end{align*}

    Substituting these bounds into \Cref{lem:gw-concentration}
    gives, with probability at least $1-\delta$,
    \begin{align*}
        \max_{a\in\cA}|\ip{a,\hat\theta_N-\truetheta}|
        \le\sqrt{\frac{2k^2+k}{T}}
        +\sqrt{\frac{4k\log(2/\delta)}{T}}.
    \end{align*}
    Since $k\ge1$, this proves the claimed bound.
    The same design is feasible on the net because the eigenvectors
    were included.
\end{proof}

\subsection{On the Tightness of Union Bounds}
\label{app:tightness-of-union-bounds}

We show that the union bound in
\Cref{lem:union-bound-concentration} is tight up to constants
for the hypercube and a constant-resolution net of the
Euclidean ball, when the design is chosen to minimize that
bound.
Here we consider the pure online setting $\Voff=0$ and
restrict attention to designs with $V(N)\succ0$.

By \Cref{lem:gw-confidence-tightness}, it is enough to
compare the optimized union bound with the Gaussian-width
confidence radius.
G-optimal design controls the largest prediction variance.
We then calculate a lower bound on the Gaussian width
for each action set.

\begin{lemma}[Tightness of the optimized union bound]
    \label{lem:union-bound-tight-online}
    Let $T\ge d(d+1)$ and $\delta\in(0,1/10]$.
    Suppose that either $\cA=\{\pm1\}^d$, or
    $\cA\subseteq B_2^d$ is a finite $1/2$-net with
    $\log|\cA|\asymp d$.
    Let $N_{\mathrm G}$ minimize
    $\max_{a\in\cA}\norm{a}_{V(N)^{-1}}$
    over $N\in\alloc_T(\cA)$ with $V(N)\succ0$.
    Then the union-bound confidence radius satisfies
    \[
        \max_{a\in\cA}
        \norm{a}_{V(N_{\mathrm G})^{-1}}
        \sqrt{2\log(2|\cA|/\delta)}
        \asymp
        \frac{d}{\sqrt T} + \sqrt{\frac{d\log(1/\delta)}{T}}.
    \]
    Conversely, if $N \in \alloc_T(\cA)$ is a fixed design with $V(N)\succ0$, that satisfies $\max_{a\in\cA} |\ip{a,\hat\theta_N-\truetheta}| \le \eps$, for some $\eps\ge0$, with probability at least $1 - \delta$, then
    \begin{align*}
        \eps 
        \gea 
        \frac{d}{\sqrt T} + \sqrt{\frac{d\log(1/\delta)}{T}}.
    \end{align*}
\end{lemma}

\begin{proof}
    By the Kiefer--Wolfowitz theorem \citep[Theorem~21.1]{lattimore2020bandit}, there is an approximate design $\pi$ supported on at most $d(d+1)/2$ actions such that $\max_{a\in\cA}\norm{a}_{V(\pi)^{-1}}^2=d$, where $V(\pi)=\sum_a\pi(a)aa^\top$.
    Assign $\lceil T\pi(a)/2\rceil$ pulls to each action in its support.
    These counts sum to at most $T/2+d(d+1)/2\le T$; assign any remaining pulls arbitrarily.
    The resulting exact design $N \in \alloc_T(\cA)$ satisfies $V(N)\succeq(T/2)V(\pi)$, so $\max_{a\in\cA} \norm{a}_{V(N)^{-1}}^2 \le 2d/T$.

    Conversely, for every exact design $N \in \alloc_T(\cA)$,
    \[
        \max_{a\in\cA}\norm{a}_{V(N)^{-1}}^2
        \ge \frac{1}{T} \sum_{a\in\cA}N(a)\norm{a}_{V(N)^{-1}}^2
        = \frac1T \tr(V(N)^{-1}V(N))
        = \frac{d}{T}.
    \]
    Thus $\max_{a\in\cA}\norm{a}_{V(N_{\mathrm G})^{-1}}^2 \asymp d/T$.
    Multiplying by the union-bound factor and using $\log|\cA|\asymp d$ proves the first claim.

    \medskip\noindent
    To show the second claim, we will make use of \Cref{lem:gw-confidence-tightness}, which requires us to show lower bounds on the Gaussian width $\gw(V(N)^{-1/2}\cA)$ and the variance term $\max_{a\in\cA}\norm{a}_{V(N)^{-1}}^2$.

    Fix any design $N \in \alloc_T(\cA)$ with $V(N)\succ0$.
    The preceding trace identity already gives $\max_{a\in\cA}\norm{a}_{V(N)^{-1}}\ge\sqrt{d/T}$.
    It therefore remains to show that $\gw(V(N)^{-1/2}\cA)\gtrsim d/\sqrt T$.

    For the hypercube, let $z=V(N)^{-1/2}g$, where $g\sim\cN(0,I_d)$.
    For any fixed $z$, the maximizing hypercube action chooses its signs coordinate-wise, so $\max_{a\in\cA}\ip{a,z}=\norm{z}_1$.
    Hence
    \[
        \gw(V(N)^{-1/2}\cA)
        =
        \bE\norm{z}_1
        =
        \sum_{i=1}^d \bE|z_i|.
    \]
    Every sampled action has squared $i$-th coordinate equal to one, so $V(N)_{ii}=T$ for every $i\in[d]$.
    Moreover, by Cauchy--Schwarz, $1=\ip{V(N)^{1/2}e_i,V(N)^{-1/2}e_i} \le\sqrt{V(N)_{ii}(V(N)^{-1})_{ii}}$.
    Therefore $(V(N)^{-1})_{ii}\ge1/T$.

    Since $z\sim\cN(0,V(N)^{-1})$, each coordinate satisfies $z_i\sim\cN(0,(V(N)^{-1})_{ii})$, and thus $\bE|z_i|=\sqrt{2/\pi}\sqrt{(V(N)^{-1})_{ii}}$.
    Consequently,
    \[
        \gw(V(N)^{-1/2}\cA)
        =
        \sqrt{\frac{2}{\pi}}
        \sum_{i=1}^d\sqrt{(V(N)^{-1})_{ii}}
        \ge
        \sqrt{\frac{2}{\pi}}\frac{d}{\sqrt T}.
    \]

    For a $1/2$-net of the ball, fix any $z\neq0$ and choose
    $a\in\cA$ satisfying
    $\norm{a-z/\norm{z}_2}_2\le1/2$.
    Then
    $\ip{a,z/\norm{z}_2}
    \ge1-\norm{a-z/\norm{z}_2}_2\ge1/2$, and hence
    $\max_{a\in\cA}\ip{a,z}\ge\norm{z}_2/2$.
    Therefore,
    \[
        \gw(V(N)^{-1/2}\cA)
        \ge
        \frac12\bE\norm{V(N)^{-1/2}g}_2.
    \]
    
    Let $\lambda_1,\ldots,\lambda_d$ be the eigenvalues of $V(N)$.
    Since every sampled action has norm at most one,
    $\tr(V(N))=\sum_{t=1}^T\norm{a_t}_2^2\le T$, and therefore
    $\sum_{i=1}^d\lambda_i\le T$.
    By rotational invariance of $g$, the coordinates of
    $V(N)^{-1/2}g$ in the eigenbasis of $V(N)$ have the same
    distribution as $(g_i/\sqrt{\lambda_i})_{i=1}^d$.
    Using $\norm{x}_2\ge\norm{x}_1/\sqrt d$ gives
    \begin{align*}
        \gw(V(N)^{-1/2}\cA)
        &\ge
        \frac{1}{2\sqrt d}
        \sum_{i=1}^d\frac{\bE|g_i|}{\sqrt{\lambda_i}}
        =
        \frac{1}{\sqrt{2\pi d}}
        \sum_{i=1}^d\lambda_i^{-1/2}
        \ge
        \frac{d}{\sqrt{2\pi T}}.
    \end{align*}
    For the last inequality, Jensen's inequality applied to the convex function $x\mapsto x^{-1/2}$ gives
    $\frac1d\sum_i\lambda_i^{-1/2}
    \ge(\frac1d\sum_i\lambda_i)^{-1/2}
    \ge\sqrt{d/T}$.
    
    Thus, for either action set and every fixed design $N$ with $V(N)\succ0$,
    the above established lower bounds on the Gaussian width and the variance term give
    \[
        \rad_{\sqrt{2\log(2/\delta)}}(\cA,V(N))
        \gtrsim
        \frac{d}{\sqrt T}
        +\sqrt{\frac{d\log(1/\delta)}{T}}.
    \]
    Applying \Cref{lem:gw-confidence-tightness} and using \Cref{eq:fixed-design-gaussian-error} proves the claimed lower bound on $\eps$.
\end{proof}

\section{Missing Results and Details for Section \texorpdfstring{\ref{sec:pureexp}}{Section 4}} \label{app:pureexp}

\subsection{\gwoofd Simple Regret Upper Bound}
\thmGWOOFDUpperBound*
\begin{proof}
    By \Cref{lem:gw-concentration}, with probability at least $1-\delta$,
    \begin{align*}
        \max_{a\in\cA}
        \abs{\ip{a,\hat\theta_{\hat N}-\truetheta}}
        \le
        \rad_{\sqrt{2\log(2/\delta)}}(\cA, \tilde V(\hat N)).
    \end{align*}
    Also, by the definition of $\hat A$, $\ip{\opta-\hat A,\hat\theta_{\hat N}}\le0$.
    Thus, with probability at least $1-\delta$
    \begin{align*}
        \ip{\opta - \hat A, \truetheta}
        &= \ip{\opta - \hat A, \truetheta - \hat \theta_{\hat N}} + \ip{\opta - \hat A, \hat \theta_{\hat N}} \\
        &\le 2 \max_{a \in \cA} \abs{\ip{a,\hat\theta_{\hat N}-\truetheta}} \\
        &\le 2 \rad_{\sqrt{2\log(2/\delta)}}(\cA, \tilde V(\hat N))
        \, ,
    \end{align*}
    which completes the proof by recalling the definition of $\hat N$ and $\sregret_T(\cA)$.
\end{proof}

\subsection{Relation to Approximate Designs}
\label{app:gw-fractional-designs}

Rounding approximate designs has a rich literature
\citep[Chapter~12]{pukelsheim2006optimal}; see also
\citet{allen2021near}.
Here we use a simple ceiling-rounding argument and verify its effect
directly on our Gaussian-width confidence radius, which is not one of
the standard design criteria.
More refined rounding methods may improve the required online budget,
but we do not pursue this here.

We first show that preserving the design matrix up to a constant
factor also preserves the confidence radius up to a constant factor.

\begin{lemma}[Monotonicity under PSD approximation]
    \label{lem:gw-oofd-psd-monotonicity}
    Let $M_1,M_2\succ0$ and suppose that
    $M_1\succeq\alpha M_2$ for some $\alpha\in(0,1]$.
    Then, for every nonempty compact $\cX\subseteq\bR^d$
    and $\beta\ge0$,
    \begin{align*}
        \rad_\beta(\cX,M_1)
        \le\frac{1}{\sqrt\alpha}\rad_\beta(\cX,M_2).
    \end{align*}
\end{lemma}

\begin{proof}
    Inverse monotonicity gives
    $M_1^{-1}\preceq\alpha^{-1}M_2^{-1}$, so
    \begin{align*}
        \max_{x\in\cX}\norm{x}_{M_1^{-1}}
        \le\frac{1}{\sqrt\alpha}
        \max_{x\in\cX}\norm{x}_{M_2^{-1}}.
    \end{align*}
    It remains to compare the Gaussian widths.
    For $g\sim\cN(0,I_d)$, define
    \begin{align*}
        X_x=\ip{x,M_1^{-1/2}g},
        \qquad
        Y_x=\alpha^{-1/2}\ip{x,M_2^{-1/2}g},
        \qquad x\in\cX.
    \end{align*}
    For every $x,x'\in\cX$,
    \begin{align*}
        \bE\paren*{X_x-X_{x'}}^2
        =
        \norm{x-x'}_{M_1^{-1}}^2
        \le
        \alpha^{-1}\norm{x-x'}_{M_2^{-1}}^2
        =
        \bE\paren*{Y_x-Y_{x'}}^2.
    \end{align*}
    Sudakov--Fernique (\Cref{lem:sudakov-fernique}) therefore gives
    \begin{align*}
        \gw\paren*{M_1^{-1/2}\cX}
        \le\frac{1}{\sqrt\alpha}
        \gw\paren*{M_2^{-1/2}\cX}.
    \end{align*}
    Combining the two comparisons proves the claim.
\end{proof}

\begin{corollary}[Comparison with approximate designs]
    \label{cor:gw-oofd-finite-fractional-comparison}
    Let $T\ge d(d+1)+2$ be an integer and let $\beta\ge0$.
    Then
    \begin{align*}
        \inf_{\design\in\cM_1(\cA)}
        \rad_\beta\paren*{\cA,\Voff+TV(\design)}
        \le
        \min_{N\in\alloc_T(\cA)}
        \rad_\beta\paren*{\cA,\tilde V(N)}
        \le
        \sqrt2
        \inf_{\design\in\cM_1(\cA)}
        \rad_\beta\paren*{\cA,\Voff+TV(\design)},
    \end{align*}
    where $V(\design)=\bE_{a\sim\design}[aa^\top]$, for an approximate design $\design\in\cM_1(\cA)$. 
\end{corollary}

\begin{proof}
    Every exact design $N\in\alloc_T(\cA)$ defines an approximate
    design with weights $\design_N(a)=N(a)/T$.
    Since $TV(\design_N)=V(N)$, the approximate-design problem
    includes every exact design.
    This proves the left inequality. 

    For the right inequality, fix any $\design\in\cM_1(\cA)$.
    We first replace it by a design with finitely many actions,
    without changing its design matrix.
    The space of symmetric $d\times d$ matrices has dimension
    $d(d+1)/2$.
    Since $\cA$ is compact and
    \begin{align*}
        V(\design)=\bE_{a\sim\design}[aa^\top],
    \end{align*}
    this matrix lies in the convex hull of
    $\{aa^\top:a\in\cA\}$.
    Carath\'eodory's theorem therefore gives a design
    $\design'\in\cM_1(\cA)$ satisfying
    \begin{align*}
        V(\design')=V(\design),
        \qquad
        \abs{\supp(\design')}\le\frac{d(d+1)}{2}+1.
    \end{align*}

    We now round the weights of $\design'$ to integer pull counts.
    For each $a\in\supp(\design')$, set
    \begin{align*}
        N_0(a)=\left\lceil\frac{T}{2}\design'(a)\right\rceil.
    \end{align*}
    Rounding up adds at most one pull per action, so
    \begin{align*}
        \sum_{a\in\supp(\design')}N_0(a)
        &\le
        \frac{T}{2}\sum_{a\in\supp(\design')}\design'(a)
        +\frac{d(d+1)}{2}+1\\
        &=
        \frac{T}{2}+\frac{d(d+1)}{2}+1
        \le T.
    \end{align*}
    Assign any unused pulls to an arbitrary action in $\cA$,
    giving an exact design $N\in\alloc_T(\cA)$.
    Since every count was rounded up,
    \begin{align*}
        V(N)
        \succeq V(N_0)
        \succeq\frac{T}{2}V(\design')
        =\frac{T}{2}V(\design).
    \end{align*}
    Consequently,
    \begin{align*}
        \tilde V(N)
        =\Voff+V(N)
        \succeq\frac12\paren*{\Voff+TV(\design)}.
    \end{align*}
    Applying \Cref{lem:gw-oofd-psd-monotonicity} with $\alpha=1/2$
    gives
    \begin{align*}
        \min_{\widetilde N\in\alloc_T(\cA)}
        \rad_\beta\paren*{\cA,\tilde V(\widetilde N)}
        \le
        \rad_\beta\paren*{\cA,\tilde V(N)}
        \le
        \sqrt2\,
        \rad_\beta\paren*{\cA,\Voff+TV(\design)}.
    \end{align*}
    Taking the infimum over $\design$ proves the right inequality.
\end{proof}

\subsection{Simple Regret Lower Bounds for Fixed-Design Algorithms}
The following is a fixed-design lower bound corresponding to the Gaussian-width term in \Cref{thm:gw-oofd-upper-bound}.
This is the offline-to-online analogue of Theorem 3 in \citep[Section 2.2]{pmlr-v336-maiti26a}.
The proof is essentially unchanged: because the online design is chosen without inspecting the offline rewards and does not adapt to the online rewards, the offline and online actions together form one deterministic fixed design.
The offline data simply adds \(\Voff\) to the design matrix in their Gaussian-prior argument.

\begin{theorem}[Gaussian-width lower bound]
    \label{thm:pureexp-fixed-design-lower-bound}
    Let $d\ge1$, $T\in\bN_0$, and $\delta\in(0,1/10]$.
    Suppose a fixed-design algorithm selects $\hat N$ and has simple regret at most the same finite $\eps>0$ with probability at least $1-\delta$ for every $\truetheta\in\bR^d$.
    Then
    \begin{align*}
        \eps\ge0.09\,\gw\paren{\tilde V(\hat N)^{-1/2}\cA}.
    \end{align*}
\end{theorem}
\begin{proof}
    We apply the design-dependent argument used to prove
    \citet[Theorem~3]{pmlr-v336-maiti26a}.
    The offline actions are fixed, and the online design $\hat N$
    does not depend on any rewards.
    Thus the offline and online observations together form one
    deterministic fixed design with $\Toff+T$ observations and
    design matrix
    \begin{align*}
        \Voff+V(\hat N)=\tilde V(\hat N).
    \end{align*}
    This matrix is invertible by our assumption on $\Voff$.

    Set $\tau=\sqrt2-1$ and consider the prior
    \begin{align*}
        \theta\sim\cN\bigl(0,\tau^2\tilde V(\hat N)^{-1}\bigr),
    \end{align*}
    independently of the reward noise.
    The argument in their Section~2.2 gives
    \begin{align}
        \bE\sregret_T(\cA)
        &\ge
        \frac{\tau(1-\tau)}{1+\tau^2}
        \gw\paren*{\tilde V(\hat N)^{-1/2}\cA}
        \nonumber\\
        &=
        \frac{\tau}{2}
        \gw\paren*{\tilde V(\hat N)^{-1/2}\cA}.
        \label{eq:fixed-design-bayes-regret-lower}
    \end{align}
    Here expectations include both the prior and the observations.

    As in their Appendix~B.3, we upper bound the same expectation
    using the assumed guarantee.
    For each fixed $\theta$, simple regret is at most $\eps$
    except on an event of probability at most $\delta$, and is
    always bounded by the range of expected rewards.
    Therefore
    \begin{align*}
        \bE\sregret_T(\cA)
        &\le
        \eps+\delta\,\bE\left[
            \max_{a\in\cA}\ip{a,\theta}
            -\min_{a\in\cA}\ip{a,\theta}
        \right]\\
        &=
        \eps+2\tau\delta\,
        \gw\paren*{\tilde V(\hat N)^{-1/2}\cA},
    \end{align*}
    where the equality follows from symmetry of the Gaussian prior
    and the definition of Gaussian width.

    Combining the two bounds and using $\delta\le1/10$ gives
    \begin{align*}
        \eps
        &\ge
        \tau(1/2-2\delta)
        \gw\paren*{\tilde V(\hat N)^{-1/2}\cA}\\
        &\ge
        \frac{3(\sqrt2-1)}{10}
        \gw\paren*{\tilde V(\hat N)^{-1/2}\cA}\\
        &\ge
        0.09\,\gw\paren*{\tilde V(\hat N)^{-1/2}\cA}.
    \end{align*}
\end{proof}

The preceding theorem gives the Gaussian-width term.
Next, we will show a lower bound on the high probability term for any algorithm (possibly adaptive).
The proof uses a two-point argument by comparing a hard alternative with the null environment $\theta=0$.
A two-point argument naturally produces the largest distance between two actions in the metric induced by the design matrix, whereas $\rad_\beta$ is expressed through single actions.
The following lemma shows that the two quantities agree up to a universal constant whenever $d\ge2$ and the design matrix is generated by actions in $\cA$.

\begin{lemma}[Pairs dominate single actions]
    \label{lem:pureexp-pairs-dominate-single-actions}
    Let $d\ge2$ and let $A_1,\ldots,A_n$ be possibly random actions in $\cA$.
    Suppose that
    \begin{align*}
        M
        \defeq
        \sum_{i=1}^n \bE[A_iA_i^\top]
    \end{align*}
    is positive definite.
    Then
    \begin{align*}
        \max_{a\in\cA}\norm{a}_{M^{-1}}
        \le
        \paren*{1+\frac{1}{\sqrt2}}
        \max_{a,b\in\cA}\norm{a-b}_{M^{-1}}.
    \end{align*}
\end{lemma}

\begin{proof}
    Let $\bar a=\frac1n\sum_{i=1}^n\bE[A_i]$.
    Since
    \begin{align*}
        \frac1n\sum_{i=1}^n
        \bE[(A_i-\bar a)(A_i-\bar a)^\top]
        =
        \frac{M}{n}-\bar a\bar a^\top
        \succeq0,
    \end{align*}
    we have $\bar a\bar a^\top\preceq M/n$, and therefore
    \begin{align}
        \norm{\bar a}_{M^{-1}}^2
        \le
        \frac1n.
        \label{eq:pureexp-mean-action-norm}
    \end{align}

    Now fix any $a\in\cA$.
    Since $a-\bar a=\frac1n\sum_{i=1}^n\bE[a-A_i]$, Jensen's inequality gives
    \begin{align*}
        \norm{a-\bar a}_{M^{-1}}
        \le
        \max_{a',b'\in\cA}\norm{a'-b'}_{M^{-1}}.
    \end{align*}
    The triangle inequality and \Cref{eq:pureexp-mean-action-norm} therefore imply
    \begin{align}
        \max_{a\in\cA}\norm{a}_{M^{-1}}
        \le
        \max_{a,b\in\cA}\norm{a-b}_{M^{-1}}
        +
        \frac1{\sqrt n}.
        \label{eq:pureexp-single-action-via-pairs}
    \end{align}

    It remains to control $1/\sqrt n$ by the pairwise distances.
    Since the actions $A_1,\ldots,A_n$ may have different distributions, define a single random action $A$ by first choosing $i$ uniformly from $[n]$ and then drawing from the marginal distribution of $A_i$.
    This simply averages the $n$ action distributions, so $\bE[A]=\bar a$ and $\bE[AA^\top]=M/n$.
    Let $A'$ be an independent copy of $A$.
    Since $A,A'\in\cA$, their distance is always at most the largest pairwise distance in $\cA$, and hence
    \begin{align*}
        \max_{a,b\in\cA}\norm{a-b}_{M^{-1}}^2 &\ge \bE\norm{A-A'}_{M^{-1}}^2 = 2\bE\norm{A-\bar a}_{M^{-1}}^2 = 2\left(\tr(M^{-1}\bE[AA^\top])-\norm{\bar a}_{M^{-1}}^2\right) = 2\left(\frac dn-\norm{\bar a}_{M^{-1}}^2\right) \\
        &\ge \frac{2(d-1)}{n}\,.
    \end{align*}
    The first equality uses that $A$ and $A'$ are independent copies with mean $\bar a$, so the cross term vanishes.
    The second equality uses $\bE[AA^\top]=M/n$, and the last inequality uses \Cref{eq:pureexp-mean-action-norm}.    
    Since $d\ge2$, this gives
    \begin{align*}
        \frac1{\sqrt n}
        \le
        \frac1{\sqrt2}
        \max_{a,b\in\cA}\norm{a-b}_{M^{-1}}.
    \end{align*}
    Substituting this into \Cref{eq:pureexp-single-action-via-pairs} proves the result.
\end{proof}

If $d=1$ and $\cA=\{1\}$, then the right-hand side of the lemma vanishes while the left-hand side equals $1/\sqrt n$.
Correspondingly, the procedure that always returns the unique action has zero simple regret, so no lower bound of the form in \Cref{thm:pureexp-high-prob-lower-bound} can hold: in dimension one the term $\max_{a\in\cA}\norm{a}_{M^{-1}}$ in $\rad_\beta$ can overstate the difficulty of the problem when $\cA$ is far from the origin relative to its diameter.
The reason we choose to present things in terms of single actions instead of differences is to make the presentation in the body easier to follow.

The next theorem gives a high-probability lower bound for any algorithm, possibly adaptive.
Combining this and the previous result (\Cref{thm:pureexp-fixed-design-lower-bound}) will give us our lower bound for all fixed design algorithms (\Cref{thm:fixed-design-lower-bound}). We will also use the next result to derive our simple regret lower bound for adaptive algorithms (\Cref{thm:hp-sr-lb}).

\begin{theorem}[High-probability lower bound]
    \label{thm:pureexp-high-prob-lower-bound}
    Let $d\ge2$, $T\in\bN_0$, and $\delta\in(0,1/3]$.
    Consider any algorithm (possibly adaptive) that outputs an action $\hat A\in\cA$.
    Let $\bP_\theta$ be the distribution of the full history and output action under parameter $\theta$.
    Suppose that, for some $\eps>0$
    \begin{align*}
        \bP_\theta\paren*{\max_{a\in\cA}\ip{a-\hat A,\theta}\le\eps}\ge1-\delta
        \qquad\text{for every }\theta\in\bR^d.
    \end{align*}
    Then
    \begin{align}
        \eps \gea \sqrt{\log(1/\delta)}\max_{a\in\cA}\norm{a}_{\tilde V_0^{-1}},
        \label{eq:pureexp-high-prob-selected-policy-lower}
    \end{align}
    where $\tilde V_0\defeq\Voff+\bE_0[\sum_{t=1}^T A_tA_t^\top]$ is the expected design matrix under $\theta=0$.
\end{theorem}

\begin{proof}
    Once the algorithm and offline actions are fixed, $\tilde V_0$ is a deterministic matrix, even though the online actions are adaptive.
    Moreover, $\tilde V_0\succeq\Voff\succ0$.

    We first prove the pairwise bound
    \begin{align}
        \eps \ge \frac14\max_{a,b\in\cA}\norm{a-b}_{\tilde V_0^{-1}}\sqrt{\log\paren*{\frac{1}{4\delta(1-\delta)}}},
        \label{eq:pureexp-high-prob-pairwise-lower}
    \end{align}
    and then convert it into a bound on single actions.

    By compactness, choose $a_0,b_0\in\cA$ maximizing $\norm{a-b}_{\tilde V_0^{-1}}$.
    This maximum is positive because $\Voff\succ0$ implies that the offline actions, and hence $\cA$, span $\bR^d$.
    Set
    \begin{align*}
        u = \frac{\tilde V_0^{-1}(a_0-b_0)}{\norm{a_0-b_0}_{\tilde V_0^{-1}}}.
    \end{align*}
    Then
    \begin{align}
        u^\top\tilde V_0u=1,
        \qquad
        \ip{a_0-b_0,u}=\max_{a,b\in\cA}\norm{a-b}_{\tilde V_0^{-1}}.
        \label{eq:pureexp-hard-direction-normalization}
    \end{align}

    Write $p_{\max}=\max_{a\in\cA}\ip{a,u}$, $p_{\min}=\min_{a\in\cA}\ip{a,u}$, and $w=p_{\max}-p_{\min}$.
    By \Cref{eq:pureexp-hard-direction-normalization},
    \begin{align*}
        w\ge\max_{a,b\in\cA}\norm{a-b}_{\tilde V_0^{-1}}>0.
    \end{align*}

    Consider the event
    \begin{align*}
        E_u=\braces*{\ip{\hat A,u}>\frac{p_{\max}+p_{\min}}{2}}.
    \end{align*}
    Replacing $u$ by $-u$ gives the disjoint event in which $\ip{\hat A,u}$ is strictly below the same midpoint.
    Hence one of the two orientations has probability at most $1/2$ under the null.
    Choose the orientation of $u$ so that
    \begin{align}
        \bP_0(E_u)\le\frac12.
        \label{eq:pureexp-null-event-half}
    \end{align}

    Now choose the alternative parameter
    \begin{align*}
        \theta_+=\frac{4\eps}{w}u.
    \end{align*}
    Under this parameter, an $\eps$-optimal recommendation satisfies
    \begin{align*}
        \frac{4\eps}{w}\paren*{p_{\max}-\ip{\hat A,u}}\le\eps.
    \end{align*}
    Rearranging gives
    \begin{align*}
        \ip{\hat A,u}\ge p_{\max}-\frac{w}{4}=\frac{p_{\max}+p_{\min}}{2}+\frac{w}{4}.
    \end{align*}
    Hence every $\eps$-optimal recommendation belongs to $E_u$.
    The assumed guarantee therefore gives
    \begin{align}
        \bP_{\theta_+}(E_u)\ge1-\delta.
        \label{eq:pureexp-alternative-event-probability}
    \end{align}

    The event $E_u$ has probability at most $1/2$ under the null, but probability at least $1-\delta$ under $\theta_+$.
    Hence the two history distributions must be well separated.
    Mapping the history and output action to the indicator $\I{E_u}$ and applying data processing gives
    \begin{align*}
        D_{\mathrm{KL}}(\bP_0\Vert\bP_{\theta_+})\ge\operatorname{kl}\paren*{\bP_0(E_u),\bP_{\theta_+}(E_u)}.
    \end{align*}
    Since $\bP_0(E_u)\le1/2$ and $\bP_{\theta_+}(E_u)\ge1-\delta>1/2$, the binary KL divergence is minimized when these probabilities are $1/2$ and $1-\delta$, respectively.
    Therefore,
    \begin{align}
        D_{\mathrm{KL}}(\bP_0\Vert\bP_{\theta_+})\ge\operatorname{kl}\paren*{\frac12,1-\delta}=\frac12\log\paren*{\frac{1}{4\delta(1-\delta)}}.
        \label{eq:pureexp-binary-kl-lower}
    \end{align}

    The adaptive divergence decomposition \citep[Lemma~15.1 and Exercise~15.8]{lattimore2020bandit} gives
    \begin{align}
        D_{\mathrm{KL}}(\bP_0\Vert\bP_{\theta_+})
        &=
        \frac12\sum_{s=1}^{\Toff}\ip{\aoff_s,\theta_+}^2
        +\frac12\bE_0\brackets*{\sum_{t=1}^T\ip{A_t,\theta_+}^2}
        =
        \frac12\theta_+^\top\tilde V_0\theta_+
        =
        \frac{8\eps^2}{w^2}.
        \label{eq:pureexp-gaussian-transcript-kl}
    \end{align}
    The action-selection rule is the same under both parameters conditional on the past, so only the Gaussian reward distributions contribute to the KL divergence.
    The expectation over the online actions is under $\bP_0$ because $\bP_0$ is the first argument of the KL divergence.
    This is why we choose the hard direction using the design matrix under the null rather than under $\theta_+$.

    Combining \Cref{eq:pureexp-binary-kl-lower,eq:pureexp-gaussian-transcript-kl} gives
    \begin{align*}
        \eps\ge\frac{w}{4}\sqrt{\log\paren*{\frac{1}{4\delta(1-\delta)}}}.
    \end{align*}
    Since $w\ge\max_{a,b\in\cA}\norm{a-b}_{\tilde V_0^{-1}}$, this proves \Cref{eq:pureexp-high-prob-pairwise-lower}.

    It remains to pass from pairs to single actions.
    Apply \Cref{lem:pureexp-pairs-dominate-single-actions} to the $\Toff$ fixed offline actions and the $T$ online actions under $\theta=0$.
    Their expected design matrix is
    \begin{align*}
        \sum_{s=1}^{\Toff}\aoff_s(\aoff_s)^\top+\sum_{t=1}^T\bE_0[A_tA_t^\top]=\Voff+\bE_0\brackets*{\sum_{t=1}^T A_tA_t^\top}=\tilde V_0.
    \end{align*}
    Therefore,
    \begin{align*}
        \max_{a,b\in\cA}\norm{a-b}_{\tilde V_0^{-1}}\ge(2-\sqrt2)\max_{a\in\cA}\norm{a}_{\tilde V_0^{-1}}.
    \end{align*}
    Substituting this into \Cref{eq:pureexp-high-prob-pairwise-lower} gives
    \begin{align*}
        \eps\ge\frac{2-\sqrt2}{4}\sqrt{\log\paren*{\frac{1}{4\delta(1-\delta)}}}\max_{a\in\cA}\norm{a}_{\tilde V_0^{-1}}.
    \end{align*}
    Finally, for every $\delta\in(0,1/3]$,
    \begin{align*}
        \log\paren*{\frac{1}{4\delta(1-\delta)}}\ge\frac{1}{10}\log(1/\delta).
    \end{align*}
    This proves the result.
\end{proof}

\thmFixedDesignLowerBound*
\begin{proof}
    Let $\hat N$ be the selected fixed design.
    Since the online actions prescribed by $\hat N$ do not depend on the observed rewards, the expected design matrix under $\theta=0$ from \Cref{thm:pureexp-high-prob-lower-bound} is simply
    \begin{align*}
        \tilde V_0
        =
        \Voff+\bE_0\brackets*{\sum_{t=1}^T A_tA_t^\top}
        =
        \Voff+V(\hat N)
        =
        \tilde V(\hat N).
    \end{align*}
    Therefore \Cref{thm:pureexp-fixed-design-lower-bound} gives
    \begin{align*}
        \eps
        \gea
        \gw\paren{\tilde V(\hat N)^{-1/2}\cA},
    \end{align*}
    while \Cref{thm:pureexp-high-prob-lower-bound} gives
    \begin{align*}
        \eps
        \gea
        \sqrt{\log(1/\delta)}
        \max_{a\in\cA}\norm{a}_{\tilde V(\hat N)^{-1}}.
    \end{align*}
    Combining these two bounds using $\max\{x,y\}\ge(x+y)/2$ gives
    \begin{align*}
        \eps
        \gea
        \gw\paren{\tilde V(\hat N)^{-1/2}\cA}
        +
        \sqrt{\log(1/\delta)}
        \max_{a\in\cA}\norm{a}_{\tilde V(\hat N)^{-1}}
        =
        \rad_{\sqrt{\log(1/\delta)}}\paren{\cA,\tilde V(\hat N)}.
    \end{align*}
    Since $\hat N\in\alloc_T(\cA)$, taking the minimum over all fixed designs gives the result.
\end{proof}

\subsection{Simple Regret Lower Bounds for Adaptive Algorithms}
Before proving our main lower bounds, we derive some auxiliary results (\Cref{lem:hp-norm,lem:mi,lem:cap}).

\begin{lemma}[High-probability bound on norm]\label{lem:hp-norm}
    Suppose $X \sim \cN(0, \Sigma)$ where $\Sigma = \diag(\sigma_1^2, \dots, \sigma_d^2)$ with $\Tr\Sigma > 0$ and $\sigma_1^2 \ge \dots \ge \sigma_d^2$. Let $\delta \in (0, 1)$. Then with probability at least $1-\delta$,
    \[
        \norm{X}_2 \ge \frac{\delta}{\sqrt{e}}\sqrt{\Tr\Sigma}\,.
    \]
\end{lemma}
\begin{proof}
    First note that $X$ has the same distribution as $\Sigma^{1/2}g$ for $g \sim \cN(0, I_d)$, and $\norm{\Sigma^{1/2}g}_2^2 = \sum_{i=1}^d \sigma_i^2 g_i^2$. So our goal is to bound $\Prob(\norm{\Sigma^{1/2}g}_2 \le \varepsilon)$. We use the Cram\'er--Chernoff method. For any $\varepsilon > 0$ and any $\alpha \ge 0$,
    \begin{align*}
        \Prob\left(\norm{\Sigma^{1/2}g}_2 \le \varepsilon\right)
        = \Prob\left(\norm{\Sigma^{1/2}g}_2^2 \le \varepsilon^2\right)
        \le e^{\alpha\varepsilon^2}\, \EE{e^{-\alpha \sum_{i=1}^d \sigma_i^2 g_i^2}}\,,
    \end{align*}
    where the last step follows from exponentiating (the map $s \mapsto e^{-\alpha s}$ is nonincreasing since $\alpha \ge 0$) and applying Markov's inequality. Using independence of $(g_i)_{i=1}^d$ and the Gaussian integral
    \[
        \int_{-\infty}^\infty e^{-\alpha \sigma_i^2 x^2}\, \frac{e^{-x^2/2}}{\sqrt{2\pi}}\, \mathrm{d}x = \left(1 + 2\alpha\sigma_i^2\right)^{-1/2},
    \]
    we get
    \begin{align*}
        \Prob\left(\norm{\Sigma^{1/2}g}_2 \le \varepsilon\right) \le e^{\alpha \varepsilon^2} \prod_{i=1}^d \left(1+2\alpha \sigma_i^2\right)^{-1/2}\,.
    \end{align*}
    Since $\alpha \ge 0$, expanding the product and dropping the (nonnegative) cross-terms gives $\prod_{i=1}^d (1+2\alpha \sigma_i^2) \ge 1+2\alpha\sum_{i=1}^d \sigma_i^2 = 1 + 2\alpha\Tr\Sigma$. Hence
    \begin{align*}
        \Prob\left(\norm{\Sigma^{1/2}g}_2 \le \varepsilon\right)
        \le e^{\alpha\varepsilon^2}\left(1 + 2\alpha\Tr\Sigma\right)^{-1/2}
        = \exp\left(\alpha\varepsilon^2 - \frac12\log\left(1+2\alpha\Tr\Sigma\right)\right).
    \end{align*}
    Assume now that $\varepsilon^2 < \Tr\Sigma$. The exponent is convex in $\alpha$ and is minimized over $\alpha \ge 0$ at the interior point $\alpha^\ast = \frac{1}{2\varepsilon^2} - \frac{1}{2\Tr\Sigma} > 0$, which gives
    \begin{align*}
        \Prob\left(\norm{\Sigma^{1/2}g}_2 \le \varepsilon\right)
        \le \exp\left(-\frac{\varepsilon^2}{2\Tr\Sigma}\right) \cdot \frac{\varepsilon\sqrt{e}}{\sqrt{\Tr\Sigma}}
        < \frac{\varepsilon\sqrt{e}}{\sqrt{\Tr\Sigma}}\,,
    \end{align*}
    where the last inequality holds because $\varepsilon > 0$. Now set $\varepsilon \coloneqq (\delta/\sqrt{e})\sqrt{\Tr\Sigma}$, so that $\delta = \varepsilon\sqrt{e}/\sqrt{\Tr\Sigma}$. Since $\delta < 1$, we have $\varepsilon^2 = \delta^2\Tr\Sigma/e < \Tr\Sigma$, so the assumption above holds. Therefore $\Prob(\norm{X}_2 \le \varepsilon) < \delta$, that is, $\norm{X}_2 > (\delta/\sqrt{e})\sqrt{\Tr\Sigma}$ with probability greater than $1-\delta$, finishing the proof.
\end{proof}

\begin{lemma}[Mutual information bound]\label{lem:mi}
Let $\theta \sim \mu = \cN(0,\Sigma)$ with $\Sigma \succeq 0$, and let $\Doff$ be offline data with deterministic design matrix $\Voff$, independent of the
online noise given $\theta$. 
Let $\cH_T = (A_1, R_1,\dots,A_T,R_T)$ be the history of a
(possibly adaptive) algorithm and
\[
  R_t = \ip{A_t,\theta} + \eta_t, \qquad \eta_t \sim \cN(0,1).
\]
Define $V_T = \sum_{t=1}^T A_tA_t^\top$ and $\bar V_T = \bE_\mu[V_T]$. Then
\begin{equation}\label{eq:mi-logdet}
  I_\mu\big(\theta; (\cH_T,\Doff)\big)
  \;\le\; \frac12 \log\det\!\Big(I_d + \Sigma^{1/2}\big(\Voff + \bar V_T\big)\Sigma^{1/2}\Big)\,.
\end{equation}
If $\Sigma = \diag(\sigma_1^2,\dots,\sigma_d^2)$ in the eigenbasis of $\Voff$ with
$\nu_i = \lambda_i(\Voff)$ and $n_i = (\bar V_T)_{ii}$, then
\begin{equation}\label{eq:mi-hadamard}
  I_\mu\big(\theta; (\cH_T,\Doff)\big)
  \;\le\; \frac12 \sum_{i=1}^d \log\big(1 + \sigma_i^2(\nu_i + n_i)\big).
\end{equation}
Moreover $\sum_{i=1}^d n_i = \bE\big[\Tr V_T\big] = \bE\big[\sum_{t\le T}\norm{A_t}_2^2\big] \le T$
whenever $\cA \subseteq B_2^d$.
\end{lemma}

\begin{proof}[Proof of \Cref{lem:mi}]
Assume first $\Sigma \succ 0$, the general case is treated at the end. Let
$\cF := (\cH_T, \Doff)$ and $V_T := \sum_{t\le T}A_tA_t^\top$.

\medskip
\noindent\textbf{Step 1 (Posterior given the realized history).}
Conditioning on a realization of $\cF$, the joint density of
$\cF$ given $\theta$ factorizes as
\[
  p(\cF \mid \theta)
  = \prod_{s \le \Toff} \phi\big(\roff_s - \ip{\aoff_s,\theta}\big)
    \prod_{t \le T} \pi_t\big(A_t \mid \cH_{t-1},\Doff \big)\,
                    \phi\big(R_t - \ip{A_t,\theta}\big),
\]
where $\phi$ is the standard normal density and $\pi_t$ is the algorithm's action-selection
kernel. Each $\pi_t$ depends on $\theta$ only through the conditioning variables
$(\cH_{t-1},\Doff)$, which are held fixed; hence the $\pi_t$ factors are constants in $\theta$
and cancel from Bayes' rule. Consequently
\[
  p(\theta \mid \cF) \;\propto\;
  \exp\Big(-\tfrac12\theta^\top\Sigma^{-1}\theta
    -\tfrac12\sum_{s\le\Toff}\big(\roff_s -\ip{\aoff_s,\theta}\big)^2
    -\tfrac12\sum_{t\le T}\big(R_t-\ip{A_t,\theta}\big)^2\Big),
\]
a Gaussian in $\theta$ with covariance
\begin{equation}\label{eq:posterior-cov}
  \Sigma_{\mathrm{post}}(\cF) \;=\; \big(\Sigma^{-1} + \Voff + V_T\big)^{-1},
\end{equation}
which depends on $\cF$ only through the realized design $\Voff + V_T$.
No conditional-independence assumption between the offline and online blocks is used: the
adaptivity of $A_t$ on $\Doff$ is entirely absorbed into the canceling $\pi_t$ factors.

\medskip
\noindent\textbf{Step 2 (Exact mutual information).}
Since $\theta \mid \cF$ is Gaussian with covariance
\eqref{eq:posterior-cov},
\[
  h(\theta \mid \cF) = \bE\Big[\tfrac12\log\big((2\pi e)^d\det\Sigma_{\mathrm{post}}(\cF)\big)\Big],
  \qquad
  h(\theta) = \tfrac12\log\big((2\pi e)^d\det\Sigma\big),
\]
so that
\begin{equation}\label{eq:mi-exact}
  I_\mu(\theta;\cF) = h(\theta) - h(\theta\mid\cF)
  = \frac12\,\bE\Big[\log\det\big(I_d + \Sigma^{1/2}(\Voff + V_T)\Sigma^{1/2}\big)\Big],
\end{equation}
using $\det\Sigma/\det\Sigma_{\mathrm{post}} = \det(I_d + \Sigma^{1/2}(\Voff+V_T)\Sigma^{1/2})$.

\medskip
\noindent\textbf{Step 3 (Jensen and Hadamard).}
The map $M \mapsto \log\det(I_d + \Sigma^{1/2}M\Sigma^{1/2})$ is concave on the PSD cone, so
Jensen's inequality applied to the random matrix $V_T$ in \eqref{eq:mi-exact} yields
\eqref{eq:mi-logdet}. The matrix
$M := I_d + \Sigma^{1/2}(\Voff + \bar V_T)\Sigma^{1/2}$ is positive definite, so
$\det M \le \prod_i M_{ii}$ by Hadamard's inequality. In the eigenbasis of $\Voff$, in which
$\Sigma$ is diagonal, $M_{ii} = 1 + \sigma_i^2(\nu_i + n_i)$, giving \eqref{eq:mi-hadamard}.
The budget constraint is $\sum_i n_i = \Tr\bar V_T = \bE[\sum_{t\le T}\norm{A_t}_2^2] \le T$.

\medskip
\noindent\textbf{Step 4 (Singular $\Sigma$).}
If $\Sigma$ has rank $k < d$, let $P \in \R^{d\times k}$ have orthonormal columns spanning
$\textrm{range}(\Sigma)$ and write $\theta = P\vartheta$ with
$\vartheta \sim \cN(0,\Sigma_k)$, $\Sigma_k := P^\top\Sigma P \succ 0$. Every observation
satisfies $\ip{X,\theta} = \ip{P^\top X,\vartheta}$, so the model is the nondegenerate model of
Steps 1--3 in dimension $k$ with designs $P^\top\Voff P$ and $P^\top V_T P$, and
$I(\theta;\cF) = I(\vartheta;\cF)$. Applying \eqref{eq:mi-exact} there and using
$\det(I_k + \Sigma_k^{1/2}P^\top M P\,\Sigma_k^{1/2})
= \det(I_d + \Sigma^{1/2}M\Sigma^{1/2})$ for $M \succeq 0$ recovers
\eqref{eq:mi-logdet}--\eqref{eq:mi-hadamard}, with the directions in $\ker\Sigma$
contributing factors of $1$.
\end{proof}

\begin{lemma}[Cap probability]\label{lem:cap}
Let $\mu = \cN(0,\Sigma)$ with $\Sigma = \diag(\sigma_1^2,\dots,\sigma_d^2)$ and
$\sigma_1^2 = \max_i \sigma_i^2 > 0$, let $\cA = B_2^d$, and let
$\tilde L(\theta,a) = 1 - \ip{\theta/\norm{\theta}_2, a}$. Fix $\Delta \in (0,1)$ and set
\[
  c_\Delta = \frac{(1-\Delta)^2}{2\Delta - \Delta^2},
  \qquad
  d_\Delta = \frac{\sum_{i=2}^d \log\left(1 + c_\Delta\frac{\sigma_i^2}{\sigma_1^2}\right)}{\log\big(1 + c_\Delta\big)} .
\]
Then
\begin{equation}\label{eq:cap}
  \rho_\mu(\Delta) \;:=\; \sup_{a \in \cA} \Prob_\mu\big(\tilde L(\theta,a) < \Delta\big)
  \;\le\; \frac12\big(1+c_\Delta\big)^{-d_\Delta/2}.
\end{equation}
\end{lemma}

\begin{proof}[Proof of \Cref{lem:cap}]
Let $\beta := (1-\Delta)^2$, so $\beta \in (0,1)$, $1-\beta = 2\Delta-\Delta^2$ and
$c := c_\Delta = \beta/(1-\beta)$. Also let $\lambda_i := \sigma_i^2/\sigma_1^2 \in [0,1]$ and $\theta = \Sigma^{1/2}Z$ with $Z \sim \cN(0,I_d)$. Since $\sigma_1^2 > 0$ we have
$\theta \neq 0$ almost surely. For $a \in \R^d$ let
\[
  M_a \;:=\; \Sigma^{1/2}\big(aa^\top - \beta I\big)\Sigma^{1/2}
      \;=\; -\beta\Sigma + vv^\top, \qquad v := \Sigma^{1/2}a ,
\]
and let $\mu_1(M) \ge \dots \ge \mu_d(M)$ denote the eigenvalues of a symmetric matrix $M$ in decreasing order.

\medskip
\noindent\textbf{Step 1 (An expression for the cap probability).}
We show that
\begin{equation}\label{eq:half-identity}
  \Prob_\mu\big(\tilde L(\theta,a) < \Delta\big)
  \;=\; \tfrac12\,\Prob\big(Z^\top M_a Z > 0\big)
  \qquad\text{for every unit vector } a,
\end{equation}
and that the supremum defining $\rho_\mu(\Delta)$ may be restricted to unit vectors.

For the restriction, let $a \in \cA$ with $a \neq 0$ and $\hat a := a/\norm{a}_2$. On the
event $\set{\ip{a,\theta} > (1-\Delta)\norm{\theta}_2}$ the left side is positive, hence so is
$\ip{\hat a,\theta}$, and therefore
$\ip{\hat a,\theta} \ge \norm{a}_2\ip{\hat a,\theta} = \ip{a,\theta} > (1-\Delta)\norm{\theta}_2$.
Thus the event for $a$ is contained in the event for $\hat a$ and for $a = 0$ it is empty.

For \eqref{eq:half-identity}, fix a unit vector $a$. Since $\norm{\theta}_2 > 0$ a.s.\ and
$1-\Delta > 0$, both sides of $\ip{a,\theta} > (1-\Delta)\norm{\theta}_2$ are then positive on
the event, so squaring is an equivalence and
\[
  \set{\tilde L(\theta,a) < \Delta}
  \;=\; \underbrace{\set{\ip{v, Z} > 0}}_{=: H^+}
        \cap \underbrace{\set{Z^\top M_a Z > 0}}_{=: G_a},
\]
using $\ip{a,\theta} = \ip{v,Z}$ and
$\ip{a,\theta}^2 - \beta\norm{\theta}_2^2 = \theta^\top(aa^\top - \beta I)\theta = Z^\top M_aZ$.
If $v = 0$ then $M_a = -\beta\Sigma \preceq 0$, so $\Prob(G_a) = 0$ and both sides of
\eqref{eq:half-identity} vanish. If $v \neq 0$, the set $G_a$ is invariant under $Z \mapsto -Z$,
which preserves the law of $Z$ and maps $H^+$ onto $H^- := \set{\ip{v,Z} < 0}$, while
$\set{\ip{v,Z} = 0}$ is null, hence $\Prob(G_a \cap H^+) = \Prob(G_a \cap H^-)$ and these two
sets partition $G_a$ up to a null set, giving \eqref{eq:half-identity}.

\medskip
\noindent\textbf{Step 2 (The supremum is attained at $a = e_1$).}
We claim that for every unit vector $a$,
\begin{equation}\label{eq:eig-dom}
  \mu_j(M_a) \;\le\; \mu_j(M_{e_1}) \qquad \text{for all } j \in [d].
\end{equation}

\emph{The case $j = 1$.} As $\norm{a}_2 = 1$ we have $aa^\top \preceq I$, hence
$aa^\top - \beta I \preceq (1-\beta)I$; congruence by $\Sigma^{1/2}$ preserves the Loewner order,
so $M_a \preceq (1-\beta)\Sigma$ and therefore
$\mu_1(M_a) \le (1-\beta)\sigma_1^2$.

\emph{The case $j \ge 2$.} Since $vv^\top$ is positive semidefinite of rank at most one, Weyl's
inequality $\mu_{i+j-1}(A+B) \le \mu_i(A) + \mu_j(B)$ applied with $A = -\beta\Sigma$,
$B = vv^\top$ and $j = 2$, so that $\mu_2(B) = 0$, gives
\[
  \mu_j(M_a) \;\le\; \mu_{j-1}\big(-\beta\Sigma\big), \qquad j = 2,\dots,d .
\]

\emph{These bounds are the eigenvalues of $M_{e_1}$.} The matrix
$M_{e_1} = \diag\big((1-\beta)\sigma_1^2,\, -\beta\sigma_2^2,\dots,-\beta\sigma_d^2\big)$ is
diagonal, and $(1-\beta)\sigma_1^2 \ge 0 \ge -\beta\sigma_i^2$ for every $i$, so
$\mu_1(M_{e_1}) = (1-\beta)\sigma_1^2$ and
$\big(\mu_2(M_{e_1}),\dots,\mu_d(M_{e_1})\big)$ is the multiset
$\set{-\beta\sigma_i^2 : i \ge 2}$ arranged in decreasing order. On the other hand the
eigenvalues of $-\beta\Sigma$ form the multiset $\set{-\beta\sigma_i^2 : i \in [d]}$, whose
\emph{smallest} element is $-\beta\sigma_1^2$ because $\sigma_1^2 = \max_i \sigma_i^2$ and
$\beta > 0$. Deleting it leaves $\set{-\beta\sigma_i^2 : i \ge 2}$. 
Hence
$\big(\mu_1(-\beta\Sigma),\dots,\mu_{d-1}(-\beta\Sigma)\big)
= \big(\mu_2(M_{e_1}),\dots,\mu_d(M_{e_1})\big)$, which together with the two displays above
proves \eqref{eq:eig-dom}. (No ordering of $\sigma_2^2,\dots,\sigma_d^2$ is needed.)

\emph{Coupling.} For any symmetric $M = UDU^\top$ with $U$ orthogonal and
$D = \diag(\mu_1(M),\dots,\mu_d(M))$ we have
$Z^\top M Z = (U^\top Z)^\top D\,(U^\top Z)$ with $U^\top Z \sim \cN(0,I_d)$; thus
$Z^\top M Z \eqd \sum_{j} \mu_j(M)\,g_j^2$ for $g \sim \cN(0,I_d)$, an expression depending on
$M$ only through its eigenvalues. Realising both quadratic forms with a single such $g$ and
using $g_j^2 \ge 0$ together with \eqref{eq:eig-dom} gives the pointwise inequality
$\sum_j \mu_j(M_a) g_j^2 \le \sum_j \mu_j(M_{e_1}) g_j^2$, whence
\[
  \Prob\big(Z^\top M_a Z > 0\big) \;\le\; \Prob\big(Z^\top M_{e_1} Z > 0\big).
\]
Applying \eqref{eq:half-identity} at $a$ and at $e_1$ and taking the supremum over unit vectors,
\begin{equation}\label{eq:sup-at-e1}
  \rho_\mu(\Delta) \;=\; \sup_{a \in \cA}\Prob_\mu\big(\tilde L(\theta,a) < \Delta\big)
  \;\le\; \Prob_\mu\big(\tilde L(\theta,e_1) < \Delta\big).
\end{equation}

\medskip
\noindent\textbf{Step 3 (Evaluation at $e_1$).}
With $\theta_i = \sigma_i Z_i$, the event at $a = e_1$ is $\theta_1 > (1-\Delta)\norm{\theta}_2$,
that is, $\theta_1 > 0$ together with
$(1-\beta)\theta_1^2 > \beta\sum_{i\ge2}\theta_i^2$. Dividing by
$(1-\beta)\sigma_1^2 > 0$, this is exactly
\[
  Z_1 > \sqrt{c\,W}, \qquad W := \sum_{i=2}^d \lambda_i Z_i^2 \;\ge\; 0, \qquad Z_1 \perp W .
\]
Conditioning on $W$ and applying $\bar\Phi(x) \le \tfrac12 e^{-x^2/2}$, valid for $x \ge 0$,
followed by $\bE[e^{-tZ^2}] = (1+2t)^{-1/2}$ for $t \ge 0$ and $Z \sim \cN(0,1)$ applied
coordinate-wise with $t = c\lambda_i/2$,
\[
  \Prob_\mu\big(\tilde L(\theta,e_1) < \Delta\big) = \bE\big[\bar\Phi(\sqrt{cW})\big]
  \;\le\; \tfrac12\,\bE\big[e^{-cW/2}\big]
  \;=\; \tfrac12 \prod_{i=2}^d \big(1 + c\,\lambda_i\big)^{-1/2}.
\]
Combining the previous display with \eqref{eq:sup-at-e1} proves the result.
\end{proof}

We will need the following result, adapted from \citet{chen2024assouad} based on Fano's inequality to derive our lower bound.
\begin{theorem}[Algorithm-specific Fano inequality; {\citet[Proposition~3]{chen2024assouad}}]
\label{thm:fano-alg}
Let $\Theta\subseteq\mathbb{R}^d$, let $\mu\in\mathcal{M}_1(\Theta)$ be a prior, and let
$L\colon\Theta\times\mathcal{A}\to\mathbb{R}_{\ge 0}$ be a measurable loss function. Fix
$T\in\mathbb{N}_0$ and a (possibly adaptive) algorithm which outputs
$\hat{A} \in \cA$ with $T$ round history $\cH_T$. Let
\[
  Y := (\Doff, \cH_T, \hat{A}).
\]
Let $\mathbb{P}_\mu$ and $\mathbb{E}_\mu$ denote probability and expectation under the joint law of
$(\theta, Y)$ in which $\theta\sim\mu$ and $Y$ is generated under $\theta$, and let
$I := I_\mu(\theta; Y)$ be the mutual information between $\theta$ and $Y$ under this law. For
$\Delta > 0$, define
\[
  \rho_\mu(\Delta) := \sup_{a\in\mathcal{A}} \mathbb{P}_{\theta\sim\mu}\bigl(L(\theta, a) < \Delta\bigr).
\]
If $I < \infty$ and $\rho_\mu(\Delta) < 1$, then, with the convention $\log(1/0) = +\infty$,
\begin{align*}
  \mathbb{P}_\mu\bigl(L(\theta, \hat{A}) < \Delta\bigr)
    &\le \frac{I + \log 2}{\log\bigl(1/\rho_\mu(\Delta)\bigr)}, \\
  \mathbb{E}_\mu\bigl[L(\theta, \hat{A})\bigr]
    &\ge \Delta\biggl(1 - \frac{I + \log 2}{\log\bigl(1/\rho_\mu(\Delta)\bigr)}\biggr).
\end{align*}
In particular, if $\log\bigl(1/\rho_\mu(\Delta)\bigr) \ge 2(I + \log 2)$, then
$\mathbb{E}_\mu\bigl[L(\theta, \hat{A})\bigr] \ge \Delta/2$.
\end{theorem}
 
\begin{proof}
The second claim is \citet[Proposition~3]{chen2024assouad} applied to the interactive statistical
decision making instance with model $\theta$, outcome $Y$, and risk $L(\theta, \hat{A})$:
since $\hat{A}$ takes values in $\mathcal{A}$, the supremum over outcomes in their bound is at
most $\rho_\mu(\Delta)$, and their lower bound is nonincreasing in that supremum. For
completeness, we include the short argument, which also gives the first claim.
 
Let $E := \{L(\theta, \hat{A}) < \Delta\}$ and $p := \mathbb{P}_\mu(E)$, let $Q$ be the marginal
law of $Y$, and let $P_{\theta, Y}$ be the joint law of $(\theta, Y)$. The event $E$ depends on $Y$
only through its last coordinate $\hat{A}$, and under the product law $\mu\otimes Q$, $\theta$
is independent of $Y$. Hence, by Fubini's theorem,
\[
  q := (\mu\otimes Q)(E)
     = \int_{\mathcal{A}} \mathbb{P}_{\theta\sim\mu}\bigl(L(\theta, a) < \Delta\bigr)\,
       Q_{\hat{A}}(\mathrm{d}a)
     \le \rho_\mu(\Delta),
\]
where $Q_{\hat{A}}$ is the marginal law of $\hat{A}$. Since
$I = D_{\mathrm{KL}}(P_{\theta, Y} \,\|\, \mu\otimes Q)$, the data processing inequality applied to
the indicator of $E$ gives $\mathrm{kl}(p \,\|\, q) \le I$, where $\mathrm{kl}(x \,\|\, y)$ is the KL
divergence between $\mathrm{Bernoulli}(x)$ and $\mathrm{Bernoulli}(y)$. Writing $h(p) \le \log 2$
for the binary entropy in nats and using $q \le \rho_\mu(\Delta) < 1$,
\[
  I \ge \mathrm{kl}(p \,\|\, q)
    = p\log\frac{1}{q} + (1 - p)\log\frac{1}{1 - q} - h(p)
    \ge p\log\frac{1}{\rho_\mu(\Delta)} - \log 2,
\]
with the conventions $0\log(1/0) = 0$ and $\log(1/0) = +\infty$. Rearranging gives the first
claim; when $\rho_\mu(\Delta) = 0$, finiteness of $I$ forces $p = 0$. The second claim follows
because $L \ge 0$ implies
$\mathbb{E}_\mu\bigl[L(\theta, \hat{A})\bigr]
 \ge \Delta\, \mathbb{P}_\mu\bigl(L(\theta, \hat{A}) \ge \Delta\bigr) = \Delta(1 - p)$.
\end{proof}
 
\begin{remark}
Proposition~9 of \citet{chen2024assouad} instead requires
$\sup_{a\in\mathcal{A}} \mathbb{P}_{\theta\sim\mu}\bigl(L(\theta, a) \le \Delta\bigr)
 \le \frac14 \exp\bigl(-2 I_\mu(T)\bigr)$,
where $I_\mu(T)$ is the supremum of the mutual information over \emph{all} $T$-round algorithms.
We use the algorithm-specific form of Theorem~\ref{thm:fano-alg} because
Lemma~\ref{lem:mi} bounds the mutual information of a given algorithm in terms of its own
expected design.
\end{remark}

\begin{lemma}[Directional error lower bound]
\label{lem:directional-error}
Let $\mathcal{A} = \mathbb{B}_2^d$ and let $\mu = \mathcal{N}(0, \Sigma)$ be the prior over $\theta$.
Assume that $\Sigma$ is diagonal in an eigenbasis of
$V_{\mathrm{off}}$, so that
\[
  V_{\mathrm{off}} = \operatorname{diag}(\nu_1, \dots, \nu_d), \qquad
  \Sigma = \operatorname{diag}(\sigma_1^2, \dots, \sigma_d^2),
\]
where the basis is indexed so that $\sigma_1^2 \ge \dots \ge \sigma_d^2$, and $\sigma_1^2 > 0$.
Thus $\nu_i$ is the eigenvalue of $V_{\mathrm{off}}$ associated with eigenvector $u_i$, which need not be the
$i$-th smallest. Fix a (possibly adaptive) pure exploration algorithm as in
Theorem~\ref{thm:fano-alg}, with output $\hat{A}\in\mathcal{A}$, and let
$
  n_i := \mathbb{E}_\mu\Bigl[\sum_{t=1}^T A_{t,i}^2\Bigr]
$
for $i \in [d]$, where $\mathbb{E}_\mu$ is the expectation under the joint law in which $\theta\sim\mu$ and the
offline data, online history, and output are generated under $\theta$. Define
$\widetilde{L}(\theta, a) := 1 - \langle \theta/\|\theta\|_2, a \rangle$ for $\theta \ne 0$. For
$\Delta\in(0, 1)$, let
\[
  c_\Delta = \frac{(1 - \Delta)^2}{2\Delta - \Delta^2}, \qquad
  d_\Delta = \frac{\sum_{i=2}^d \log\bigl(1 + c_\Delta \sigma_i^2/\sigma_1^2\bigr)}{\log(1 + c_\Delta)}.
\]
If $\Sigma$ and $\Delta$ satisfy
\begin{equation}
  \label{eq:mi-cap-condition}
  \sum_{i=1}^d \log\bigl(1 + \sigma_i^2(\nu_i + n_i)\bigr) + \log 2
  \le \frac{d_\Delta}{2}\log(1 + c_\Delta),
\end{equation}
then $\mathbb{E}_\mu\bigl[\widetilde{L}(\theta, \hat{A})\bigr] \ge \Delta/2$.
\end{lemma}
 
\begin{proof}
Since $\sigma_1^2 > 0$, we have $\theta \ne 0$ $\mu$-almost surely, so we may apply
Theorem~\ref{thm:fano-alg} with $\Theta = \mathbb{R}^d\setminus\{0\}$ and $L = \widetilde{L}$, which
is nonnegative on $\Theta\times\mathcal{A}$ because
$\langle \theta/\|\theta\|_2, a \rangle \le \|a\|_2 \le 1$. Let
$Y = (\Doff, \cH_T, \hat{A})$ be as in Theorem~\ref{thm:fano-alg}.
 
\emph{Mutual information.} Since
$\theta$ and $\hat{A}$ are conditionally independent given $(\Doff, \cH_T)$, by the
chain rule for mutual information and Lemma~\ref{lem:mi},
\[
  2 I_\mu(\theta; Y) = 2 I_\mu\bigl(\theta; (\cH_T, \Doff)\bigr)
  \le \sum_{i=1}^d \log\bigl(1 + \sigma_i^2(\nu_i + n_i)\bigr) < \infty.
\]
 
\emph{Cap probability.} By Lemma~\ref{lem:cap},
\[
  \rho_\mu(\Delta)
  = \sup_{a\in\mathcal{A}} \mathbb{P}_\mu\bigl(\widetilde{L}(\theta, a) < \Delta\bigr)
  \le \tfrac12 (1 + c_\Delta)^{-d_\Delta/2} \le \tfrac12.
\]
Taking logarithms and then using \eqref{eq:mi-cap-condition} and the bound on the mutual
information above,
\[
  \log\frac{1}{\rho_\mu(\Delta)}
  \ge \log 2 + \frac{d_\Delta}{2}\log(1 + c_\Delta)
  \ge 2\log 2 + \sum_{i=1}^d \log\bigl(1 + \sigma_i^2(\nu_i + n_i)\bigr)
  \ge 2\bigl(I_\mu(\theta; Y) + \log 2\bigr).
\]
 
\emph{Conclusion.} Since $I_\mu(\theta; Y) < \infty$ and $\rho_\mu(\Delta) \le 1/2$, the last
display is the condition in the final statement of Theorem~\ref{thm:fano-alg}, which gives
$\mathbb{E}_\mu\bigl[\widetilde{L}(\theta, \hat{A})\bigr] \ge \Delta/2$.
\end{proof}

\begin{restatable}[Expected simple regret lower bound]{theorem}{ExpSRLB}\label{thm:expected-simple-regret}
    Let $\cA=B_2^d$, and fix a possibly adaptive pure-exploration algorithm with output $\hat A$. Let $\mu = \cN(0,\Sigma)$ be the prior over the parameters and $\Sigma = \diag(\sigma_1^2,\dots,\sigma_d^2)$ in the eigenbasis of $\Voff$ with $\sigma_1^2 \ge \dots \ge\sigma_d^2$ and $\sigma_1^2 > 0$.
    Suppose $\Sigma$ and $\Delta \in (0, 1)$ satisfy \Cref{eq:exp-fano} with $n_i = \bE_\mu[\sum_{t=1}^T A_{t,i}^2]$ with coordinates in the specified eigenbasis. Then
    \begin{equation*}
        \bE_\mu\big[\SReg_T(\theta, \hat A)\big]
      \;\ge\; \frac{\Delta^2}{32\sqrt{e}} \sqrt{\Tr \Sigma} \,.
    \end{equation*}
\end{restatable}
\begin{proof}
First note that we can decompose $\SReg_T$ in terms of the scale of the problem and the error in the direction:
    \begin{align*}
        \SReg_T(\theta, \hat A) = \sup_{a \in \cA} \ip{\theta, a} - \ip{\theta, \hat A} = \norm{\theta}_2 - \ip{\theta, \hat A} = \norm{\theta}_2 \cdot \left(1 - \left\langle \frac{\theta}{\norm{\theta}_2}, \hat A \right\rangle \right) \,.
    \end{align*}
    Therefore, the goal is to lower bound $\bE_\mu[\norm{\theta}_2 \cdot \tilde L(\theta, \hat A)]$. Using \Cref{lem:hp-norm}, we get that with probability $1-\delta_0$, $ \norm{\theta}_2 \ge \frac{\delta_0}{\sqrt{e}}\sqrt{\Tr\Sigma}$. Now conditioning on the event $E = \{\norm{\theta}_2 \ge (\delta_0/\sqrt{e})\sqrt{\Tr\Sigma} \}$, we get
    \begin{align*}
        \bE_\mu[\norm{\theta}_2 \cdot \tilde L(\theta, \hat A)] & \ge \bE_\mu[ \norm{\theta}_2 \cdot \tilde L(\theta, \hat A) \cdot \1_E ]  \ge \frac{\delta_0}{\sqrt{e}} \sqrt{\Tr\Sigma} 
        \cdot \bE_\mu[ \tilde L(\theta, \hat A) \cdot \1_E ] \,.
    \end{align*}
    From \Cref{lem:directional-error}, $\bE_\mu[\tilde L(\theta, \hat A)] \ge \Delta/2$, provided $\Sigma$ and $\Delta$ satisfy \Cref{eq:exp-fano}. Now, since
    \begin{align*}
        \frac{\Delta}{2} \le \bE_\mu[\tilde L(\theta, \hat A)] = \bE_\mu[\tilde L(\theta, \hat A) \cdot \1_E] + \bE_\mu[\tilde L(\theta, \hat A) \cdot \1_{E^c}] \,,
    \end{align*}
    we get
    \begin{align*}
       \bE_\mu[ \tilde L(\theta, \hat A) \cdot \1_E ] &\ge \frac{\Delta}{2} - \bE_\mu[ \tilde L(\theta, \hat A) \1_{E^c} ] \\
       &\ge \frac{\Delta}{2} - 2\Prob_\mu(E^c) \tag{since $\tilde L(\theta, \hat A)\le 2$} \\
       &\ge \frac{\Delta}{2} - 2\delta_0\,.
    \end{align*}
    Finally,
    \begin{align*}
        \bE_\mu[\SReg_T(\theta, \hat A)] &\ge \frac{\delta_0}{\sqrt{e}} \sqrt{\Tr\Sigma} 
        \cdot \bE_\mu[ \tilde L(\theta, \hat A) \cdot \1_E ] \\
        &\ge \frac{\delta_0}{\sqrt{e}} \sqrt{\Tr\Sigma}\left(\frac{\Delta}{2} - 2\delta_0 \right) \\
        &\ge \frac{\Delta^2}{32\sqrt{e}} \sqrt{\Tr\Sigma}\,. \tag{optimizing over $\delta_0 \in (0,1)$}
    \end{align*}
\end{proof}

We are now ready to show our simple regret lower bound for adaptive algorithms.
\HpSrLb*
\begin{proof}
    Note that for any $a \in \cA$ and $\theta \in \R^d$, we have $\sregret_T(\theta, a) = \norm{\theta}_2 - \ip{a, \theta} \le 2\norm{\theta}_2$. Now for any fixed $\theta$, since the algorithm has a simple regret of at most $\epsilon$ with probability $1-\delta$ ($(\epsilon,\delta)$-PAC algorithm), we have
    \begin{equation*}
        \bE[\sregret_T(\theta, \hat A)] \le \epsilon + \delta\cdot (2\norm{\theta}_2) \,.
    \end{equation*}
    Taking expectation over $\theta \sim \mu$ we get
    \begin{equation*}
        \bE_\mu[\sregret_T(\theta, \hat A)] \le \epsilon + 2\delta\sqrt{\Tr\Sigma}\,,
    \end{equation*}
    where the last inequality follows from Jensen's.
    Combining it with \Cref{thm:expected-simple-regret} and $\delta \in (0, \Delta^2/(128\sqrt{e}))$, we get
    \begin{equation}\label{eq:high-prob-from-exp-reg-lb}
        \epsilon \ge \frac{\Delta^2}{64\sqrt{e}} \sqrt{\Tr\Sigma}\,.
    \end{equation}
    Also, recall from \Cref{thm:pureexp-high-prob-lower-bound} that for any $\delta \in (0, 1/3)$, we also have
    \begin{equation*}
        \epsilon \gtrsim \sqrt{\log(1/\delta)}\sup_{a \in \cA} \norm{a}_{\tilde V_0^{-1}} \,.
    \end{equation*}
    Combining the previous display with \Cref{eq:high-prob-from-exp-reg-lb}, we have
    \begin{equation*}
        \epsilon \gtrsim \Delta^2 \sqrt{\Tr \Sigma} + \sqrt{\log(1/\delta)}\sup_{a \in \cA} \norm{a}_{\tilde V_0^{-1}}\,,
    \end{equation*}
    proving the desired result.
\end{proof}

\corPureExpSpectralCoverage*
\begin{proof}
    First we prove the upper bound, then the lower bound.

    \textbf{Upper Bound.}

    Let $u_1,\ldots,u_d$ be an orthonormal basis of eigenvectors of $\Voff$ and $\nu_1, \dots, \nu_d$ be the corresponding eigenvalues.
    Then
    \begin{align*}
        \Voff
        =
        \sum_{i=1}^d\nu_i u_i u_i^\top,
        \qquad
        \nu_i>0.
    \end{align*}
    Fix nonnegative integer counts $n_1,\ldots,n_d$ summing to $T$, and let $N$ pull $u_i$ exactly $n_i$ times.
    Since $u_i\in B_2^d$, this is a feasible design and
    \begin{align*}
        \tilde V(N)
        =
        \sum_{i=1}^d\paren*{\nu_i+n_i}u_i u_i^\top.
    \end{align*}
    Since $\max_{a\in B_2^d}\ip{a,v}=\norm{v}_2$ for every $v\in\bR^d$,
    \begin{align*}
        \gw\paren*{\tilde V(N)^{-1/2}B_2^d}
        =
        \bE\norm{\tilde V(N)^{-1/2}g}_2
        \le
        \sqrt{\tr(\tilde V(N)^{-1})}
        =
        \sqrt{\sum_{i=1}^d\frac{1}{\nu_i+n_i}},
    \end{align*}
    where the inequality is Jensen's inequality.
    Moreover,
    \begin{align*}
        \max_{a\in B_2^d}\norm{a}_{\tilde V(N)^{-1}}
        =
        \sqrt{\lambda_{\max}(\tilde V(N)^{-1})}
        =
        \frac{1}{\sqrt{\min_{i\in[d]}\paren*{\nu_i+n_i}}}.
    \end{align*}
    Taking a $\min$ over $n_1, \dots, n_d$ implies that with probability at least $1-\delta$, the action $\hat a$ returned by \gwoofd satisfies
    \begin{align}
        \ip{\opta-\hat a,\truetheta}
        \le
        2
        \min_{\substack{n_1,\ldots,n_d\in\{0,1,\ldots\}:\\
                        \sum_{i=1}^d n_i=T}}
        \paren*{
            \sqrt{\sum_{i=1}^d\frac{1}{\nu_i+n_i}}
            +
            \sqrt{
                \frac{2\log(2/\delta)}
                {\min_{i\in[d]}\paren*{\nu_i+n_i}}
            }
        }.
        \label{eq:gw-oofd-spectrum-bound}
    \end{align}

    \textbf{\textit{Equally-covered offline data:}}

    Suppose first that the offline data are equally-covered, and let $c>0$ be the universal constant in \Cref{def:spectral-coverage-regimes}, so $\nu_i\ge c\Toff/d$ for every $i$.
    Allocate the $T$ online pulls as evenly as possible across $u_1,\ldots,u_d$, so every count $n_i$ is either $\lfloor T/d\rfloor$ or $\lceil T/d\rceil$.

    If $T<d$, the equally-covered condition implies $\Voff\succ0$, and hence $\Toff\ge d$ because $\operatorname{rank}(\Voff)\le\Toff$.
    Therefore $T+\Toff\le2\Toff$, and for every $i$,
    \begin{align*}
        \nu_i+n_i
        \ge
        \frac{c\Toff}{d}
        \gea
        \frac{T+\Toff}{d}.
    \end{align*}
    If $T\ge d$, then
    \begin{align*}
        n_i
        \ge
        \left\lfloor\frac{T}{d}\right\rfloor
        \ge
        \frac{T}{2d},
    \end{align*}
    and hence again
    \begin{align*}
        \nu_i+n_i
        \ge
        \frac{c\Toff}{d}+\frac{T}{2d}
        \gea
        \frac{T+\Toff}{d}.
    \end{align*}
    In either case,
    \begin{align*}
        \sum_{i=1}^d\frac{1}{\nu_i+n_i}
        \lea
        \frac{d^2}{T+\Toff},
        \qquad
        \min_{i\in[d]}\paren*{\nu_i+n_i}
        \gea
        \frac{T+\Toff}{d}.
    \end{align*}
    Applying \Cref{eq:gw-oofd-spectrum-bound} proves the equally-covered line for every $T\ge0$.

    \textbf{\textit{$k$-uncovered offline data:}}

    Next suppose that the data are $k$-uncovered for some $1\le k<d$ and that $T\ge k$.
    Allocate the online pulls as evenly as possible across the $k$ poorly explored eigenvectors $u_1,\ldots,u_k$.
    Thus, for every $i\le k$,
    \begin{align*}
        n_i
        \ge
        \left\lfloor\frac{T}{k}\right\rfloor
        \ge
        \frac{T}{2k},
    \end{align*}
    while $n_i=0$ for $i>k$.
    The $k$ poorly explored directions contribute at most
    \begin{align*}
        \sum_{i=1}^k\frac{1}{\nu_i+n_i}
        \le
        \frac{2k^2}{T}.
    \end{align*}
    For the remaining directions, the coverage condition gives
    \begin{align*}
        \sum_{i=k+1}^d\frac{1}{\nu_i}
        \le
        (d-k)\frac{k}{(d-k)T}
        =
        \frac{k}{T}.
    \end{align*}
    Hence
    \begin{align*}
        \sum_{i=1}^d\frac{1}{\nu_i+n_i}
        \lea
        \frac{k^2}{T}.
    \end{align*}
    The coverage condition gives that for all $i>k$,
    \begin{align*}
        \nu_i
        \ge
        \frac{(d-k)T}{k}
        \ge
        \frac{T}{k}.
    \end{align*}
    Thus
    \begin{align*}
        \min_{i\in[d]}\paren*{\nu_i+n_i}
        \ge
        \frac{T}{2k}.
    \end{align*}
    Applying \Cref{eq:gw-oofd-spectrum-bound} proves the $k$-uncovered line.

    Finally, suppose the data are $d$-uncovered and $T\ge d$.
    Allocate the online pulls as evenly as possible across all $d$ eigenvectors.
    Then every count satisfies $n_i\ge T/(2d)$, and therefore
    \begin{align*}
        \sum_{i=1}^d\frac{1}{\nu_i+n_i}
        \le
        \frac{2d^2}{T},
        \qquad
        \min_{i\in[d]}\paren*{\nu_i+n_i}
        \ge
        \frac{T}{2d}.
    \end{align*}
    A final application of \Cref{eq:gw-oofd-spectrum-bound} gives the $d$-uncovered result.

\textbf{Lower Bound 1.}

Choose $\Delta = 1/10$. For the equally-covered case, assume $d \ge 2$ and take
  $\Sigma_{\mathrm{equal}} := \dfrac{d}{100\,(T+\Toff)}\,I_d$. Similarly, for the $k$-uncovered case, assume $2 \le k \le d$ and $\nu_i \le 1$ for $i \le k$. Let
  $U$ be the span of the eigenvectors of $\Voff$ associated with $\nu_1,\dots,\nu_k$, let $P_k$
  be the orthogonal projection onto $U$, and take
  $\Sigma_k := \dfrac{k}{100\,(T+k)}\,P_k$. Consequently, we next show for any $(\eps,\delta)$-PAC algorithm with
$0 < \delta < \big(12800\sqrt e\big)^{-1} \approx 4.74\times10^{-5}$,
\[
  \eps \;\ge\; \frac{1}{64000\sqrt e}\cdot\frac{d}{\sqrt{T+\Toff}}
  \quad\text{and}\quad
  \eps \;\ge\; \frac{1}{64000\sqrt e}\cdot\frac{k}{\sqrt{T+k}}
\]
in the two regimes respectively.
\emph{Constants.} With $\Delta = 1/10$,
\[
  c_\Delta = \frac{(1-\Delta)^2}{2\Delta-\Delta^2} = \frac{81/100}{19/100} = \frac{81}{19},
  \qquad 1 + c_\Delta = \frac{100}{19},
  \qquad \log\frac{100}{19} = 1.6607\ldots
\]
Both priors are diagonal in the eigenbasis of $\Voff$ --- $\Sigma_{\mathrm{equal}}$ trivially,
and $\Sigma_k$ because $U$ is spanned by eigenvectors of $\Voff$ --- and in both cases
$\sigma_1^2 = \max_i\sigma_i^2$. Since all actions lie in $B_2^d$, $\Tr\Voff \le \Toff$
and $\Tr\bar V_T \le T$.

\medskip
\noindent\textbf{Effective dimension.}
In case (1), $\sigma_i^2/\sigma_1^2 = 1$ for every $i$, so
$d_\Delta = \sum_{i\ge2}\log(1+c)/\log(1+c) = d-1$ exactly. In case (2),
$\sigma_i^2/\sigma_1^2 = 1$ for $i \le k$ and $ \sigma_i^2 = 0$ for $i > k$, and
$\log(1 + c \cdot 0) = 0$, so $d_\Delta = k-1$ exactly. Write $q := d$ in case (1) and
$q := k$ in case (2), so that $d_\Delta = q-1$ and $q \ge 2$ by hypothesis.

\medskip
\noindent\textbf{Mutual information.}
We bound the left-hand side of \eqref{eq:exp-fano} directly, using $\log(1+x) \le x$ for
$x \ge 0$; note that an upper bound on $\log\det$ would \emph{not} suffice here, since
Hadamard's inequality gives $\log\det \le \sum_i\log(\cdot)_{ii}$ and therefore bounds the
determinant by the diagonal sum, not conversely.

In case (1), with $\sigma_i^2 = d/(100(T+\Toff))$ for all $i$,
\[
  \sum_{i=1}^d \log\big(1+\sigma_i^2(\nu_i+n_i)\big)
  \;\le\; \frac{d}{100(T+\Toff)}\sum_{i=1}^d (\nu_i + n_i)
  \;\le\; \frac{d}{100(T+\Toff)}\big(\Toff + T\big) \;=\; \frac{d}{100}.
\]
In case (2), $\sigma_i^2 = k/(100(T+k))$ for $i \le k$ and $\sigma_i^2=0$ for $i > k$, so the
terms with $i > k$ vanish and, using $\sum_{i\le k}\nu_i \le k$ (from $\nu_i \le 1$) and
$\sum_{i\le k}n_i \le \Tr\bar V_T \le T$,
\[
  \sum_{i=1}^d \log\big(1+\sigma_i^2(\nu_i+n_i)\big)
  \;\le\; \frac{k}{100(T+k)}\sum_{i\le k}(\nu_i+n_i)
  \;\le\; \frac{k}{100(T+k)}\big(k + T\big) \;=\; \frac{k}{100}.
\]

\medskip
\noindent\textbf{Verification of \eqref{eq:exp-fano}.}
It remains to check, for every integer $q \ge 2$,
\[
  \frac{q}{100} + \log 2 \;\le\; \frac{q-1}{2}\,\log\frac{100}{19}.
\]
At $q = 2$ the left side is $0.02 + 0.69315 = 0.71315$ and the right side is
$\tfrac12(1.66073) = 0.83037$, so the inequality holds. The right side increases in $q$ with
slope $\tfrac12\log(100/19) = 0.830 > 0.01$, the slope of the left side, so it holds for all
$q \ge 2$. Hence \eqref{eq:exp-fano} is satisfied and, by \Cref{lem:directional-error},
$\bE_\mu[\tilde L(\theta,\hat A)] \ge \Delta/2$.

\medskip
\noindent\textbf{Traces and conclusion.}
$\Tr\Sigma_{\mathrm{equal}} = d^2/\big(100(T+\Toff)\big)$ and
$\Tr\Sigma_k = k^2/\big(100(T+k)\big)$, since $\rank P_k = k$. Applying
\Cref{thm:expected-simple-regret} with $\Delta^2/(32\sqrt e) = 1/(3200\sqrt e)$,
\[
  \bE_\mu\big[\SReg_T\big] \;\ge\; \frac{1}{3200\sqrt e}\sqrt{\Tr\Sigma}
  \;=\; \frac{1}{32000\sqrt e}\cdot\frac{d}{\sqrt{T+\Toff}}
  \ \ \text{resp.}\ \ \frac{1}{32000\sqrt e}\cdot\frac{k}{\sqrt{T+k}} .
\]

\medskip
\noindent\textbf{PAC consequence.}
Since $\SReg_T \le 2\norm{\theta}_2$ pointwise, the $(\eps,\delta)$-PAC property gives
$\bE_\theta[\SReg_T] \le \eps + 2\norm{\theta}_2\delta$ for every $\theta$; integrating against
$\mu$ and applying Cauchy--Schwarz,
$\bE_\mu\norm{\theta}_2 \le \sqrt{\Tr\Sigma}$, yields
$\bE_\mu[\SReg_T] \le \eps + 2\delta\sqrt{\Tr\Sigma}$. Combining with the above,
\[
  \eps \;\ge\; \Big(\frac{1}{3200\sqrt e} - 2\delta\Big)\sqrt{\Tr\Sigma},
\]
so $\delta < \big(12800\sqrt e\big)^{-1}$ gives $2\delta < \tfrac12\cdot(3200\sqrt e)^{-1}$ and
hence $\eps \ge \sqrt{\Tr\Sigma}/(6400\sqrt e)$, which is the stated bound.

\begin{remark}
No coverage hypothesis is used in case (1): only $\Tr\Voff \le \Toff$ enters, so the bound
$d/\sqrt{T+\Toff}$ holds for every offline design. In case (2) only $\nu_i \le 1$ on the starved
block is used; since the prior is supported on $U$, the covered directions are invisible to both
the mutual information and the cap probability, so no condition on $\nu_{k+1},\dots,\nu_d$ and no
relation between $k$ and $d$ is required. Both hypotheses return in the matching upper bounds.
For $T \ge k$ the second rate reads $k/\sqrt{T+k} \ge k/(\sqrt2\sqrt T)$. The constants are not
optimised; taking $\Delta$ larger and tightening the lower-tail cutoff improves them
substantially, at the cost of the clean closed forms.
\end{remark}

\textbf{Lower Bound 2.}
To show a lower bound on $ \sup_{a \in B_2^d} \norm{a}_{\tilde V_0^{-1}}$ where $\tilde V_0 \;:=\; \Voff + \bE_0\Big[\sum_{t=1}^T A_tA_t^\top\Big]$, first note that 
\begin{equation}\label{eq:design-lb}
  \sup_{a \in B_2^d} \norm{a}_{\tilde V_0^{-1}}
  \;=\; \lambda_{\min}\big(\tilde V_0\big)^{-1/2}
  \;\ge\; \max_{1 \le k \le d}\left(\frac{k}{\sum_{i \le k}\nu_i + T}\right)^{1/2}.
\end{equation}
The first equality is immediate from
$\sup_{\norm{a}_2 \le 1} a^\top \tilde V_0^{-1} a = \lambda_{\max}(\tilde V_0^{-1})$.
For the inequality, fix $k \in [d]$ and let $P \in \R^{d \times k}$ have orthonormal columns spanning the
eigenspace of $\Voff$ associated with $\nu_1, \dots, \nu_k$. Restricting the Rayleigh quotient
to $\textrm{range}(P)$ and then bounding the minimum eigenvalue by the average,
\[
  \lambda_{\min}(\tilde V_0)
  \;\le\; \lambda_{\min}\big(P^\top \tilde V_0 P\big)
  \;\le\; \frac1k \Tr\big(P^\top \tilde V_0 P\big)
  \;=\; \frac1k\Big(\sum_{i \le k}\nu_i + \Tr\big(P^\top \bar V_T P\big)\Big),
\]
where $\bar V_T := \bE_0[\sum_{t \le T} A_tA_t^\top] \succeq 0$. Since $P$ has orthonormal
columns, $\Tr(P^\top \bar V_T P) \le \Tr \bar V_T$, and as $A_t \in B_2^d$,
\[
  \Tr \bar V_T \;=\; \bE_0\Big[\sum_{t=1}^T \norm{A_t}_2^2\Big] \;\le\; T .
\]
Combining and taking $k$ to be optimal gives inequality in \Cref{eq:design-lb}.
For equally-covered offline data, take $k = d$. Since the offline actions lie in $B_2^d$,
\[
  \sum_{i=1}^d \nu_i \;=\; \Tr \Voff \;=\; \sum_{s=1}^{\Toff}\norm{\aoff_s}_2^2 \;\le\; \Toff ,
\]
so $\sup_{a}\norm{a}_{\tilde V_0^{-1}} \ge \sqrt{d/(\Toff + T)}$. Note that no assumption on the offline spectrum is needed.
For the $k$-uncovered case, since $\nu_i \le 1$ for $i \le k$, we have $\sum_{i\le k}\nu_i \le k$ directly, which gives 
\[
    \sup_{a \in B_2^d}\norm{a}_{\tilde V_0^{-1}} \;\ge\; \sqrt{\frac{k}{T+k}},
\]
which is at least $\sqrt{k/(2T)}$ whenever $T \ge k$.

The final lower bound can be obtained by combining the two lower bounds, finishing the proof.
\end{proof}

\section{Missing Results and Details for Section \texorpdfstring{\ref{sec:gw-oope}}{Section 5}} \label{app:gw-oope}

\subsection{Offline Data Splitting}
\begin{lemma}[Gaussian splitting of the offline data] \label{lem:gw-oope-gaussian-split}
    Define
    \begin{align*}
        b_{\mathrm{off}}
        \defeq
        \sum_{s=1}^{\Toff}\roff_s\aoff_s.
    \end{align*}
    Draw $z_1,\ldots,z_m\stackrel{\mathrm{iid}}{\sim}\cN(0,\Voff/m)$ independently of the data, let
    $\bar z=m^{-1}\sum_{j=1}^m z_j$, and define
    \begin{align*}
        b_{\mathrm{off}}^\ell
        \defeq
        \frac{1}{m}b_{\mathrm{off}}+z_\ell-\bar z,
        \qquad
        \Voff[\ell]\defeq\frac{1}{m}\Voff.
    \end{align*}
    Then $b_{\mathrm{off}}^1,\ldots,b_{\mathrm{off}}^m$ are independent and
    \begin{align*}
        b_{\mathrm{off}}^\ell
        \sim
        \cN\paren*{\Voff[\ell]\truetheta,\Voff[\ell]}
        \qquad
        \text{for every }\ell\in[m].
    \end{align*}
    Moreover, $\sum_{\ell=1}^m b_{\mathrm{off}}^\ell=b_{\mathrm{off}}$.
\end{lemma}

\begin{proof}
    Since the offline actions are fixed and the reward noise is standard Gaussian,
    \begin{align*}
        b_{\mathrm{off}}
        =
        \Voff\truetheta
        +
        \sum_{s=1}^{\Toff}\noiseoff_s\aoff_s
        \sim
        \cN(\Voff\truetheta,\Voff).
    \end{align*}
    Therefore
    \begin{align*}
        \bE[b_{\mathrm{off}}^\ell]
        =
        \frac{1}{m}\Voff\truetheta.
    \end{align*}
    Moreover,
    \begin{align*}
        \operatorname{Cov}(z_\ell-\bar z)
        =
        \frac{m-1}{m^2}\Voff,
    \end{align*}
    and hence
    \begin{align*}
        \operatorname{Cov}(b_{\mathrm{off}}^\ell)
        =
        \frac{1}{m^2}\Voff
        +
        \frac{m-1}{m^2}\Voff
        =
        \frac{1}{m}\Voff.
    \end{align*}
    For $\ell\neq j$,
    \begin{align*}
        \operatorname{Cov}(z_\ell-\bar z,z_j-\bar z)
        =
        -\frac{1}{m^2}\Voff,
    \end{align*}
    and consequently
    \begin{align*}
        \operatorname{Cov}(b_{\mathrm{off}}^\ell,b_{\mathrm{off}}^j)
        =
        \frac{1}{m^2}\Voff
        -
        \frac{1}{m^2}\Voff
        =0.
    \end{align*}
    The vectors are jointly Gaussian, so zero cross-covariance implies independence.
    Finally,
    \begin{align*}
        \sum_{\ell=1}^m b_{\mathrm{off}}^\ell
        =
        b_{\mathrm{off}}
        +
        \sum_{\ell=1}^m z_\ell
        -
        m\bar z
        =
        b_{\mathrm{off}}.
    \end{align*}
\end{proof}

\subsection{Regret Upper Bound for \gwoope }
\thmGWOOPEUpperBound*

\begin{proof}
    We call each $b_{\mathrm{off}}^\ell$ an offline block.
    Consider a phase $\ell$ reached by the algorithm and condition on the online observations and offline blocks used strictly before that phase.
    In particular, this history does not include $b_{\mathrm{off}}^\ell$, any future block, or the original unsplit offline statistic.
    The active set $\cA_\ell$, the chosen design $\hat N_\ell$, and the decision whether the phase fits in the remaining budget are then fixed.

    If the phase is completed, its action sequence is fixed before the current block and online rewards are used.
    By \Cref{lem:gw-oope-gaussian-split}, the current block is independent of the preceding history.
    Consequently,
    \[
        \hat\theta_\ell-\truetheta
        =
        \tilde V_\ell(\hat N_\ell)^{-1}
        \left(
            b_{\mathrm{off}}^\ell-\Voff[\ell]\truetheta
            +\sum_{t=1}^{n_\ell}\eta_{\ell,t}A_{\ell,t}
        \right)
    \]
    is conditionally Gaussian with covariance $\tilde V_\ell(\hat N_\ell)^{-1}$.
    Applying \Cref{lem:gw-concentration} conditionally gives, with probability at least $1-\delta/(2\ell^2)$,
    \begin{equation}
        \max_{a\in\cA_\ell}\abs{\ip{a,\hat\theta_\ell-\truetheta}}
        \le\rad_{\beta_\ell}(\cA_\ell,\tilde V_\ell(\hat N_\ell))
        \le\eps_\ell.
        \label{eq:gw-oope-phase-confidence}
    \end{equation}
    The tower property and a union bound, using $\sum_{\ell\ge1}\delta/(2\ell^2)=\delta\pi^2/12<\delta$, show that this event holds simultaneously for all completed phases with probability at least $1-\delta$.
    Work on this event for the rest of the proof.
    In particular, for any completed phase and any $a,a'\in\cA_\ell$,
    \begin{equation}
        \abs{\ip{a'-a,\hat\theta_\ell-\truetheta}}\le2\eps_\ell.
        \label{eq:gw-oope-phase-pairwise-confidence}
    \end{equation}

    We first show that every reached phase, including a phase rejected for lack of budget, satisfies
    \begin{equation}
        \cA(4\eps_\ell)\subseteq\cA_\ell\subseteq\cA(12\eps_\ell).
        \label{eq:gw-oope-sandwich}
    \end{equation}
    This holds initially because $\cA_1=\cA$ and every gap is at most $2\le12\eps_1$.
    Suppose it holds at phase $\ell$, and this phase is completed.
    An action $a\in\cA(4\eps_{\ell+1})=\cA(2\eps_\ell)$ already belongs to $\cA_\ell$.
    For every $a'\in\cA_\ell$, its estimated gap satisfies
    $\ip{a'-a,\hat\theta_\ell}\le\ip{\opta-a,\truetheta}+2\eps_\ell\le4\eps_\ell$,
    so it is retained.
    Conversely, $\opta\in\cA_\ell$, and any retained action satisfies
    $\ip{\opta-a,\truetheta}\le\ip{\opta-a,\hat\theta_\ell}+2\eps_\ell\le6\eps_\ell=12\eps_{\ell+1}$.
    This proves the induction, including the sandwich for $\cA_{m+1}$ when all phases are completed.
    In particular, $\opta$ is never eliminated.

    Every design $N$ feasible in the definition of $n_\ell^\star$ is supported on $\cA_\ell$ and satisfies
    \[
        \rad_{\beta_\ell}(\cA_\ell,\tilde V_\ell(N))
        \le\rad_{\beta_\ell}(\cA(12\eps_\ell),\tilde V_\ell(N))
        \le\eps_\ell.
    \]
    It is therefore also feasible for the algorithm.
    Since the algorithm chooses a minimum-size design,
    \begin{equation}
        n_\ell\le n_\ell^\star
        \label{eq:gw-oope-phase-budget-comparison}
    \end{equation}
    for every reached phase; this comparison is immediate when $n_\ell^\star=+\infty$.

    Each completed phase $\ell$ uses at most $T$ samples, all from $\cA_\ell\subseteq\cA(12\eps_\ell)$.
    Its regret is consequently at most $12\eps_\ell\min\{n_\ell^\star,T\}$.
    If the algorithm stops because the active set is a singleton, all remaining regret is zero.
    If it instead stops because phase $j$ does not fit, the number of remaining rounds is smaller than $n_j\le n_j^\star$ and no larger than $T$.
    The action used in those rounds belongs to $\cA_j\subseteq\cA(12\eps_j)$, so the remaining regret is at most $12\eps_j\min\{n_j^\star,T\}$.
    Charging these rounds to the uncompleted phase $j$ shows that the sum in the theorem already covers all regret in either stopping case.

    It remains to consider the case when all $m$ phases are completed.
    The terminal action $\hat a$ maximizes $\ip{a,\hat\theta_m}$ over $\cA_{m+1}$, which contains $\opta$.
    Using \cref{eq:gw-oope-phase-confidence} and this empirical maximality,
    \begin{align}
        \ip{\opta-\hat a,\truetheta}
        &\le2\rad_{\beta_m}(\cA_m,\tilde V_m(\hat N_m))
        \nonumber\\
        &\le2\min\braces*{
            \eps_m,\rad_{\beta_m}(\cA(12\eps_m),\Voff[m])
        }.
        \label{eq:gw-oope-terminal-gap-certificate}
    \end{align}
    The last inequality uses the design's accuracy constraint, together with $\cA_m\subseteq\cA(12\eps_m)$ and $\tilde V_m(\hat N_m)\succeq\Voff[m]$.
    Multiplying this gap bound by at most $T$ remaining rounds and adding the completed-phase contributions proves the result.
\end{proof}

\subsection{Regret Lower Bound} \label{sec:reg-lb}
The following result generalizes the regret lower bound for linear bandits when $\cA = B_2^d$ with offline data \citep[Theorem 24.2]{lattimore2020bandit}.

\ThmRegLB*
\begin{proof}
Let $r \coloneqq \norm{\alpha}_2$, $s_i \in \{\pm 1\}$, so $\theta_i = s_i\alpha_i$
and $a(\theta)_i = s_i\alpha_i/r$, and $\sum_i \alpha_i^2 = r^2$.

\medskip
\noindent\textbf{Step 1 (Regret dominates a weighted squared deviation).}
For any round $t$,
\[
  \frac{r}{2}\sum_{i=1}^d\Big(\frac{\alpha_i}{r} - s_i A_{ti}\Big)^2
  = \sum_{i=1}^d\Big(\frac{\alpha_i^2}{r} - \alpha_i s_i A_{ti}\Big)
    - \frac{r}{2} + \frac{r}{2}\norm{A_t}_2^2
  \;\le\; \ip{a(\theta) - A_t,\, \theta},
\]
using $\sum_i\alpha_i^2 = r^2$ and $\norm{A_t}_2 \le 1$. Summing over $t$,
\[
  \bE_\theta[\regret_T(\cA)] \;\ge\; \frac{r}{2}\sum_{i=1}^d
  \bE_\theta\Big[\sum_{t=1}^T\Big(\frac{\alpha_i}{r} - s_i A_{ti}\Big)^2\Big].
\]

\medskip
\noindent\textbf{Step 2 (Truncation).}
For each $i$ let
$\tau_i \coloneqq T \wedge \min\set{t : \sum_{s\le t}A_{si}^2 \ge T\alpha_i^2/r^2}$,
a stopping time, so that $\sum_{t\le\tau_i}A_{ti}^2 \le T\alpha_i^2/r^2 + 1$ because
$A_{\tau_i i}^2 \le 1$. Since the summands are non-negative,
\[
  \bE_\theta[\regret_T(\cA)] \;\ge\; \frac{r}{2}\sum_{i=1}^d \bE_\theta\big[U_i(s_i)\big],
  \qquad
  U_i(x) \coloneqq \sum_{t=1}^{\tau_i}\Big(\frac{\alpha_i}{r} - xA_{ti}\Big)^2,
  \quad x \in \set{\pm1},
\]
and, by $(u-v)^2 \le 2u^2 + 2v^2$ together with $\tau_i \le T$,
\begin{equation}\label{eq:Mi}
  0 \;\le\; U_i(\pm1) \;\le\; \frac{2T\alpha_i^2}{r^2}
  + 2\Big(\frac{T\alpha_i^2}{r^2}+1\Big) = \frac{4T\alpha_i^2}{r^2} + 2 \eqqcolon M_i .
\end{equation}

\medskip
\noindent\textbf{Step 3 (Cost of flipping one coordinate).} 
Fix $i$ and let $\theta'$ agree with $\theta$ except $\theta_i' = -\theta_i$, so
$\theta - \theta' = 2s_i\alpha_i e_i$.
Let $\Prob_\theta$ and $\Prob_{\theta'}$ be the laws of $U_i(1)$ with respect to the environment-learner interaction measure induced by $\theta$ and $\theta'$.
The offline design is deterministic, so
$\sum_{s\le\Toff}\ip{\aoff_s,\theta-\theta'}^2
= \Toff (\theta-\theta')^\top \bar V_{\textrm{off}}(\theta-\theta') = 4\alpha_i^2\bar\nu_i\Toff$.
By the chain rule for KL divergence up to a stopping point \citep[Chapter 15]{lattimore2020bandit} and Step 2, we have
\[
  \KL\big(\Prob_\theta \,\|\, \Prob_{\theta'}\big)
  \;\le\; 2\alpha_i^2\bar\nu_i\Toff
  + \frac12\bE_\theta\Big[\sum_{t=1}^{\tau_i}4\alpha_i^2A_{ti}^2\Big]
  \;\le\; 2\alpha_i^2\Big(\bar\nu_i\Toff + \frac{T\alpha_i^2}{r^2} + 1\Big)\,.
\]

\medskip
\noindent\textbf{Step 4 (Pairing the two signs).}
As $0 \le U_i \le M_i$, Pinsker's inequality gives
$\bE_\theta[U_i(1)] \ge \bE_{\theta'}[U_i(1)] - M_i\sqrt{\KL(\Prob_\theta\|\Prob_{\theta'})/2}$.
By the parallelogram identity
$U_i(1) + U_i(-1) = 2\sum_{t\le\tau_i}\big(\alpha_i^2/r^2 + A_{ti}^2\big)$, which is always at
least $2T\alpha_i^2/r^2$: if $\tau_i = T$ the first sum already equals $2T\alpha_i^2/r^2$, and if
$\tau_i < T$ the stopping rule fired, so $\sum_{t\le\tau_i}A_{ti}^2 \ge T\alpha_i^2/r^2$.
Hence, taking $s_i = +1$ and $s_i' = -1$ and using \eqref{eq:Mi},
\begin{equation}\label{eq:Bi}
  \bE_\theta\big[U_i(1)\big] + \bE_{\theta'}\big[U_i(-1)\big]
  \;\ge\; \frac{2T\alpha_i^2}{r^2}
  - \alpha_i\Big(\frac{4T\alpha_i^2}{r^2}+2\Big)
    \sqrt{\bar\nu_i\Toff + \frac{T\alpha_i^2}{r^2} + 1} \eqqcolon B_i .
\end{equation}

\medskip
\noindent\textbf{Step 5 (Averaging over the hypercube).} Making the dependence of $\theta$ on signs explicit, define $\theta_i(s) = s_i \alpha_i$, where $s \in \{\pm 1\}^d$. For a fixed \(i\) and \(s_{-i}\in\{\pm1\}^{d-1}\), let \(\theta(s_{-i})\) and \(\theta'(s_{-i})\) be the parameters with those other signs and with \(i\)-th signs \(+1\) and \(-1\).
Now, bounding the maximum by the average over $s \in \{\pm 1\}^d$ and grouping:
\begin{align*}
\max_{\theta(s)\in\Theta}\bE_{\theta(s)}[\regret_T(\cA)]
  & \;\ge\; \frac{1}{2^d}\sum_{s \in \{\pm 1\}^d}\bE_{\theta(s)}[\regret_T(\cA)] \\
  &\;\ge\; \frac{r}{2^{d+1}}\sum_{i=1}^d\sum_{s_{-i} \in \{\pm 1\}^{d-1}} \Big(\bE_{\theta(s_{-i})}[U_i(1)] + \bE_{\theta'(s_{-i})}[U_i(-1)]\Big) \ge\frac{r}{4}\sum_{i=1}^d B_i \,.
\end{align*}

\medskip
\noindent\textbf{Step 6 (Optimizing the radius).}
By construction $\alpha_i^2/r^2 = w^\star_i$, so summing \eqref{eq:Bi} and using
$\sum_i\alpha_i^2 = r^2$ together with \eqref{eq:f-variational} gives
\[
  \max_{\theta\in\Theta}\bE_\theta[\regret_T(\cA)]
  \;\ge\; \frac{r}{4}\big[2T - r f^\star\big] \eqqcolon g(r),
  \qquad f^\star \coloneqq \min_{w\in\Delta_d}f(w).
\]
The map $g$ is a concave quadratic, increasing on $[0, r_\star]$ with $r_\star = T/f^\star$ and
$g(r_\star) = T^2/(4f^\star)$. Two cases, matching $r = \min\set{1, T/f^\star}$:
if $f^\star \ge T$ then $r_\star \le 1$, so $r = r_\star$ and
$g(r) = T^2/(4f^\star) = \frac14\min\set{T, T^2/f^\star}$;
if $f^\star < T$ then $r = 1 < r_\star$, so
$g(1) = \frac14(2T - f^\star) \ge \frac{T}{4} = \frac14\min\set{T, T^2/f^\star}$.
\end{proof}

\corRegLB*
\begin{proof}
    First we prove the upper bound, then the lower bound.

    \textbf{Upper bound.}

    \noindent\emph{Proof plan.}
    The bound in \cref{eq:gw-oope-oracle-bound} accounts for regret phase by phase.
    We split the phases based on a target accuracy $\bar\eps$.
    Earlier phases have relatively small budgets, while later phases have small gaps, so capping each phase budget by $T$ is enough to control their total contribution.
    We choose $\bar\eps$ so that both contributions are of order $T\bar\eps$.
    The main task is then to construct suitable designs inside the smaller near-optimal set $\cA(4\eps_\ell)$, while controlling prediction errors on $\cA(12\eps_\ell)$.
    On the Euclidean ball, these sets shrink around $\opta$ as the accuracy improves.
    Although this makes the available exploration directions smaller, it also makes the action differences smaller, and these two effects will cancel.
    We exploit this by using symmetric pairs of actions near $\opta$, together with repeated pulls of $\opta$ itself.
    We first obtain a bound that ignores offline information, then use the coverage assumptions to improve it.
    For very informative offline data, we use the terminal offline certificate so that the regret bound can decrease below a constant.

    Fix $T\ge1$, $\norm{\truetheta}_2\le1$, and $\delta\in(0,1)$, and work on the event in \Cref{thm:gw-oope-upper-bound}, which has probability at least $1-\delta$.
    Recall that $\eps_\ell=2^{-\ell}$ and $\beta_\ell\le\beta_m=\sqrt{2\log(4m^2/\delta)}$.

    \medskip\noindent\emph{From a target accuracy to regret.}
    All targets $\bar\eps$ used below for the online bounds satisfy $\bar\eps\ge1/T$.
    If $\bar\eps>1/16$, then $\regret_T(\cA)\le2T\le32T\bar\eps$.
    Otherwise, let $L$ be the largest phase with $\eps_L\ge8\bar\eps$.
    Since $\eps_1=1/2$ and $\eps_m\le1/T<8\bar\eps$, we have $1\le L<m$.

    Suppose that, for every $\ell\le L$, there is a design supported on $\cA(4\eps_\ell)$ whose confidence radius on $\cA(12\eps_\ell)$ is at most $\eps_\ell$ and whose size is at most $T(\bar\eps/\eps_\ell)^2$.
    Then $n_\ell^\star\le T(\bar\eps/\eps_\ell)^2$ in these phases, and
    \[
        \sum_{\ell=1}^L\eps_\ell\min\{n_\ell^\star,T\}
        \le T\bar\eps^2\sum_{\ell=1}^L\eps_\ell^{-1}
        \le\frac{2T\bar\eps^2}{\eps_L}
        \le\frac{T\bar\eps}{4}.
    \]
    For later phases, we need no bound on $n_\ell^\star$: the cap of $T$ gives
    \[
        \sum_{\ell=L+1}^m\eps_\ell\min\{n_\ell^\star,T\}
        \le T\sum_{\ell=L+1}^m\eps_\ell
        \le2T\eps_{L+1}
        <16T\bar\eps.
    \]
    The terminal term in \cref{eq:gw-oope-oracle-bound} is at most $2T\eps_m\le2T\bar\eps$.
    Thus the theorem gives $\regret_T(\cA)\lesssim T\bar\eps$.
    It remains to construct designs with the stated sample counts.

    \medskip\noindent\emph{A comparison design inside the near-optimal cap.}
    If $\truetheta=0$, the regret is zero.
    Otherwise, write $\opta=\truetheta/\norm{\truetheta}_2$ and set $\rho_\ell^2=\min\{1,4\eps_\ell/\norm{\truetheta}_2\}$.
    For every $a\in\cA(12\eps_\ell)$, the gap condition gives $\norm{\truetheta}_2(1-a^\top\opta)\le12\eps_\ell$.
    Since $\norm{a}_2\le1$ and $\norm{\opta}_2=1$,
    \begin{equation}
        \norm{a-\opta}_2^2
        \le 2(1-a^\top\opta)
        \le \min\left\{4,\frac{24\eps_\ell}{\norm{\truetheta}_2}\right\}
        \le 6\rho_\ell^2.
        \label{eq:gw-oope-spectral-cap-diameter}
    \end{equation}
    Thus every action whose error we need to control lies within distance $\sqrt6\rho_\ell$ of $\opta$.

    Let $U$ be an $s$-dimensional subspace in which we want to acquire information, with $s\ge1$, and write $P_S$ for the orthogonal projector onto a subspace $S$.
    Project $U$ onto the tangent space at $\opta$: let $W=P_{(\opta)^\perp}U$ and choose an orthonormal basis $w_1,\ldots,w_q$ of $W$, where $q\le s$.
    To learn in a direction $w_j\perp\opta$, we would like to sample near-optimal actions with as large a component in that direction as possible. 
    Consider a unit action of the form $c_1\opta\pm c_2w_j$, with $c_2\ge0$. 
    By Pythagoras, $c_2^2=1-c_1^2$, while membership in $\cA(4\eps_\ell)$ requires $\norm{\truetheta}_2(1-c_1)\le4\eps_\ell$.
    When $4\eps_\ell\le\norm{\truetheta}_2$, the largest perpendicular component is therefore obtained by taking $c_1=1-4\eps_\ell/\norm{\truetheta}_2=1-\rho_\ell^2$.
    Its squared length is $c_2^2=1-(1-\rho_\ell^2)^2=\rho_\ell^2(2-\rho_\ell^2)$, which lies between $\rho_\ell^2$ and $2\rho_\ell^2$.
    When $4\eps_\ell\ge\norm{\truetheta}_2$, the actions $\pm w_j$ already belong to $\cA(4\eps_\ell)$, so the largest perpendicular component is $1=\rho_\ell$.

    Thus, in either case, $\rho_\ell$ is within a factor $\sqrt2$ of the largest available perpendicular component. 
    We use this value rather than the exact maximum to keep the expressions simple, and choose $x_{j,\pm}=\sqrt{1-\rho_\ell^2}\,\opta\pm\rho_\ell w_j$.
    These actions have unit norm, and $1-\sqrt{1-u}\le u$ for $u\in[0,1]$ gives
    \begin{align*}
        \ip{\opta-x_{j,\pm},\truetheta}
        = \norm{\truetheta}_2\bigl(1-\sqrt{1-\rho_\ell^2}\bigr)
        \le \norm{\truetheta}_2\rho_\ell^2
        \le 4\eps_\ell.
    \end{align*}
    Hence $x_{j,\pm}\in\cA(4\eps_\ell)$, as required.

    For an integer $n\ge8s$ and $q\ge1$, take $r=\lfloor n/(4q)\rfloor$ copies of each $x_{j,+}$ and $x_{j,-}$, and put all remaining pulls on $\opta$.
    Denote the resulting design by $N_{\ell,n}(U)$.
    At least $n/2$ pulls are assigned to $\opta$, and $r\ge n/(8q)$.
    The symmetric pairs cancel their cross terms, giving
    \begin{equation}
        V(N_{\ell,n}(U))
        =(n-2qr\rho_\ell^2)\opta(\opta)^\top+2r\rho_\ell^2P_W
        \succeq\frac n2\opta(\opta)^\top+\frac{n\rho_\ell^2}{4s}P_W.
        \label{eq:gw-oope-integer-cap-radial}
    \end{equation}
    When $q=0$, take all $n$ pulls at $\opta$; the same lower bound holds with $P_W=0$.
    Since $U\subseteq\operatorname{span}(\opta,W)$ and $\rho_\ell^2/(4s)\le1/2$, it follows that
    \begin{equation}
        V(N_{\ell,n}(U))\succeq\frac{n\rho_\ell^2}{4s}P_U.
        \label{eq:gw-oope-integer-cap-subspace}
    \end{equation}
    Also, for every positive definite $M\succeq V(N_{\ell,n}(U))$,
    \begin{equation}
        \norm{\opta}_{M^{-1}}^2\le\frac2n.
        \label{eq:gw-oope-integer-cap-optimal-action}
    \end{equation}
    Indeed, since $M\succeq(n/2)\opta(\opta)^\top$ we have that $(n/2)\norm{\opta}_{M^{-1}}^2\le1$.

    To turn these information bounds into a confidence radius, translate the cap by $\opta$.
    Translation does not change Gaussian width, and \cref{eq:gw-oope-spectral-cap-diameter} bounds the translated cap by a ball of radius $\sqrt6\rho_\ell$.
    For $g\sim\cN(0,I_d)$, its Gaussian width is therefore at most $\sqrt6\rho_\ell\,\bE\norm{M^{-1/2}g}_2\le\sqrt6\rho_\ell\sqrt{\tr(M^{-1})}$.
    The triangle inequality also gives $\norm{a}_{M^{-1}}\le\norm{\opta}_{M^{-1}}+\sqrt6\rho_\ell\sqrt{\lambda_{\max}(M^{-1})}$ on the cap.
    Combining the two parts of the confidence radius,
    \begin{equation}
        \rad_{\beta_\ell}(\cA(12\eps_\ell),M)
        \lesssim
        \rho_\ell\sqrt{\tr(M^{-1})}
        +\beta_m\left(\norm{\opta}_{M^{-1}}
        +\rho_\ell\sqrt{\lambda_{\max}(M^{-1})}\right).
        \label{eq:gw-oope-cap-confidence-radius}
    \end{equation}
    We reserved a constant fraction of the pulls for $\opta$ to control $\norm{\opta}_{M^{-1}}$, since when $\rho_\ell$ is large the component $\sqrt{1 - \rho_\ell^2}$ in $\opta$ from $x_{j, \pm}$ is small.
    Note that the design we chose is just for analysis purposes and the algorithm does not need to know $\opta$ or $U$.

    \medskip\noindent\emph{A bound that does not use offline coverage.}
    Take $U=\bR^d$ and $N=N_{\ell,n}(\bR^d)$ for $n\ge8d$.
    Ignoring the positive semidefinite offline matrix, \cref{eq:gw-oope-integer-cap-subspace} gives $\tilde V_\ell(N)\succeq n\rho_\ell^2I_d/(4d)$.
    Thus
    \[
        \rho_\ell^2\tr(\tilde V_\ell(N)^{-1})\le\frac{4d^2}{n},
        \qquad
        \rho_\ell^2\lambda_{\max}(\tilde V_\ell(N)^{-1})\le\frac{4d}{n}.
    \]
    Together with \cref{eq:gw-oope-integer-cap-optimal-action,eq:gw-oope-cap-confidence-radius}, these inequalities imply
    \begin{equation}
        \rad_{\beta_\ell}(\cA(12\eps_\ell),\tilde V_\ell(N))
        \lesssim\frac{d+\sqrt d\,\beta_m}{\sqrt n},
        \qquad n\ge8d.
        \label{eq:gw-oope-online-comparison-radius}
    \end{equation}
    The factor $\rho_\ell^2$ in the available information has cancelled the factor coming from the diameter of the cap.

    Choose a sufficiently large numerical constant $C$ and set $\bar\eps=C(d+\sqrt d\,\beta_m)/\sqrt T$.
    This target satisfies $\bar\eps\ge1/T$.
    As above, the case $\bar\eps>1/16$ follows from the trivial regret bound.
    For a phase with $\eps_\ell\ge8\bar\eps$, take $n=\lfloor T(\bar\eps/\eps_\ell)^2\rfloor$.
    Since $\eps_\ell\le1/2$, the quantity inside the floor is at least $4C^2d^2\ge2$, and hence
    \[
        n\ge\frac T2\left(\frac{\bar\eps}{\eps_\ell}\right)^2
        \ge2C^2d^2\ge8d,
        \qquad
        \frac{d+\sqrt d\,\beta_m}{\sqrt n}
        \le\frac{\sqrt2}{C}\eps_\ell.
    \]
    For sufficiently large $C$, \cref{eq:gw-oope-online-comparison-radius} therefore certifies radius at most $\eps_\ell$.
    Thus $n_\ell^\star\le n\le T(\bar\eps/\eps_\ell)^2$, and the target-accuracy argument gives
    \begin{equation}
        \regret_T(\cA)\lesssim(d+\sqrt d\,\beta_m)\sqrt T.
        \label{eq:gw-oope-online-regret}
    \end{equation}
    This bound holds for any offline information matrix.

    \medskip\noindent\emph{Equally-covered offline data.}
    The coverage condition gives $\Voff[\ell]\succeq c\Toff I_d/(md)$.
    Since $\cA(12\eps_\ell)\subseteq B_2^d$, the zero design satisfies
    \begin{equation}
        \rad_{\beta_\ell}(\cA(12\eps_\ell),\Voff[\ell])
        \le\sqrt{\tr((\Voff[\ell])^{-1})}
        +\beta_m\sqrt{\lambda_{\max}((\Voff[\ell])^{-1})}
        \le\frac{d+\sqrt d\,\beta_m}{\sqrt{c\Toff/m}}.
        \label{eq:gw-oope-equally-covered-offline-radius}
    \end{equation}
    Set $\bar\eps=(d+\sqrt d\,\beta_m)/\sqrt{c\Toff/m}$.
    If $\bar\eps>1/2$, then $\regret_T(\cA)\le2T\le4T\bar\eps$.
    Otherwise, let $L$ be the largest phase with $\eps_L\ge\bar\eps$.
    The zero design attains the required accuracy in every phase $\ell\le L$, so $n_\ell^\star=0$.
    If $L<m$, the remaining terms in the sum satisfy
    \[
        \sum_{\ell=L+1}^m\eps_\ell\min\{n_\ell^\star,T\}
        \le T\sum_{\ell=L+1}^m\eps_\ell
        \le2T\eps_{L+1}
        <2T\bar\eps.
    \]
    If $L=m$, the entire sum is zero.
    In either case, \cref{eq:gw-oope-equally-covered-offline-radius} bounds the terminal term by $2T\bar\eps$, so \cref{eq:gw-oope-oracle-bound} gives $\regret_T(\cA)\lesssim T\bar\eps$.
    Unlike the online target-accuracy argument, this reasoning places no lower bound on $\bar\eps$.
    It therefore remains valid when the offline error is much smaller than $1/T$.

    We have obtained $\regret_T(\cA)\lesssim T(d+\sqrt d\,\beta_m)/\sqrt{\Toff/m}$.
    Both this bound and \cref{eq:gw-oope-online-regret} hold on the same event.
    Taking the smaller, and using $\max\{T,\Toff/m\}\ge(T+\Toff/m)/2$, yields
    \begin{equation}
        \regret_T(\cA)
        \lesssim\frac{T(d+\sqrt d\,\beta_m)}{\sqrt{T+\Toff/m}}.
        \label{eq:gw-oope-equally-covered-regret}
    \end{equation}
    This proves the equally-covered bound.

    \medskip\noindent\emph{$k$-uncovered offline data.}
    First suppose $1\le k<d$.
    Let $u_1,\ldots,u_d$ be orthonormal eigenvectors of $\Voff$, ordered by increasing eigenvalue, and let $U=\operatorname{span}\{u_1,\ldots,u_k\}$.
    For $n\ge8k$, take $N=N_{\ell,n}(U)$.
    The coverage assumption and \cref{eq:gw-oope-integer-cap-subspace} give
    \[
        \tilde V_\ell(N)
        \succeq\frac{n\rho_\ell^2}{4k}P_U
        +\frac{(d-k)T}{km}P_{U^\perp}.
    \]
    This bound controls the information separately in $U$ and $U^\perp$, so its inverse can be handled one subspace at a time.
    Inverting it and summing its eigenvalues gives the trace bound below; taking the largest inverse eigenvalue gives the second bound:
    \begin{align*}
        \rho_\ell^2\tr(\tilde V_\ell(N)^{-1})
        &\le\frac{4k^2}{n}+\frac{mk\rho_\ell^2}{T}
        \le\frac{4k^2}{n}+\frac{mk}{T},
        \\
        \rho_\ell^2\lambda_{\max}(\tilde V_\ell(N)^{-1})
        &\le\max\left\{\frac{4k}{n},
        \frac{mk\rho_\ell^2}{(d-k)T}\right\}
        \le\frac{4k}{n}+\frac{mk}{T}.
    \end{align*}
    Using \cref{eq:gw-oope-integer-cap-optimal-action,eq:gw-oope-cap-confidence-radius} and $\sqrt{x+y}\le\sqrt x+\sqrt y$, we obtain
    \begin{equation}
        \rad_{\beta_\ell}(\cA(12\eps_\ell),\tilde V_\ell(N))
        \lesssim
        \frac{k+\sqrt k\,\beta_m}{\sqrt n}
        +\beta_m\sqrt{\frac{mk}{T}}.
        \label{eq:gw-oope-k-uncovered-comparison-radius}
    \end{equation}
    Here $\beta_m\ge1$ absorbs the width contribution $\sqrt{mk/T}$ from the covered directions, while $k\ge1$ absorbs the term $\beta_m/\sqrt n$ for the error at $\opta$.
    The first term decreases with the online sample count; the second comes from the offline information in the other $d-k$ directions.

    Set $\bar\eps=C(k+\beta_m\sqrt{mk})/\sqrt T$ for a sufficiently large numerical constant $C$.
    Again $\bar\eps\ge1/T$, and $\bar\eps>1/16$ is handled by the trivial bound.
    For each phase with $\eps_\ell\ge8\bar\eps$, take $n=\lfloor T(\bar\eps/\eps_\ell)^2\rfloor$.
    The same rounding calculation gives $n\ge(T/2)(\bar\eps/\eps_\ell)^2\ge2C^2k^2\ge8k$.
    Since $m\ge1$,
    \[
        \frac{k+\sqrt k\,\beta_m}{\sqrt n}
        \le\frac{\sqrt2}{C}\eps_\ell,
        \qquad
        \beta_m\sqrt{\frac{mk}{T}}
        \le\frac{\bar\eps}{C}
        \le\frac{\eps_\ell}{8C}.
    \]
    For large enough $C$, both terms in \cref{eq:gw-oope-k-uncovered-comparison-radius} therefore sum to at most $\eps_\ell$.
    Hence $n_\ell^\star\le n\le T(\bar\eps/\eps_\ell)^2$, and the target-accuracy argument yields
    \begin{equation}
        \regret_T(\cA)
        \lesssim k\sqrt T+\beta_m\sqrt{mkT},
        \qquad 1\le k<d.
        \label{eq:gw-oope-k-uncovered-regret}
    \end{equation}
    For $k=d$, there are no covered directions to account for, and \cref{eq:gw-oope-online-regret} already gives the sharper bound $\regret_T(\cA)\lesssim d\sqrt T+\beta_m\sqrt{dT}$ without the factor $\sqrt m$.

    These high-probability bounds hold for the arbitrary confidence parameter $\delta$ fixed at the start.
    Each can also be capped by the deterministic bound $2T$.

    \medskip\noindent\emph{Conversion to expected regret.}
    The preceding bounds hold for an arbitrary choice of $\delta\in(0,1)$.
    Since $\regret_T(\cA)\le2T$ is always true, taking expectations adds at most $2T\delta$ to the corresponding high-probability bound.
    Run \gwoope{} with $\delta=1/[2(T+\Toff)]$ in the equally-covered case and with $\delta=1/(2T)$ in the $k$-uncovered case.
    The failure-event contributions are then $T/(T+\Toff)$ and $1$, respectively.
    They are absorbed by the main terms, since $T/(T+\Toff)\le T/\sqrt{T+\Toff}\le Td/\sqrt{T+\Toff/m}$ and $1\le k\sqrt T$.

    Substituting these choices into $\beta_m=\sqrt{2\log(4m^2/\delta)}$ and applying
    \cref{eq:gw-oope-equally-covered-regret,eq:gw-oope-k-uncovered-regret,eq:gw-oope-online-regret} gives
    \begin{equation}
        \bE\regret_T(\cA)\lesssim
        \begin{cases}
            \displaystyle
            \min\left\{T,\,
            \frac{T\bigl(d+\sqrt{d\log(8m^2(T+\Toff))}\bigr)}
                {\sqrt{T+\Toff/m}}\right\},
            &\text{equally-covered},
            \\[9pt]
            \displaystyle
            \min\left\{T,\,
            k\sqrt T+\sqrt{mkT\log(8m^2T)}\right\},
            &\text{$k$-uncovered},\quad 1\le k<d,
            \\[7pt]
            \displaystyle
            \min\left\{T,\,
            d\sqrt T+\sqrt{dT\log(8m^2T)}\right\},
            &\text{$d$-uncovered}.
        \end{cases}
        \label{eq:gw-oope-spectral-expected-regret}
    \end{equation}
    The minima use the deterministic bound $2T$.
    Finally, $T+\Toff/m\ge(T+\Toff)/m$ and $m=\max\{1,\lceil\log_2 T\rceil\}$.
    Thus these bounds give the rates $dT/\sqrt{T+\Toff}$ and $k\sqrt T$ up to logarithmic factors.

    \textbf{Lower Bound.}
    
Both follow from \Cref{thm:reg-lb} by exhibiting a feasible $w \in \Delta_d$ and bounding
$f(w)$ from above; since $f^\star \le f(w)$, this lower bounds $T^2/f^\star$. Throughout write
$\nu_i \coloneqq \bar\nu_i\Toff$ for the eigenvalues of the unnormalized offline covariance, in
which $f(w) = \sum_i\sqrt{w_i}(4Tw_i+2)\sqrt{\nu_i + Tw_i + 1}$.

\medskip
\noindent\textbf{Equally covered ($\nu_i \ge c\Toff/d$ for all $i$; $d \le 2T$).}
Take $w_i = 1/d$. Then using Jensen's inequality,
\[
  f(w) \le d\cdot\frac{1}{\sqrt d}\Big(\frac{4T}{d}+2\Big)
         \sqrt{\frac{\Toff}{d} + \frac{T}{d} + 1}.
\]
Since $d \le 2T$ we have $2 \le 4T/d$, so $4T/d + 2 \le 8T/d$; and
$1 \le 2T/d \le 2(T+\Toff)/d$, so the radicand is at most $3(T+\Toff)/d$. Hence
\[
  f^\star \;\le\; f(w) \;\le\; \frac{8\sqrt3\,T\sqrt{T+\Toff}}{d},
  \qquad\text{so}\qquad
  \frac{T^2}{f^\star} \;\ge\; \frac{1}{8\sqrt3}\cdot\frac{dT}{\sqrt{T+\Toff}} .
\]
For $d \le 8\sqrt3\sqrt{T+\Toff}$, the right-hand side is at most $T$, so the same bound holds for
$\min\set{T, T^2/f^\star}$, and \Cref{thm:reg-lb} gives
$\max_{\theta\in\Theta}\bE_\theta[\regret_T(\cA)] \ge \frac{1}{32\sqrt3}\,dT/\sqrt{T+\Toff}$.

\medskip
\noindent\textbf{$k$-uncovered ($\nu_i \le 1$ for $i \le k$, $\nu_i \ge (d-k)T/k$ for $i > k$;
$k \le T$).}
Take $w$ uniform on the uncovered block, $w_i = 1/k$ for $i \le k$ and $w_i = 0$ otherwise. Every
term with $i > k$ carries the prefactor $\sqrt{w_i} = 0$ and vanishes, whatever $\nu_i$ is; for
$i \le k$ we use $\nu_i \le 1$:
\[
  f(w) = \frac{1}{\sqrt{k}}\Big(\frac{4T}{k}+2\Big) \sum_{i=1}^k \sqrt{\nu_i + \frac{T}{k} + 1}
  \;\le\; \sqrt k\Big(\frac{4T}{k}+2\Big)\sqrt{\frac{T}{k}+2}
  \;\le\; \sqrt k\cdot\frac{8T}{k}\cdot\sqrt{\frac{3T}{k}}
  = \frac{8\sqrt3\,T^{3/2}}{k},
\]
where $k \le T$ gives both $2 \le 4T/k$ and $2 \le 2T/k$. Hence
$T^2/f^\star \ge k\sqrt T/(8\sqrt3)$, which is at most $T$ since $k \le 8\sqrt3\sqrt T$, so
\Cref{thm:reg-lb} gives
$\max_{\theta\in\Theta}\bE_\theta[\regret_T(\cA)] \ge \frac{1}{32\sqrt3}\,k\sqrt T$.

\end{proof}

\section{Technical Results} \label{app:technical-results}

\begin{restatable}[Lipschitz Concentration {\citep[Theorem~5.6]{boucheron2013concentration}}]{lemma}{lemLipschitzConcentration}
    \label{lem:lipschitz-concentration}
    Let $g\sim\cN(0,I_d)$ and let $f:\bR^d\to\bR$ be
    $L$-Lipschitz in Euclidean distance, meaning that
    \begin{align*}
        |f(z)-f(z')|\le L\norm{z-z'}_2
        \qquad\text{for all }z,z'\in\bR^d.
    \end{align*}
    For every $\delta\in(0,1)$,
    \begin{align*}
        \bP\left(
            f(g)\le\bE f(g)+L\sqrt{2\log(1/\delta)}
        \right)\ge1-\delta.
    \end{align*}
\end{restatable}

\begin{restatable}[Sudakov--Fernique inequality {\citep[Theorem~7.2.8]{vershynin2026high}}]{lemma}{lemSudakovFernique}
    \label{lem:sudakov-fernique}
    Let $(X_t)_{t \in T}$ and $(Y_t)_{t \in T}$ be two mean-zero Gaussian processes.
    Assume that for all $t, s \in T$, we have
    \begin{align*}
        \bE\brackets*{(X_t-X_s)^2}
        \le
        \bE\brackets*{(Y_t-Y_s)^2}.
    \end{align*}
    Then
    \begin{align*}
        \bE\sup_{t \in T}X_t
        \le
        \bE\sup_{t \in T}Y_t.
    \end{align*}
\end{restatable}

\end{document}